\documentclass[lettersize,journal]{IEEEtran}
\usepackage{amsmath,amsfonts}
\usepackage{array}
\usepackage[caption=false,font=normalsize,labelfont=sf,textfont=sf]{subfig}
\usepackage{textcomp}
\usepackage{stfloats}
\usepackage{url}
\usepackage{verbatim}
\usepackage{graphicx}
\usepackage{cite}
\usepackage{colortbl}  
\usepackage{array}

\usepackage{multirow}
\usepackage{xcolor}
\usepackage{mathrsfs}

\usepackage{enumitem}
\usepackage{pifont}
\newcommand{\cmark}{\textcolor{green!80!black}{\ding{51}}}
\newcommand{\xmark}{\textcolor{red}{\ding{55}}}
\usepackage{amsthm}
\usepackage{thmtools} 
\theoremstyle{plain}
\newtheorem{theorem}{Theorem}[section]
\newtheorem{proposition}[theorem]{Proposition}
\newtheorem{lemma}[theorem]{Lemma}

\theoremstyle{definition}

\newtheorem{assumption}[theorem]{Assumption}
\theoremstyle{remark}

\newtheorem{game}[theorem]{Game}

\usepackage[ruled,vlined,linesnumbered]{algorithm2e} 
\makeatletter
\newcommand{\nosemic}{\renewcommand{\@endalgocfline}{\relax}}
\newcommand{\dosemic}{\renewcommand{\@endalgocfline}{\algocf@endline}}
\let\oldnl\nl
\newcommand{\nonl}{\renewcommand{\nl}{\let\nl\oldnl}}

\allowdisplaybreaks
\makeatother

\begin{document}

\title{{FedAL}: Black-Box Federated Knowledge Distillation Enabled by Adversarial Learning}

\author{Pengchao Han, Xingyan Shi, Jianwei Huang,~\IEEEmembership{Fellow,~IEEE}
\thanks{
Pengchao Han is with the School of Information Engineering, Guangdong University of Technology, Guangzhou 510006, China 
(email: hanpengchao@gdut.edu.cn).

Xingyan Shi is with the School of Data Science, The Chinese University of Hong Kong, Shenzhen, China (email: 120010030@link.cuhk.edu.cn).

Jianwei Huang is with the School of Science and Engineering, Shenzhen Institute of Artificial Intelligence and Robotics for Society, Shenzhen Key Laboratory of Crowd Intelligence Empowered Low-Carbon Energy Network, and CSIJRI Joint Research Centre on Smart Energy Storage, The Chinese University of Hong Kong, Shenzhen, Guangdong, 518172, P.R. China (corresponding author, e-mail: jianweihuang@cuhk.edu.cn).

}
\thanks{This work is supported by the National Natural Science Foundation of China (Project 62271434), Shenzhen Science and Technology Innovation Program (Project JCYJ20210324120011032), Guangdong Basic and Applied Basic Research Foundation (Projects 2021B1515120008, 2022A1515110056), Shenzhen Key Lab of Crowd Intelligence Empowered Low-Carbon Energy Network (No. ZDSYS20220606100601002), Shenzhen Stability Science Program 2023, and the Shenzhen Institute of Artificial Intelligence and Robotics for Society.
}
}


\maketitle

\begin{abstract}
Knowledge distillation (KD) can enable collaborative learning among distributed clients that have different model architectures and do not share their local data and model parameters with others. 
Each client updates its local model using the average model output/feature of all client models as the target, known as federated KD.
However, existing federated KD methods often do not perform well when clients' local models are trained with heterogeneous local datasets. 
In this paper, we propose Federated knowledge distillation enabled by Adversarial Learning (\texttt{FedAL}) to address the data heterogeneity among clients. First, to alleviate the local model output divergence across clients caused by data heterogeneity, the server acts as a discriminator to guide clients' local model training to achieve consensus model outputs among clients through a min-max game between clients and the discriminator. 
Moreover, catastrophic forgetting may happen during the clients' local training and global knowledge transfer due to clients' heterogeneous local data. Towards this challenge, we design the less-forgetting regularization for both local training and global knowledge transfer to guarantee clients' ability to transfer/learn knowledge to/from others. 
Experimental results show that \texttt{FedAL} and its variants achieve higher accuracy than other federated KD baselines. 
\end{abstract}


\section{Introduction}\label{sec:intro}
\IEEEPARstart{A}{}single client, such as an organization or a company, usually has a limited amount of training data, which may not be sufficient for training high-quality machine learning models. To address this problem, \textit{collaborative learning} among multiple clients can be useful for producing models with better accuracy. However, there are several challenges. First, clients have their own local datasets and they may not be willing to share their raw data with others due to privacy concerns~\cite{kairouz2019advances}.
Second, clients on the edge of wireless networks often have different computation and memory resources, resulting in clients with heterogeneous models that have different architectures and parameters. Clients may not want to reveal their model architectures to other clients 
to further prevent privacy leakage~\cite{li2019fedmd, pmlr-v97-hoang19a}.  We refer to a client's model with unknown architecture to other clients as a black-box model. 
Typical collaborative learning methods, such as federated learning (FL) \cite{mcmahan2017communication, sun2021pain,sun2022profit,jiao2024provably}, involve frequent transmission of clients' local model parameters to a central server for global model aggregation. This leads to expensive communication overheads, particularly for large neural network models with millions or even billions of parameters.
Therefore, an important research direction is to enable collaborative learning among multiple clients with heterogeneous black-box models without sharing either the local data or the model architectures of clients.

\begin{figure}
	\centering
	\includegraphics[width=1\columnwidth]{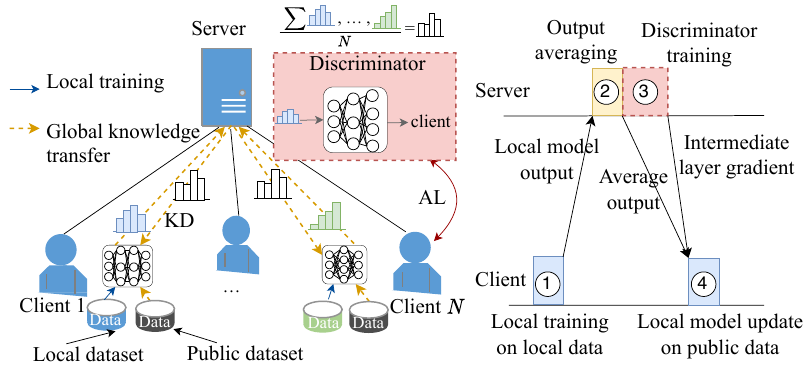}
	\caption{Framework of \texttt{FedAL}, where \texttt{FedAL} includes the components in the red dashed region that \texttt{FedMD} does not have.}
	\label{fig:fdal}
\end{figure}

Knowledge distillation (KD)~\cite{ba2014deep} can be used to transfer knowledge among multiple heterogeneous models, known as \textit{federated KD}.
Vanilla KD is an emerging technique to achieve knowledge transfer from a teacher model to a student model. In federated KD, each client can act as the teacher and share knowledge with other clients without exposing its own model or data. 
Most federated KD algorithms use the average of clients' model outputs or intermediate features, i.e., \textit{output/feature averaging}, to guide the clients' local model update. The dimension of a model output or intermediate feature is remarkably smaller than the model parameter. Thus, federated KD is more communication-efficient than FL.
A state-of-the-art method of federated KD is federated model distillation (\texttt{FedMD})~\cite{li2019fedmd}. 
In \texttt{FedMD}, as shown in Figure~\ref{fig:fdal}, each client trains its local model using the local dataset. To achieve global knowledge transfer, all clients share a small unlabeled public dataset\footnote{There are several ways of obtaining the public dataset. For example, clients can construct the public dataset by sharing a small and non-sensitive part of their data. Another way is for the server to purchase or generate a public dataset and share with clients~\cite{zhang2021fedzkt}.}. Specifically, each client uses the public data as the input and transmits the corresponding local model output to the server. Then the server broadcasts the average output of all local models to clients and each client updates its own model so that its output gets close to the average output.

\begin{figure}
	\centering
 \includegraphics[width=0.88\columnwidth]{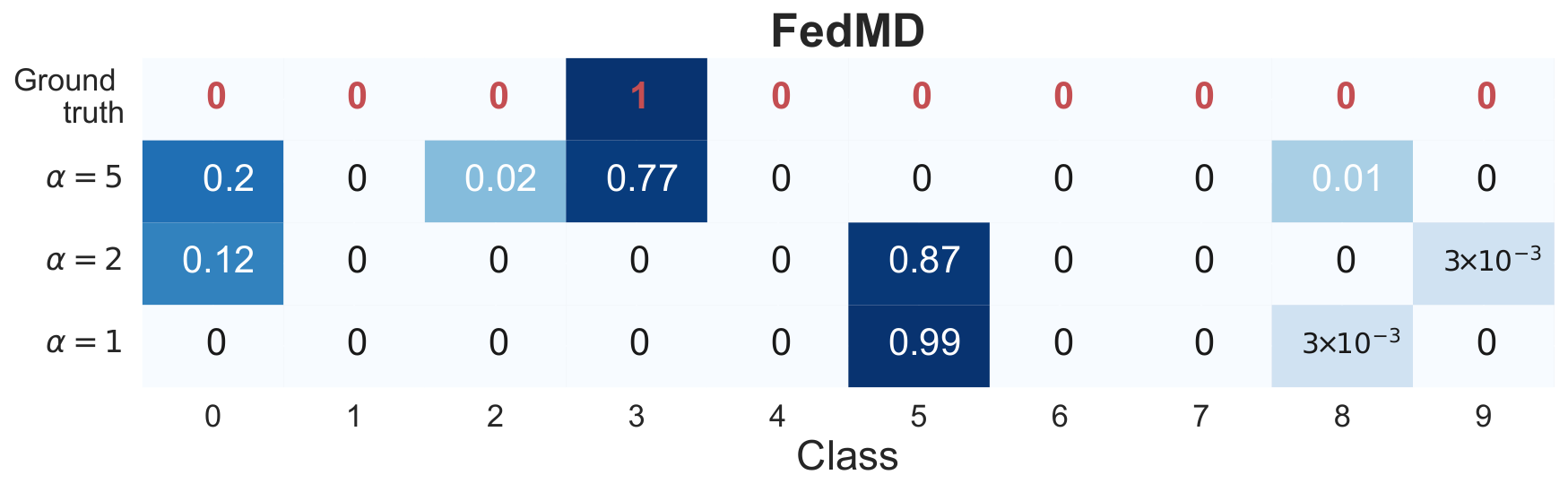}\vspace{3pt}
 \includegraphics[width=0.88\columnwidth]{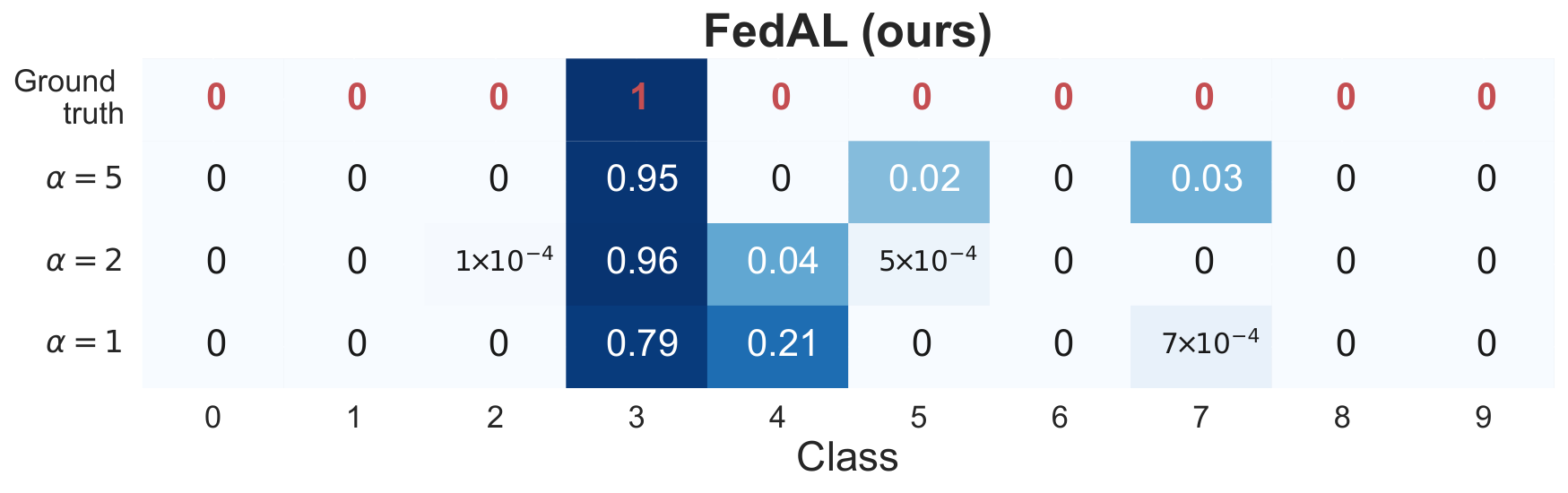}
	\caption{
An example of average model output probability distribution of the final models trained using \texttt{FedMD} and \texttt{FedAL} for data samples labeled by ``3'' of the SVHN dataset. 
 The value of $\alpha$ captures the heterogeneity of the clients' data, and the entry values in the matrix capture the predicted probabilities of the input data for different classes. 
  We distribute the whole SVHN dataset to 20 clients as their local datasets according to Dirichlet distribution with parameter $\alpha$. Smaller $\alpha$ indicates larger heterogeneity of local data across clients.} 
	\label{fig:fdal_com}
\end{figure}

A problem with \texttt{FedMD} is that it may give a low accuracy when the clients' local models are trained with heterogeneous local datasets, as shown in Figure~\ref{fig:fdal_com}, where $\alpha$ is the Dirichlet parameter of the data distribution across clients, and a smaller $\alpha$ indicates a larger data heterogeneity. When the local datasets across clients are non-IID with a relatively small value of $\alpha$, the average of all local model outputs can divert a lot from the ground-truth label (see Figure~\ref{fig:fdal_com} (top)).
In \texttt{FedMD}, clients update their local models using the average output as the only target. In this way, although the local models at different clients may give consistent prediction after the KD procedure, this prediction may be incorrect, because the average model output that is used as the training target can be quite different from the ground-truth label.
We also note that the ground-truth label is unknown to the system as we consider unlabeled public data, so we cannot use the ground-truth label as the target directly. This leads to the following question that we study in this paper.

\textbf{Key question}: \emph{How to guide clients to update their local models, so that all clients' models will consistently converge to the (unknown) ground-truth label of the input data?}

There are several challenges to answering this question. 
Firstly, since the public dataset is unlabeled, the server cannot determine the quality of the local model outputs from each client when taking the public data as input. As a result, it becomes challenging to develop a better global knowledge transfer mechanism than the simple output averaging method.
Second, when a client updates its local model using the local dataset, it may forget the global knowledge learned from other clients. Namely, the accuracy of a client's model on the global datasets decreases after local training.  
At the same time, when a client learns the knowledge of other clients, it may forget the knowledge learned from its own local dataset, therefore limiting the amount of useful knowledge it can share with other clients in the next round. 

To address these challenges, we propose an algorithm for Federated knowledge distillation enabled by Adversarial Learning (\texttt{FedAL}), as shown in Figure~\ref{fig:fdal}. First, to alleviate the local model output divergence across clients caused by data heterogeneity,
\texttt{FedAL} includes a discriminator on the server and considers each client as a generator. The adversarial learning (AL) between clients and the discriminator provides guidance for clients' local model training beyond output averaging, pushing all clients to produce the same ``correct'' model output. 
Specifically, using the received local model outputs of clients with the public dataset as input, 
the discriminator discriminates which client a model output comes from. 
According to the  discrimination results, clients update their own models to make the outputs of different clients as indistinguishable as possible.
To achieve this, the discriminator transmits the gradients of the discrimination over a client's local model output to the corresponding client. 
In addition, to address the catastrophic forgetting issue in federated KD, we apply less-forgetting (LF) regularization to both clients' local training and global knowledge transfer stages to guarantee clients' ability to transfer/learn knowledge to/from others.

The main contributions of this paper are as follows.
\begin{itemize}
	\item \textit{\texttt{FedAL} framework design}: We address the challenging problem of efficient knowledge transfer in federated KD where clients' local models are trained using heterogeneous local data. We overcome this challenge by proposing \texttt{FedAL}, which applies a discriminator to facilitate all clients producing the same high-quality model output and LF regularization to promote efficient knowledge transfer among clients. 
	\item \textit{Min-max game formulation}: To overcome the challenge of local model output divergence among clients, which cannot be effectively addressed through simple output averaging,  we formulate the AL between clients and the discriminator as a min-max game. At the equilibrium, all clients behave consistently on the same input data. 
 \item \textit{LF regularization design:} We design the LF regularization in both local training and global knowledge transfer stages of each client. When updating the local model of a client in local training (global knowledge transfer, respectively), we restrict the distance between the updated model parameter and the original model parameter obtained by the last round of global knowledge transfer (local training, respectively) of the client.
	\item \textit{Theoretical analysis}: We demonstrate the convergence error of \texttt{FedAL} for non-convex loss functions and establish its capability for achieving good generalization.
	\item \textit{Experiments}: Experimental results show that \texttt{FedAL} and its variants achieve higher accuracy than other federated KD baselines. 
\end{itemize}

\section{Related Work} 

The most popular method for training a global machine learning model among distributed clients without sharing their local data is FL~\cite{mcmahan2017communication, kairouz2019advances}. 
Most FL algorithms, e.g., federated averaging (\texttt{FedAvg})~\cite{mcmahan2017communication}, are based on distributed stochastic gradient descent (SGD). They assume that all clients have identical model architectures. 

An advantage of performing KD in the federated setting is that it supports the case where different clients may have their own models with different architectures. 
The average of clients' model outputs or encoded data features has been used in various works on federated KD~\cite{li2019fedmd,  8904164, jeong2018communication,tan2021fedproto,  lin2020ensemble}, which includes \texttt{FedMD}~\cite{li2019fedmd} and related methods.
Other works on federated KD consider generating the public data for global KD~\cite{zhu2021data, 9879661, he2020group} and using multiple history local models~\cite{yao2021local} to improve the model accuracy.
As mentioned earlier, the above works~\cite{li2019fedmd, 8904164,jeong2018communication, tan2021fedproto, lin2020ensemble, zhu2021data, 9879661, he2020group} do not work well when the local training data of clients are heterogeneous. 
There are some works utilizing KD for different purposes in FL, such as to achieve privacy  preservation~\cite{wu2021fedkd,li2020practical}, personalization~\cite{cho2021personalized,ozkara2021quped}, and one-shot FL~\cite{heinbaughdata}, which are orthogonal to our work.

 To demonstrate the key differences between our proposed method and existing works, we compared the proposed \texttt{FedAL} with the related methods in the literature, as presented in Table \ref{tab:related-work}. Specifically, we evaluated whether these methods support heterogeneous model architectures and black-box models among clients and whether they provide theoretical analysis on generalization bound and convergence.

\begin{table*}[t] \centering
    \caption{Comparison with Related Federated KD Works. }
    \label{tab:related-work}
        \begin{center}
            \begin{small}
 \begin{tabular}{m{1.5cm}m{2.3cm}m{1.2cm}m{1.2cm}m{1.2cm}m{1.2cm}m{5cm}}
		\hline
  Ref. & Method & Hetero-geneous models & Black-box  models & Genera-lization  bound & Conver-gence & Comparison with \texttt{FedAL}\tabularnewline
  \hline
  \cite{li2019fedmd} & \texttt{FedMD} & \cmark & \cmark & \xmark & \xmark & \cmark\tabularnewline
  \rowcolor{gray!15} \cite{sattler2020communication,hu2021mhat} &  \texttt{CFD},\texttt{MHAT} & \cmark& \cmark& \xmark & \xmark & \cmark \tabularnewline
   \cite{huang2022learn} & \texttt{FCCL}& \cmark& \cmark& \xmark & \xmark& \cmark\tabularnewline
   \rowcolor{gray!15} \cite{huang2023generalizable}  & \texttt{FCCL+} & \cmark& \cmark& \xmark & \xmark& \cmark\tabularnewline
 \cite{jeong2018communication, 8904164} & \texttt{FD}/
 \texttt{HFD} & \cmark & \cmark & \xmark & \xmark & Low accuracy results\tabularnewline
 \rowcolor{gray!15}\cite{tan2021fedproto} & \texttt{FedProto} & \cmark & \cmark & \xmark & \cmark & Only comparable for models with same intermediate feature dimension\tabularnewline
 \cite{he2020group} &  \texttt{FedGKT} & \cmark  & \cmark & \xmark  &  \xmark & Not comparable because it exposes true labels to the server \tabularnewline
 \rowcolor{gray!15}\cite{cho2021personalized} & \texttt{PerFed-CKT} & \cmark & \cmark & \cmark & \cmark & Not comparable because it has a different objective\tabularnewline
 \cite{zhang2021fedzkt} &  \texttt{FedZKT} & \cmark  & \xmark  & \xmark  & \xmark  & Not comparable \tabularnewline
  \rowcolor{gray!15}\cite{mcmahan2017communication} & \texttt{FedAvg} &  \xmark  & \xmark & \xmark & \cmark & Not comparable\tabularnewline
\cite{lin2020ensemble} & \texttt{FedDF} & \cmark & \xmark & \cmark & \xmark & Not comparable, but compared with a new AL extension in the paper\tabularnewline
\rowcolor{gray!15}\cite{zhu2021data} & \texttt{FedGen} & \xmark & \xmark & \cmark & \xmark & Not comparable\tabularnewline
\cite{zhu2021data} & \texttt{FedDistill+} & \cmark & \xmark & \xmark & \xmark & Not comparable\tabularnewline
\rowcolor{gray!15}\cite{9879661} & \texttt{FedFTG} & \xmark & \xmark & \xmark & \xmark & Not comparable\tabularnewline
\cite{yao2021local} &  \texttt{FedGKD} & \xmark  &  \xmark & \xmark  & \cmark  & Not comparable \tabularnewline
\rowcolor{gray!15}\cite{heinbaughdata} & \texttt{FedCVAE} & N/A & N/A & \xmark & \xmark & Not comparable because it trains a separate global model\tabularnewline
 \textbf{Ours} & \texttt{FedAL} & \cmark & \cmark & \cmark & \cmark & -\tabularnewline
\tabularnewline
\hline
	\end{tabular}
\end{small}
\end{center}
\end{table*}

For \texttt{FedCVAE} \cite{heinbaughdata}, the first two criteria are ``N/A'' because \texttt{FedCVAE} only uses the local models for creating data, where the resulting data is then used for training a different ``global'' model.
We would like to emphasize the following:
\begin{itemize}
    \item Only  \texttt{FedMD} \cite{li2019fedmd}, \texttt{FD}  \cite{jeong2018communication} / \texttt{HFD}  \cite{8904164}, \texttt{FedProto} \cite{tan2021fedproto}, \texttt{FedGKT}  \cite{he2020group}, \texttt{PerFed-CKT}  \cite{cho2021personalized}, \texttt{CFD}~\cite{sattler2020communication}, \texttt{MHAT}~\cite{hu2021mhat}, \texttt{FCCL}~\cite{huang2022learn}, and \texttt{FCCL+}~\cite{huang2023generalizable} have the same problem settings as our paper, supporting black-box models with different model architectures across clients. The black-box models imply that each client's model architecture and parameters are not shared with the server or other clients. Only the model outputs are shared. Different from \texttt{FedMD}, \texttt{CFD}~\cite{sattler2020communication} and \texttt{MHAT}~\cite{hu2021mhat} utilize unlabeled public data and local model outputs from clients to train a global model on the server. This global model output is then transmitted back to clients for KD. \texttt{FCCL}\cite{huang2022learn} and \texttt{FCCL+} \cite{huang2023generalizable} incorporate less-forgetting regularizations to enhance model performance.
    \texttt{PerFed-CKT} focuses on a different problem from our paper by stressing client personalization.  \texttt{FedGKT} shares the ground-truth labels of client's local data to the server, risking privacy leaking of users. Thus, \texttt{PerFed-CKT} and \texttt{FedGKT} are not comparable to our method.
    \item Other methods including \texttt{FedZKT}  \cite{zhang2021fedzkt}, \texttt{FedAvg} \cite{mcmahan2017communication}, \texttt{FedDF} \cite{lin2020ensemble}, \texttt{FedGen} \cite{zhu2021data}, \texttt{FedDistill+} \cite{zhu2021data}, \texttt{FedFTG} \cite{9879661}, and \texttt{FedGKD} \cite{yao2021local} expose local model parameters of clients to the server to achieve global model aggregation, whereas we do not share the clients' model architectures or parameters. \texttt{FedCVAE} \cite{heinbaughdata} obtains only one global model for all clients, whereas we aim to train multiple client models with heterogeneous architectures. Thus, these methods are for different problem setups and they are not comparable to our \texttt{FedAL} method, so we should not be expected to compare with them. However, to show the compatibility of our idea of AL in the parameter-sharing scenario, we have taken \texttt{FedDF} as an example to extend our \texttt{FedAL} method to support parameter averaging among homogeneous models. 
\end{itemize}

\begin{table*}[t] \centering
   \caption{Comparison with Related Works for LF Regularization. }
    \label{tab:related-work-lf}
        \begin{center}
            \begin{small}
 \begin{tabular}{>{\centering}m{1cm}>{\centering}m{1.5cm}>{\centering}m{1.5cm}>{\centering}m{6.5cm}>{\centering}m{2.cm}>{\centering}m{3cm}}
\hline 
Ref.  & Method & Black-box models & Loss function & Operating dataset & Refereed model\tabularnewline
\hline 
\cite{lee2021preservation} &\texttt{FedLSD} & \xmark & KD between local and global model output & Local dataset & Global model \tabularnewline
\rowcolor{gray!15}\cite{tang2023fedrad}& \texttt{FedRAD}  & \xmark & Relational KD between local and global model outputs& Local dataset & Local and global model\tabularnewline
\cite{he2022learning} &\texttt{FedSSD}  & \xmark &  MSE of class-weighted local and global model outputs & Local and public datasets& Global model\tabularnewline
\rowcolor{gray!15}\cite{xu2022acceleration} & \texttt{FedReg} & \xmark & MSE between local and global model outputs & Selected private dataset & Global model \tabularnewline
 \cite{wang2024dfrd} & \texttt{DFRD}  & \xmark & KD between local and global model outputs & Generated dataset & Global model\tabularnewline 
\rowcolor{gray!15}\cite{aljahdali2024flashback} &\texttt{Flashback} & \xmark & KD between local and global model outputs & Local and public datasets & Local and global models\tabularnewline 
\cite{kim2024federated}& - & \xmark & KD between local and global model outputs and different local model outputs & Local dataset & Local and global models  \tabularnewline
\rowcolor{gray!15}\cite{liu2023adaptive} &\texttt{FedBR} & \xmark & KD between different intermediate features & Local dataset & Local and global models \tabularnewline 
\cite{9964434} & \texttt{pFedSD}  &  \xmark & KD between personalized and current local model outputs & Local  dataset & Personalized local model \tabularnewline 
\rowcolor{gray!15}\cite{shoham1910overcoming} & \texttt{FedCurv} & \cmark & FIM weighted local model update & Local dataset & Local models\tabularnewline
\cite{huang2022learn} &\texttt{FCCL} & \cmark & KD between local and global model outputs, KD between current and the optimal local model outputs & Local dataset & Optimal and instantaneous local models   \tabularnewline 
\rowcolor{gray!15}\cite{huang2023generalizable}  &\texttt{FCCL+}  & \cmark & KD between local and global model outputs and feature similarity & Local and public dataset & Local models\tabularnewline
\hline 
	\end{tabular}
\end{small}
\end{center}
\end{table*}

Some works on standard (non-federated) KD have found AL to be effective for improving student model accuracy, for the scenario of one teacher and one student model sharing the same training dataset~\cite{xu2017training, liu2020learning, gao2020private, wang2018kdgan, wang2018adversarial}. 
However, these works use the original labeled training data of the teacher model. As a result, the methods do not apply to the federated setting with unlabeled public dataset and multiple clients.

LF regularization effectively reduces catastrophic forgetting in FL under data heterogeneity. Table \ref{tab:related-work-lf} summarizes related works, identifying models that help prevent forgetting. Most research on LF regularization, cited in references \cite{lee2021preservation,tang2023fedrad,he2022learning,xu2022acceleration,wang2024dfrd,aljahdali2024flashback,kim2024federated,liu2023adaptive,9964434}, follows a model-sharing approach where clients share models with the server or each other. This technique focuses on minimizing the distance between local and global model outputs or intermediate features to preserve global knowledge. This distance is typically measured using KL divergence (\cite{lee2021preservation,tang2023fedrad,wang2024dfrd,aljahdali2024flashback,kim2024federated,liu2023adaptive,9964434}) or mean squared error (MSE) (\cite{he2022learning,xu2022acceleration}).
In implementing LF regularization, clients generally use their local datasets (\cite{lee2021preservation,tang2023fedrad,he2022learning,aljahdali2024flashback,kim2024federated,liu2023adaptive,9964434}). Some methods enhance effectiveness by sampling local datasets (\cite{xu2022acceleration}), generating new data (\cite{wang2024dfrd}), or using public datasets (\cite{he2022learning,aljahdali2024flashback}). The global model is crucial in extracting and retaining global knowledge. Additionally, to prevent the loss of local knowledge during global model aggregation, some algorithms integrate clients' local models into the process (\cite{tang2023fedrad,aljahdali2024flashback,kim2024federated,liu2023adaptive,9964434}).

LF regularization for black-box models often involves using the Fisher Information Matrix (FIM) or the aggregated global model output for regularization. In \texttt{FedCurv} \cite{shoham1910overcoming}, clients share the FIM to reduce forgetting, although this requires substantial computational resources. \texttt{FCCL}\cite{huang2022learn} pre-trains an optimal local model on each client to aid LF regularization. \texttt{FCCL+} \cite{huang2023generalizable}  further exchanges intermediate features and local model outputs among clients to combat forgetting. However, these approaches do not address the issue of clients forgetting local dataset knowledge during global knowledge transfer.

Our work differs from the literature by stressing the heterogeneous local training data and the black-box models among clients. To the best of our knowledge, this is the first work that improves the performance of knowledge transfer among black-box heterogeneous models with different architectures and are trained using heterogeneous local datasets.

\section{Problem Formulation and Preliminaries}
\label{sec:priliminaries}
\begin{table}[ht] \centering
    \caption{Summary of notation} \label{tab:notation}
	\vskip 0.15in
        \begin{center}
            \begin{footnotesize}
	\begin{tabular}{p{2cm}p{5.5cm}}
		\hline
		Symbol & Description\tabularnewline
	\hline
 \rowcolor{gray!15}\multicolumn{2}{l}{\textbf{Clients}:} \tabularnewline
 		$\mathcal{N}$ & Set of clients \tabularnewline
 \rowcolor{gray!15} \multicolumn{2}{l}{\textbf{Datasets}:} \tabularnewline
$\mathcal{D}_n$ & The local dataset of client $n$\tabularnewline
  $\mathcal{D}$ & The global dataset\tabularnewline
  $\mathcal{P}$& Public dataset \tabularnewline
  $\mathbf{x}$ & An input data sample \tabularnewline
  $y$ & The ground-truth label of $\mathbf{x}$ \tabularnewline
  $\mathcal{K}$ & Set of data labels \tabularnewline
  \rowcolor{gray!15}\multicolumn{2}{l}{\textbf{Models}:} \tabularnewline
  $\boldsymbol{\theta}_n$ & Local model parameter of client $n$\tabularnewline
  $\boldsymbol{\Theta}$ & Set of model parameters of all clients\tabularnewline
$\boldsymbol{\Theta}_{-n}$& Set of model parameters of all clients except for client $n$\tabularnewline
$\boldsymbol{w}$ & Parameter of discriminator \tabularnewline
$f_n\left(\boldsymbol{x},\boldsymbol{\theta}_n\right)$ & Client $n$'s local model output before softmax \tabularnewline 
   $p_n\left(\boldsymbol{x},\boldsymbol{\theta}_n\right)$ & Model output probability distribution of client $n$ \tabularnewline
$f_{n,k}\left(\boldsymbol{x},\boldsymbol{\theta}_n\right)$ & The $k$th element of $f_{n}\left(\boldsymbol{x},\boldsymbol{\theta}_n\right)$\tabularnewline
$p_{n,k}\left(\boldsymbol{x},\boldsymbol{\theta}_n\right)$ & The $k$th element of $p_{n}\left(\boldsymbol{x},\boldsymbol{\theta}_n\right)$\tabularnewline
$\bar{f}_{-n}\left(\boldsymbol{x}\right)$ & Average model output of clients except for client $n$ \tabularnewline
  $\bar{p}_{-n}\left(\boldsymbol{x}\right)$ & The softmax of   $\bar{f}_{-n}\left(\boldsymbol{x}\right)$ \tabularnewline
  $h\left(p_n\left(\boldsymbol{x},\boldsymbol{\theta}_n\right),\boldsymbol{w}\right)$ & The discriminator model\tabularnewline
\rowcolor{gray!15}\multicolumn{2}{l}{\textbf{Loss functions}:} \tabularnewline
  $\ell\left(y,f_n\left(\boldsymbol{x},\boldsymbol{\theta}_n\right)\right)$ &  Loss function of client $n$ for data sample $\left\{\boldsymbol{x},y\right\}$\tabularnewline
  $\mathscr{K}\left(\cdot\right)$ & KL divergence \tabularnewline
  $\mathscr{E}\left(\cdot\right)$ & Cross-entropy loss function \tabularnewline
  $r^{\mathrm{loc}}\left(\boldsymbol{x},\boldsymbol{\theta}_n\right)$ & The LF regularization of client $n$ for local training \tabularnewline
   $r^{\mathrm{glo}}\left(\boldsymbol{x},{\boldsymbol{\theta}}_n\right)$ & The LF regularization of client $n$ for global knowledge transfer \tabularnewline
   $V_n^{\mathrm{loc}}\left(\boldsymbol{x},y,\boldsymbol{\theta}_n\right)$ & The loss function of client $n$ for local training \tabularnewline
    $V_n^{\mathrm{glo}}\left(\boldsymbol{x},\boldsymbol{\theta}_n,\boldsymbol{w}\right)$ & The loss function of client $n$ for global knowledge transfer \tabularnewline
    $\mathcal{V}_n\left(\boldsymbol{x},y,\boldsymbol{\theta}_n,\boldsymbol{w}\right)$& Overall objective of client $n$ \tabularnewline
    $U_n\left(\boldsymbol{x}, \boldsymbol{\theta}_n,\boldsymbol{w}\right)$ & Objective of the AL between client $n$ and the discriminator \tabularnewline
  $\mathcal{U}\left(\boldsymbol{x}, \boldsymbol{\Theta},\boldsymbol{w}\right)$ & Objective of the discriminator \tabularnewline
\rowcolor{gray!15}\multicolumn{2}{l}{\textbf{Others}:} \tabularnewline
$\psi\left(\cdot\right)$ & Softmax function \tabularnewline
  $E$ & Temperature in $p_n\left(\boldsymbol{x},\boldsymbol{\theta}_n\right)$ \tabularnewline
  $\tau$ & Number of consecutive iterations in each local training or global knowledge transfer stage \tabularnewline

		\hline
	\end{tabular}
\end{footnotesize}
\end{center}
\vskip -0.1in
\end{table}

In this section, we first formulate the collaborative learning problem and then introduce the framework of federated KD. We summarize the notation used in this paper as Table~\ref{tab:notation}.
\subsection{Problem Formulation}
Given a set $\mathcal{N}=\left\{1,\ldots,N\right\}$ of clients, the client $ n\in\mathcal{N}$ has a local dataset $\mathcal{D}_n$. 
The global dataset is $\mathcal{D}:=\cup_{n=1}^N \mathcal{D}_n$. 
A sample $\left\{\boldsymbol{x},y\right\}\in \mathcal{D}$ consists of an input data $\boldsymbol{x}$ and its corresponding label $y$ 
in the set $\mathcal{K}=\{1, \ldots, K\}$ of classes.
Each client $n$ trains its local model $f_n$ characterized by the parameter $\boldsymbol{\theta}_n$ using its local dataset $\mathcal{D}_n$. The output of the model before the softmax is $f_n\left(\boldsymbol{x},\boldsymbol{\theta}_n\right)\in\mathbb{R}^K$. The objective of each client $n$ is to minimize the expected prediction error, i.e., loss function, $\ell\left(y,f_n\left(\boldsymbol{x},\boldsymbol{\theta}_n\right)\right)$. 
The global loss function of all clients is $\frac{1}{N}\sum_{n=1}^N\ell\left(y,f_n\left(\boldsymbol{x},\boldsymbol{\theta}_n\right)\right)$. 

In collaborative learning, all clients collaborate to solve the following 
problem without sharing their local datasets, to minimize the expected loss function of all local models on the union of all clients' datasets:
\begin{align}\label{eq:obj_global}
\min_{\boldsymbol{\theta}_n, \forall n}\,\mathbb{E}_{\left\{\boldsymbol{x},y\right\}\in\mathcal{D}} \left[\frac{1}{N}\sum_{n=1}^N\ell\left(y,f_n\left(\boldsymbol{x},\boldsymbol{\theta}_n\right)\right)\right].
\end{align}

We can solve \eqref{eq:obj_global} using FL when all clients have an identical model architecture. 
In this case, the server can directly average the model parameters $\left\{\boldsymbol{\theta}_1,\ldots,\boldsymbol{\theta}_N\right\}$ submitted by clients (e.g., using the \texttt{FedAvg} algorithm~\cite{mcmahan2017communication}) iteratively over multiple rounds to achieve global model aggregation.
However, when clients have different local model architectures, FL is no longer applicable. Federated KD overcomes this problem. 

\subsection{Federated KD}
In general, federated KD solves problem \eqref{eq:obj_global} by alternating between local training at clients and global knowledge transfer. We explain this procedure based on the \texttt{FedMD} algorithm \cite{li2019fedmd} as follows. In local training, each client $n$ trains its local model using its local dataset.
The global knowledge transfer contains two steps, including output averaging (\textit{not} parameter averaging) on the server and KD on each client, aiming to stimulate clients to produce similar model output probability distributions that are close to the ground-truth label on the common unlabeled public dataset $\mathcal{P}$\footnote{For the unlabeled public dataset, we do not know the true label $y$ for each data sample $\boldsymbol{x}\in\mathcal{P}$. Labeling data samples is costly in practice. Thus, obtaining the unlabeled public dataset would be easier than a labeled training dataset.}. 
Clients can contribute part of their local private data to create a public dataset. Many existing federated KD efforts use this unlabeled public dataset to enhance knowledge transfer among clients, as seen in references \cite{li2019fedmd, sattler2020communication,cho2021personalized,lin2020ensemble,hu2021mhat,huang2022learn,huang2023generalizable}. Some methods generate this public dataset using clients' local models, a process distinct from our approach \cite{zhang2021fedzkt,9879661,zhu2021data}. Our FedAL algorithm is fully compatible with these data-generating techniques.

In the process of output averaging in \texttt{FedMD}, each client $n$ first sends its model output to the central server, taking the same data samples in 
$\mathcal{P}$ as input. Then, the server aggregates the model outputs of clients and broadcasts the average output to all the clients. 
Upon receiving the average model output, each client $n$ carries out KD to achieve knowledge transfer.
Specifically, let $\psi\left(f_n\left(\boldsymbol{x},\boldsymbol{\theta}_n\right),E\right)$ denote the softmax function
with temperature $E$.
To enable efficient KD,  a temperature of $E\geq 1$ is set to clearly differentiate the similarity among classes of the input data. The output probability distribution for client $n$ is
\begin{align}%
p_n\!\left(\boldsymbol{x},\boldsymbol{\theta}_n\right):=\psi\left(f_n\left(\boldsymbol{x},\boldsymbol{\theta}_n\right),E\right).
\end{align}
Let $f_{n,k}\left(\boldsymbol{x},\boldsymbol{\theta}_n\right)$ be the $k$th element of $f_{n}\left(\boldsymbol{x},\boldsymbol{\theta}_n\right)$. The $k$th element of $p_n\!\left(\boldsymbol{x},\boldsymbol{\theta}_n\right)$ is
\begin{align}
	p_{n,k}\!\left(\boldsymbol{x},\boldsymbol{\theta}_n\right)=\frac{\exp({f_{n,k}\left(\boldsymbol{x},\boldsymbol{\theta}_n\right)/E})}{\sum_{k'=1}^K \exp({{f_{n,k'}\left(\boldsymbol{x},\boldsymbol{\theta}_n\right)}/E})}.
\end{align}
Then each client can learn from others by minimizing the KL divergence, denoted by $\mathscr{K}\left(\cdot\right)$, of its output probability distribution to the average of others~\cite{Distilling}:
\begin{align} 
&\min_{\boldsymbol{\theta}_n}\,
\mathbb{E}_{\boldsymbol{x}\in\mathcal{P}} \left[\mathscr{K}\left(\bar{p}_{-n}\left(\boldsymbol{x}\right),p_n\left(\boldsymbol{x},\boldsymbol{\theta}_n\right)\right)\right], \label{eq:kl}
\end{align}
where we have 
\begin{equation}
    \bar{f}_{-n}\left(\boldsymbol{x}\right):=\frac{1}{N-1}\sum_{m=1,m\neq n}^N  f_m\left(\boldsymbol{x},\boldsymbol{\theta}_m\right),
\end{equation}
\begin{equation}
    \bar{p}_{-n}\!\left(\boldsymbol{x}\right)\!:=\!\psi\!\left(\bar{f}_{-n},E\right).
\end{equation}

We can observe from \eqref{eq:kl} that \texttt{FedMD} forces the model output of each client close to the average output of all clients. As mentioned earlier, the average output may not always be a good approximation for the true label of the input data,
especially when clients' local models are trained with heterogeneous data. This motivates us to propose the \texttt{FedAL} algorithm, to be explained next.

\section{Proposed Algorithm: \texttt{FedAL}}
In this section, we first present the framework of \texttt{FedAL}. Then, we formulate the min-max game between clients and the discriminator and propose the LF regularization. 
Finally, we formulate the objectives for the server and clients and present the \texttt{FedAL} algorithm.

\subsection{\texttt{FedAL} Framework}
We propose \texttt{FedAL}, which includes a novel discriminator at the server, as shown in Figure~\ref{fig:fdal}. Our design objective for \texttt{FedAL} is as follows:
\begin{itemize}
    \item Design the interaction between clients and the server based on AL to guide clients' local model updates. Specifically, the discriminator maximizes the model output discrepancies among clients to facilitate all clients to minimize their discrepancies in an AL fashion. We model such interactions as a min-max game. 
    \item Improve clients' ability to learn from others and transfer knowledge to others. We achieve this using LF regularization in both local training and global knowledge transfer.
\end{itemize}

\subsection{Min-Max Game Formulation} 
\label{sec:system_model}

In the interaction between clients and the discriminator, the discriminator works adversarially to stimulate all clients to mimic each other. Specifically, the discriminator distinguishes the output probability distributions from clients. We denote the discriminator as $h\left(p_n\left(\boldsymbol{x},\boldsymbol{\theta}_n\right), \boldsymbol{w}\right): \mathbb{R}^K\rightarrow \mathbb{R}^{N}$ with parameter $\boldsymbol{w}$.  Let $\mathscr{E}\left(\cdot\right)$ be the cross-entropy loss function\footnote{We use the cross-entropy loss function to evaluate the correctness of the discriminator in the formulation and experiments throughout this paper. However, the same methodology can also be used for other losses, e.g., mean square error.}. 
The discriminator's objective toward each client $n$ is to maximize its correctness of classifying client $n$'s output into the client index $n$, i.e.,
\begin{align}
&U_n\left(\boldsymbol{x}, \boldsymbol{\theta}_n,\boldsymbol{w}\right):=-
\mathscr{E}\left(n, h\left(p_n\left(\boldsymbol{x},\boldsymbol{\theta}_n\right),\boldsymbol{w}\right)\right). \label{eq:min-max-obj}
\end{align}
Clients have an opposite objective (to minimize the discriminator's correctness of classification), by producing model outputs that are as indistinguishable as possible.  For convenience, we let $\boldsymbol{\Theta}:=\left\{\boldsymbol{\theta}_n, \forall n \in \mathcal{N}\right\}$ 
be the set of model parameters of all clients.
Combining  all clients, the discriminator's overall objective is to maximize
\begin{equation}
\mathcal{U}\,\left(\boldsymbol{x},\boldsymbol{\Theta},\boldsymbol{w}\right):=\frac{1}{N}\sum_{n \in \mathcal{N}}U_n\!\left(\boldsymbol{x},\boldsymbol{\theta}_n,\boldsymbol{w}\right).
\end{equation}

Thus, we formulate the min-max game between clients and the discriminator as follows.
\begin{game} [Min-Max Game] \label{game:1}$\\$\vspace{-10pt}
	\begin{itemize}
		\item \textit{Players:}  The server (discriminator) and $N$ clients.
		\item \textit{Strategies:} Each client $n$ trains its local model to update $\boldsymbol{\theta}_n$. The server trains a discriminator to update $\boldsymbol{w}$.
		\item \textit{Payoffs:} Each client $n$ seeks to minimize $U_n\left(\boldsymbol{x}, \boldsymbol{\theta}_n,\boldsymbol{w}\right)$ in \eqref{eq:min-max-obj}; the server seeks to maximize $\mathcal{U}\,\left(\boldsymbol{x},\boldsymbol{\Theta},\boldsymbol{w}\right)$. 
	\end{itemize}
\end{game}

We first analyze the best responses of the discriminator (Lemma~\ref{lem:dis_best}) and each client (Lemma~\ref{lem:client_best}), based on which we present the equilibrium condition for the game. Throughout the paper, we put the proofs of all lemmas and theorems in our supplementary material \cite{han2023fedal}.
\begin{lemma}
	[Discriminator's best response] \label{lem:dis_best}
	Given fixed client models $\boldsymbol{\Theta}$, the discriminator's  best response choice $\boldsymbol{w}^{*}(\boldsymbol{\Theta})$ that maximizes \eqref{eq:min-max-obj} for any input sample $\boldsymbol{x}$ satisfies
	\begin{equation} \label{eq:best_dis}
	\left[h\left(p_n\left(\boldsymbol{x},\boldsymbol{\theta}_n\right),\boldsymbol{w}^{*}(\boldsymbol{\Theta})\right)\right]_n = \frac{p_n\left(\boldsymbol{x},\boldsymbol{\theta}_n\right)}{\sum_{m=1}^{N} p_m\left(\boldsymbol{x},\boldsymbol{\theta}_m\right)}, 
	\end{equation} 
	for all $ n \in \mathcal{N}$, where $\left[h\left(\cdot\right)\right]_n$ indicates the $n$th element of~$h$.
\end{lemma}

Denote $\boldsymbol{\Theta}_{-n}:=\left\{\boldsymbol{\theta}_m, \forall m\in\mathcal{N}, m\neq n\right\}$. We can obtain the best response of a client $n$ as Lemma~\ref{lem:client_best}. 
\begin{lemma}
	[Client's best response] \label{lem:client_best}
	Given fixed client models $\boldsymbol{\Theta}_{-n}$ and the discriminator $\boldsymbol{w}$, the best response choice of the client $n$ (i.e., $\boldsymbol{\theta}_n^\ast\left( \boldsymbol{\Theta}_{-n} , \boldsymbol{w}\right)$) that minimizes \eqref{eq:min-max-obj} for any input sample $\boldsymbol{x}$ satisfies \begin{equation}p_n\left(\boldsymbol{x},\boldsymbol{\theta}_n^\ast\left( \boldsymbol{\Theta}_{-n} , \boldsymbol{w}\right)\right)={\sum_{n=1}^{N} p_n\left(\boldsymbol{x},\boldsymbol{\theta}_n\right)}/{N}.
 \end{equation}
\end{lemma}

Lemma~\ref{lem:client_best} says that a client's best response is to produce a probability distribution identical to all clients' average output probability. 
Given Lemmas \ref{lem:dis_best} and \ref{lem:client_best}, we have the following theorem.
\begin{theorem}
	[Equilibrium]\label{thm:equilibrium}
	The unique equilibrium of the min-max Game~\ref{game:1} satisfies $p_n\left(\boldsymbol{x},\boldsymbol{\theta}_n^\ast\right)=p_m\left(\boldsymbol{x},\boldsymbol{\theta}_m^\ast\right), \forall n,m \in \mathcal{N}$ for all $\boldsymbol{x}$.
\end{theorem}
{
\begin{proof}
	We can prove Theorem~\ref{thm:equilibrium} by combining Lemmas~\ref{lem:dis_best} and~\ref{lem:client_best}.
\end{proof}
}
Theorem~\ref{thm:equilibrium} indicates that the discriminator helps minimize the model output discrepancies among clients. Ideally, clients produce identical output probability distributions of local models.

\subsection{Less-Forgetting Regularization}
During the local training, each client $n$ trains its local model to fit the local dataset $\mathcal{D}_n$ and increase the accuracy on $\mathcal{D}_n$. However, this will lead to the forgetting of the global knowledge, i.e., the accuracy of the client' local model on the global dataset $\mathcal{D}$ will decrease, making the whole model training process unstable. To alleviate such a forgetting behavior, we apply LF regularization~\cite{lee2021preservation} to achieve robust model training.

Let ${\boldsymbol{\theta}}_n^{*,0}$ be the local model parameter of client $n$ obtained from the last global knowledge transfer at any round.
The LF objective aims to restrict the updating distance of the model parameter with respect to ${\boldsymbol{\theta}}_n^{*,0}$, i.e., to reduce $r^{\mathrm{loc}}\left(\boldsymbol{x},\boldsymbol{\theta}_n\right)$, where
\begin{equation}
	r^{\mathrm{loc}}\left(\boldsymbol{x},\boldsymbol{\theta}_n\right):= 
 \mathscr{K}\left({p}_{n}\left(\boldsymbol{x},{\boldsymbol{\theta}}^{*,0}_n\right), p_n\left(\boldsymbol{x},\boldsymbol{\theta}_n\right)\right).
\end{equation}

Similarly, during global knowledge transfer, a client model may also forget the knowledge obtained from local training. We therefore extend LF regularization to the global knowledge transfer stage. Specifically, let $\boldsymbol{\theta}_n^{*+\frac{1}{2},0}$ be the local model parameter of client $n$ obtained from the local training. The LF objective in the global knowledge transfer stage aims to reduce the global LF regularization $r^{\mathrm{g lo}}\left(\boldsymbol{\theta}_{n}\right)$, where
\begin{align}
&r^{\mathrm{glo}}\left(\boldsymbol{x},\boldsymbol{\theta}_n\right):=
\mathscr{K}\left({p}_{n}\left(\boldsymbol{x},\boldsymbol{\theta}_{n}^{*+\frac{1}{2},0}\right), p_n\left(\boldsymbol{x},\boldsymbol{\theta}_{n}\right)\right).
\end{align}

\subsection{Objectives}
Based on the above discussion, we are ready to present the objectives of the discriminator (server) and clients.

\textbf{Objective of the discriminator (server).}
The server trains the discriminator during the global knowledge transfer stage with the objective of 
\begin{equation}\label{eq:dis_obj}
    \max_{\boldsymbol{w}}\,\mathbb{E}_{\boldsymbol{x}\in\mathcal{P}} \left[
\mathcal{U}\,\left(\boldsymbol{x},\boldsymbol{\Theta},\boldsymbol{w}\right)\right].
\end{equation}

\textbf{Objective of clients.}
In \texttt{FedAL}, each client conducts local training and global knowledge transfer in each training round. In local training, each client $n$ aims to solve the following optimization problem: 
\begin{align}\label{eq:client_obj_loc}
\!\min_{\boldsymbol{\theta}_n}\,\mathbb{E}_{\boldsymbol{x}\in\mathcal{D}_n}\left[V^{\mathrm{loc}}_n\!\left(\boldsymbol{x},y,\boldsymbol{\theta}_n\right)\right],
\end{align}
where
\begin{align}\label{eq:client_loss_loc}  
V^{\mathrm{loc}}_n\left(\boldsymbol{x},y,\boldsymbol{\theta}_n\right):=\ell\left(y,f_n\left(\boldsymbol{x},\boldsymbol{\theta}_n\right)\right)+r^{\mathrm{loc}}\left(\boldsymbol{x},\boldsymbol{\theta}_n\right).
\end{align}
During this local training stage, each client $n$  trains its local model to minimize the expected loss on the local dataset $\mathcal{D}_n$ while incorporating LF regularization to avoid extensively forgetting the knowledge learned until the end of the previous global knowledge transfer stage.

For global knowledge transfer, we define the following for each client $n$: 
\begin{align}\label{eq:client_loss_glo}  
V^{\mathrm{glo}}_n\left(\boldsymbol{x},{\boldsymbol{\theta}}_{n},\boldsymbol{w}\right) &:=  \mathscr{K}\left(\bar{p}_{-n}\left(\boldsymbol{x}\right), p_n\left(\boldsymbol{x},{\boldsymbol{\theta}}_{n}\right)\right) \nonumber \\
	& \quad\quad  + U_n\left(\boldsymbol{x},{\boldsymbol{\theta}}_{n},\boldsymbol{w}\right)+ r^{\mathrm{glo}}\left(\boldsymbol{x},\boldsymbol{\theta}_n\right).
\end{align}
Function \eqref{eq:client_loss_glo} of client $n$ consists of the KL divergence $\mathscr{K}\left(\cdot\right)$ for KD, the adversarial loss that competes with the discriminator $U_n\left(\cdot\right)$, and the LF regularization $r^{\mathrm{glo}}\left(\cdot\right)$.

Then, each client $n$ aims to solve the following optimization problem in global knowledge transfer:  
\begin{align}\label{eq:client_obj_glo}  
	\min_{\boldsymbol{\theta}_{n}}\,\mathbb{E}_{\boldsymbol{x}\in \mathcal{P}} 
 \left[V^{\mathrm{glo}}_n\left(\boldsymbol{x},{\boldsymbol{\theta}}_{n},\boldsymbol{w}\right)\right].
\end{align}
Note that \eqref{eq:client_obj_loc} and  \eqref{eq:client_obj_glo} are solved in alternating steps on the local dataset $\mathcal{D}_n$ and public dataset $\mathcal{P}$, respectively.

Overall, the objective of each client $n$ is to minimize
\begin{equation}
\mathcal{V}_n\left(\boldsymbol{x},y,\boldsymbol{\theta}_n,\boldsymbol{w}\right):=V_n^{\mathrm{loc}}\left(\boldsymbol{x},y,\boldsymbol{\theta}_n\right)+V_n^{\mathrm{glo}}\left(\boldsymbol{x},\boldsymbol{\theta}_n,\boldsymbol{w}\right).
\end{equation}
\textbf{Remark}: 
We note that \eqref{eq:client_loss_glo} includes both the KL divergence with respect to the average model output and the client's objective in AL, in addition to LF regularization. As discussed in Section~\ref{sec:intro}, we know that only using output averaging, i.e., $\mathscr{K}\left(\bar{p}_{-n}\left(\boldsymbol{x}\right), p_n\left(\boldsymbol{x},{\boldsymbol{\theta}}_{n}\right)\right)$, can lead to incorrect predictions, when the clients’ local models are trained with heterogeneous local datasets (i.e., small $\alpha$) (Figure~\ref{fig:fdal_com}). The reason is that the average model output of clients can be quite different from the ground-truth label in this case. However, if we only use the AL term, i.e., $U_n\left(\boldsymbol{x},{\boldsymbol{\theta}}_{n},\boldsymbol{w}\right)$, in \eqref{eq:client_loss_glo}, we do not know what should be the target output of local models because the public dataset $\mathcal{P}$ is unlabeled. 
\texttt{FedAL} incorporates adversarial learning (AL) between clients and the server (discriminator) in addition to output averaging. In essence, when the local models are very different because they were trained using heterogeneous local datasets, the average output of models may be biased towards some specific class(es). If this happens, it is beneficial to first fine tune the models so that their outputs become more aligned with each other, which is achieved by the AL procedure of \texttt{FedAL}. Intuitively, this alignment removes the bias of local models. After such bias is removed, the average of local model outputs becomes closer to the ground-truth class label. Therefore, \texttt{FedAL} gives better performance than \texttt{FedMD}, as shown by the experiment results in Section \ref{subsec:perm_homo}.

\begin{algorithm}[tb]
\small
    \caption{Federated knowledge distillation enabled by Adversarial Learning (\texttt{FedAL})} 
        \label{alg:fdal} 
 {
    \KwIn{$\eta_{l}^t, \eta_{d}^t, \forall t, \tau, T$ }\label{algline:input}
    \KwOut{$\boldsymbol{\theta}_n, \forall n \in \mathcal{N}$}\label{algline:output}	
    \SetKwFor{EachClient}{each client $n \in \mathcal{N}$:}{}{}
    \SetKwFor{TheServer}{the server:}{}{}

    Initialize $\boldsymbol{w}^{0,0}, {\boldsymbol{\theta}}_n^{0,0}, \forall n \in \mathcal{N}$;
    
    \For{$t = 0, \ldots,  T-1$}
    {
    \vspace{3pt}

    \nonl \textbf{\underline{Local training}} 
    
    \EachClient{}{
     \For{$i=0,\ldots,\tau-1$}
     {
    Randomly sample a mini-batch of data~$\mathcal{D}^{t,i}_n$; \label{algline:local_minibatch}
    
    ${\boldsymbol{\theta}}_{n}^{t,i+1}\!\leftarrow\!{\boldsymbol{\theta}}_n^{t,i}\!-\!\frac{\eta_{l}^t\!}{\left|\mathcal{D}_n^{t,i}\right|}\sum_{\left\{\boldsymbol{x},y\right\}\in\mathcal{D}_n^{t,i}}\!\nabla\!V^{\mathrm{loc}}_n\!\left(\boldsymbol{x},y,{\boldsymbol{\theta}}_n^{t,i}\right)$;\label{algline:local_training}
    }

    ${\boldsymbol{\theta}}_{n}^{t+\frac{1}{2},0} \leftarrow \boldsymbol{\theta}_{n}^{t,\tau}$; \label{algline:local_training_end}
    }
    \vspace{3pt}
    
    \nonl \textbf{\underline{Global knowledge transfer}} 
    
    \For{$i=0,\ldots,\tau-1$}
     {
    Randomly sample a mini-batch of public data~$\mathcal{P}^{t,i}$; \label{algline:global_minibatch}
    
    \EachClient{}{
    Send {$\Big\{f_n\Big(\boldsymbol{x}, {\boldsymbol{\theta}}_{n}^{t+\frac{1}{2},i}\Big), \forall \boldsymbol{x} \in \mathcal{P}^{t,i} \Big\} $} to the server; \label{algline:upload}
    }
    
    \TheServer{}{
    $\boldsymbol{w}^{t,i+1}\!\leftarrow\!\boldsymbol{w}^{t,i}\!+\!\frac{\eta_{d}^t}{\left|\mathcal{P}^{t,i}\right|}\sum_{\boldsymbol{x}\in \mathcal{P}^{t,i}}\nabla  U_n\left(\boldsymbol{x},{\boldsymbol{\theta}}_{n}^{t+\frac{1}{2},i}\!,\boldsymbol{w}^{t,i}\right)$; \label{algline:dis}            

    \For{$\boldsymbol{x}\in \mathcal{P}^{t,i}$}{
        
        $\bar{f}^{t+\frac{1}{2},i}\left(\boldsymbol{x}\right)\!\leftarrow\!{\sum_{n\in\mathcal{N}}\! f_n\!\Big(\!\boldsymbol{x},\!{\boldsymbol{\theta}}_{n}^{t+\frac{1}{2},i}\!\Big)}/{\left|\mathcal{N}\right|}$; \label{algline:output_averaging}
        
        Send $\bar{f}^{t+\frac{1}{2},i}\!\left(\boldsymbol{x}\right),\!\nabla_{f_n}\!U_n\Big(\!\boldsymbol{x},{\boldsymbol{\theta}}_{n}^{t+\frac{1}{2},i},\boldsymbol{w}^{t,i+1}\!\Big)$ to client $n, \forall n \in \mathcal{N}$; \label{algline:download}
        }
    }
    
    \EachClient{}{
    ${\boldsymbol{\theta}}_{n}^{t+\frac{1}{2},i+1} \leftarrow {\boldsymbol{\theta}}_{n}^{t+\frac{1}{2},i} - \frac{\eta_{l}^t}{\left|\mathcal{P}^{t,i}\right|}\sum_{\boldsymbol{x}\in \mathcal{P}^{t,i}}\!\nabla V^{\mathrm{glo}}_n\!\Big(\!\boldsymbol{x},{\boldsymbol{\theta}}_{n}^{t+\frac{1}{2},i},\boldsymbol{w}^{t,i+1}\!\Big) $; \label{algline:client-global}
    }
    }
    ${\boldsymbol{w}}^{t+1,0} \leftarrow {\boldsymbol{w}}^{t,\tau}$;\label{algline:dis_next}
    
    \EachClient{}{
    ${\boldsymbol{\theta}}_{n}^{t+1,0} \leftarrow {\boldsymbol{\theta}}_{n}^{t+\frac{1}{2},\tau}$;\label{algline:client_next}
    }
    
    }
}
 \end{algorithm}

\subsection{\texttt{FedAL} Algorithm} 

\begin{figure}
	\centering
	\includegraphics[width=0.9\columnwidth]{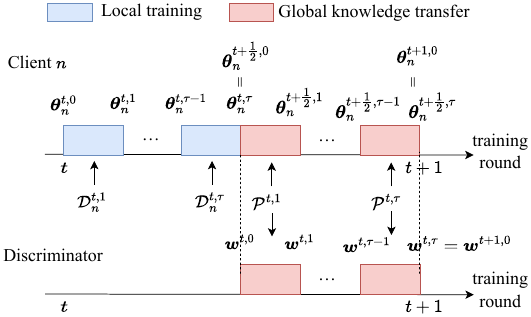}
	\caption{Model parameter updates of \texttt{FedAL} in each training round.}
	\label{fig:para_update}
\end{figure}

We design the \texttt{FedAL} algorithm considering the min-max game (between clients and the discriminator) and the LF regularization to solve the problems \eqref{eq:dis_obj}, \eqref{eq:client_obj_loc}, and \eqref{eq:client_obj_glo}.

Algorithm~\ref{alg:fdal} describes the overall process of \texttt{FedAL}, 
where clients perform local training 
and global knowledge transfer 
iteratively and alternately. Clients and the discriminator update their models with step sizes $\eta_{l}$ and $\eta_{d}$, respectively. 
Figure \ref{fig:para_update} depicts the model parameter update process of client $n$ and the discriminator for one training round of \texttt{FedAl}.
At the beginning of each training round $t$, client $n$'s local model parameter is $\boldsymbol{\theta}_n^{t,0}$. The discriminator's model parameter is $\boldsymbol{w}^{t,0}$. 
There are one or multiple consecutive iterations, denoted by $\tau\geq 1$ in every local training stage and global knowledge transfer stage.

\textbf{Local training.} In each iteration $i$,
each client $n$ randomly samples a mini-batch of local data  $\mathcal{D}^{t,i}_n\subset \mathcal{D}_n$ (line \ref{algline:local_minibatch}) to carry out local training (line~\ref{algline:local_training}). Each client performs local training for $\tau$ iteration  
and updates its model parameter to $\boldsymbol{\theta}_n^{t,\tau}$, which is equal to the initial model parameter for global knowledge transfer, denoted by $\boldsymbol{\theta}_n^{t+\frac{1}{2},0}$ (line \ref{algline:local_training_end}).

\textbf{Global knowledge transfer.} There are also $\tau$ consecutive iterations in each global knowledge transfer stage.  
In each iteration $i$ of the global knowledge transfer in training round $t$, the mini-batch of public data sampled randomly for model training is $\mathcal{P}^{t,i}$ (line \ref{algline:global_minibatch}). 
Note that $\mathcal{P}^{t,i}$ should be consistent among all the clients in $\mathcal{N}$, which can be achieved by specifying the same sampling random seed for clients. The client's model update for minimizing $U_n$ works in a splittable way. To minimize $U_n$, each client $n$ updates its local model parameterized by $\boldsymbol{\theta}_n^{t+1/2,i}$ using the gradient of $U_n$ with respect to $\boldsymbol{\theta}_n^{t+1/2,i}$. According to the chain rule, this gradient can be expressed as the following:
\begin{align}
&\nabla_{\boldsymbol{\theta}_n}U_n\left(\boldsymbol{x},\boldsymbol{\theta}_n^{t+\frac{1}{2},i},\boldsymbol{w}^{t,i+1}\right) \nonumber\\&:=\left[\nabla_{\boldsymbol{\theta}_n^{t+\frac{1}{2},i}}f_n\left(\boldsymbol{x},\boldsymbol{\theta}_n^{t+\frac{1}{2},i}\right)\right]^T\nabla_{f_n}U_n\left(\boldsymbol{x},\boldsymbol{\theta}_n^{t+\frac{1}{2},i},\boldsymbol{w}^{t,i+1}\right),
\end{align}
where $U_n$ is defined in  \eqref{eq:min-max-obj}.
By the definition of $U_n$, the gradient of $U_n$ with respect to $\boldsymbol{\theta}_n^{t+1/2,i}$ depends on client $n$'s model output probability $p_n$, which is the softmax of its model output $f_n$. Thus, in the forward process, each client needs to transmit its local model output on the public dataset $\mathcal{P}^{t,i}$ to the server (line~\ref{algline:upload}).

Upon receiving clients' model outputs, the server first updates the discriminator's parameter from $\boldsymbol{w}^{t,i}$ to $\boldsymbol{w}^{t,i+1}$ according to SGD by feeding the model outputs of clients into the discriminator $h\left(p_n\left(\boldsymbol{x},\boldsymbol{\theta}_n^{t+1/2,i}\right),\boldsymbol{w}^{t,i}\right)$  (line \ref{algline:dis}). The server then calculates $U_n$ based on the discriminator's output and the client index $n$.
To achieve back-propagation of the $U_n$ to clients' local models, the server further calculates the gradient of $U_n$ over the client's model output $\nabla_{f_n}U_n\left(\boldsymbol{x},\boldsymbol{\theta}_n^{t+1/2,i},\boldsymbol{w}^{t,i+1}\right)$ and transmits it along with the  average output $\bar{f}^{t+\frac{1}{2},i}\left(\boldsymbol{x}\right):={\sum_{n\in\mathcal{N}} f_n\left(\boldsymbol{x},{\boldsymbol{\theta}}_n^{t+\frac{1}{2},i}\right)}/{\left|\mathcal{N}\right|}$  to client $n$ (line~\ref{algline:download}). 

After that, each client executes KD using $\bar{f}^{t+\frac{1}{2},i}\left(\boldsymbol{x}\right)$ and proceeds back-propagation of $\nabla_{f_n}U_n\left(\boldsymbol{x},{\boldsymbol{\theta}}_n^{t+\frac{1}{2},i},\boldsymbol{w}^{t,i+1}\right)$ over its model. Thus, the client obtains the full gradient according to the chain rule.
The full gradient is used to update its local model parameter of the client, resulting in the updated parameter $\boldsymbol{\theta}_n^{t+\frac{1}{2},i+1}$ (line~\ref{algline:client-global}). 
Clients and the discriminator carry out the global knowledge transfer for $\tau$ iterations and obtain the model parameter ${\boldsymbol{\theta}}_n^{t+1,0}$ and $\boldsymbol{w}^{t+1,0}$ for client $n$ and discriminator, respectively, to start the next training round (lines \ref{algline:dis_next} and \ref{algline:client_next}).

It is important to note that the clients \textit{do not }have the discriminator locally. Instead, the client updates are computed across both clients and the server according to the chain rule in gradient computation.

\section{Theoretical Analysis}
We analyze the performance of \texttt{FedAL} in terms of communication overhead, generalization bound, 
and convergence. We reveal the effect of \texttt{FedAL} on improving the generalization bound and propose the first convergence error bound on federated KD.

\subsection{Communication Overhead} We analyze the communication overhead of \texttt{FedAL} for its upstream and downstream transmissions. 
For the upstream transmission from clients to the server, each client uploads its model output over $\mathcal{P}^{t,i}$ to the server with transmitted data size of $K\cdot\left|\mathcal{P}^{t,i}\right|$ in each iteration. 
For the downstream transmission, the server sends the average model outputs $\bar{f}^{t+\frac{1}{2},i}$ and the gradient $\nabla_{f_n}U_n\left(\boldsymbol{x},{\boldsymbol{\theta}}_n^{t+\frac{1}{2},i},\boldsymbol{w}^{t,i+1}\right)$ to each client $n$. Thus, the total data size that the server transmits to each client is $2K\cdot\left|\mathcal{P}^{t,i}\right|$.

\subsection{Generalization Bound}
The \emph{generalization bound} measures the difference between the average training loss of clients that we observe when running the algorithm and the expected global loss in \eqref{eq:obj_global} that we aim to minimize. 
We analyze the generalization bound of \texttt{FedAL} taking the cross-entropy loss function as an example \footnote{It is worth noting that we use the cross-entropy loss function to illustrate the classification tasks.  However, our analysis applies to more loss functions, e.g., linear regression, absolute loss, and support vector machine (SVM) leveraging \eqref{eq:prediction_dif}.}, i.e., using $\mathscr{E}\left(y,f_n\left(\boldsymbol{x},\boldsymbol{\theta}_n\right)\right)$ to replace $\ell\left(y,f_n\left(\boldsymbol{x},\boldsymbol{\theta}_n\right)\right)$ in~\eqref{eq:client_loss_loc}. 

For convenience, we use $f_n$ and $p_n$  for short of $f_n\left(\boldsymbol{x},\boldsymbol{\theta}_n\right)$ and $p_n\left(\boldsymbol{x},\boldsymbol{\theta}_n\right)$, respectively.
Let $\mathscr{E}_\mathcal{D}\left(\cdot\right)$ and $\mathscr{K}_\mathcal{D}\left(\cdot\right)$ be the expected KL divergence and cross-entropy loss over dataset $\mathcal{D}$, respectively, where we can replace the model and dataset as needed, e.g., $\mathscr{E}_{\mathcal{D}_n}\left(f_n\right)$ and $\mathscr{K}_{\mathcal{P}}\left(p_n,\frac{1}{N}\sum_{n=1}^N{p_n}\right)$. 
We denote the global optimal model as 
\begin{equation}
    f^{*}=\arg\min_{f}\mathscr{E}_{\mathcal{D}}\left(f\right).
\end{equation}
We further define the optimal loss for the combined local dataset $\mathcal{D}_n$ and public dataset $\mathcal{P}$  as 
\begin{equation}
\lambda_{n}^{*}=\mathscr{E}_{\mathcal{D}_n}\left(f^{*}\right)+\mathscr{E}_{\mathcal{P}}\left(f^{*}\right).
\end{equation}
Denote $d_{\mathcal{H} \Delta \mathcal{H}} (\mathcal{P}, \mathcal{D}_n)$ as the domain discrepancy between datasets (domains) $\mathcal{P}$ and $\mathcal{D}_n$, which measures the level of discrepancy between two models on both domains ~\cite{ben2010theory}. Here, with a slight abuse of notation, we use $\mathcal{D}$ to denote both a dataset and the distribution of all samples in the dataset.
By conducting the AL between clients and the discriminator, we give the following assumption.
\begin{assumption}[Sufficient adversarial training] \label{asmp:al}
There exists a positive value $\zeta_n$ for client $n$, such that the final model of the client obtained from the AL between clients and the server satisfies
\hspace{-3pt}
	\begin{align}\label{eq:assum} 
		\mathscr{K}_{\mathcal{P}}\left(p_n,\frac{1}{N}\sum_{n=1}^N{p_n}\right) 
		\leq\zeta_n, \forall n \in \mathcal{N}.
	\end{align}
\end{assumption}
Assumption \ref{asmp:al} measures the distance between client $n$'s model output probability and the average probability of all clients. The adversarial learning between clients and server reduces the distance bound $\zeta_n, \forall n$.

We begin by formulating the knowledge transfer between two clients to establish the generalization bound of \texttt{FedAL}.

\begin{lemma}\label{lem:consensus}
	[Knowledge transfer from clients $n$ to $m$]
	Consider two clients $n$ and $m$ with local dataset distributions $\mathcal{D}_n$ and $\mathcal{D}_m$, respectively, based on which clients train their local models $f_n$ and $f_m$.  Under Assumption~\ref{asmp:al}, we have
	\begin{equation}
		\begin{aligned}\label{eq:bound_kl}
		\mathscr{E}_{\mathcal{D}_n}\left(f_m\right) & \leq \mathscr{E}_{\mathcal{D}_n}(f_n) + \sqrt{2\log 2 \zeta_n} + \sqrt{2\log 2 \zeta_m}\\ &
		 + d_{\mathcal{H} \Delta \mathcal{H}} (\mathcal{P}, \mathcal{D}_n) + 2\lambda_n^{*}, \forall n,m \in \mathcal{N}. 
		\end{aligned} 
	\end{equation}
\end{lemma}

The proof of Lemma~\ref{lem:consensus} is in our supplementary material \cite{han2023fedal}. Lemma~\ref{lem:consensus} shows the efficiency of knowledge transfer from clients $n$ to $m$.

To analyze the generalization bound of \texttt{FedAL}, we apply Lemma~\ref{lem:consensus} to multiple clients. In practice, client $n$ executes local training on $\hat{\mathcal{D}}_n$ with total $\phi$ samples, which is consistent among clients. To evaluate the theoretical model performance on $\mathcal{D}_n$ while training using $\hat{\mathcal{D}}_n$, we denote $\mathcal{F}$ as a hypothesis (model) class with finite VC dimension~\cite{Understanding} and
let $\epsilon_{\mathcal{F}}$ be the growth function of $\mathcal{F}$. 
Then, we have Theorem~\ref{thm:generalization_bound}.
\begin{theorem}[Generalization bound]\label{thm:generalization_bound} 
	Given $\delta\in(0,1)$, with at least a probability $1-\delta$, 
	the expected global loss of all client models 
	is bounded by:
		\begin{align} 
			&\frac{1}{N}\!\sum_{n=1}^N\!\mathscr{E}_{\mathcal{D}}\left(f_n\right) 
		\!\leq\!\frac{1}{N}\!\sum_{n=1}^{N}  \mathscr{E}_{\hat{\mathcal{D}}_n}\left(f_n\right)\!+\!\frac{\!4\!+\!\sqrt{ \log ( \epsilon_{\mathcal{F}} (2\phi) ) } }{ \delta \sqrt{ 2\phi }} \nonumber\\
			&\!+\!\frac{N\!-\!1}{N^2}\!\sum_{n=1}^N\left[2\sqrt{2\log 2 \zeta_n}\!+\!d_{\mathcal{H} \Delta \mathcal{H}} (\mathcal{P},\! \mathcal{D}_n)\!+\!2\lambda_{n}^{*}\right].  \label{eq:bound_FDAL}
		\end{align}

\end{theorem}

The proof of Theorem~\ref{thm:generalization_bound} is in our supplementary material \cite{han2023fedal}. Theorem~\ref{thm:generalization_bound} says that we can achieve a good generalization bound of \texttt{FedAL} by carrying out AL to obtain a small $\zeta_n$ for each $n$. 

\textbf{Remark}: 
For \texttt{FedMD}, its generalization bound is~\cite{lin2020ensemble,zhu2021data} 
\begin{align} \label{eq:gen_fedmd}
			&\frac{1}{N}\sum_{n=1}^N \mathscr{E}_{\mathcal{D}}\left(f_n\right) 
			\leq  \frac{1}{N} \sum_{n=1}^{N}  \mathscr{E}_{\hat{\mathcal{D}}_n}\left(f_n\right)\!+\!\frac{ 4 + \sqrt{ \log ( \epsilon_{\mathcal{F}} (2\phi) ) } }{ \delta \sqrt{ 2\phi }} \nonumber\\
   &\ \ \ \ \ + \frac{1}{N}\sum_{n=1}^N\left[\frac{1}{2}d_{\mathcal{H} \Delta \mathcal{H}} (\mathcal{D}, \mathcal{D}_n)+\lambda_{n}^{*}\right].  
\end{align}

Different from the generalization bound \eqref{eq:gen_fedmd} for federated KD, the generalization bound of \texttt{FedAL} in Theorem~\ref{thm:generalization_bound} is tightly related to domain discrepancy between clients' local and public datasets $d_{\mathcal{H} \Delta \mathcal{H}} (\mathcal{P}, \mathcal{D}_n)$. Therefore, for the first time  in the literature, we reveal the effect of the public dataset in the generalization bound of federated KD.

\subsection{Convergence Analysis}
We first introduce several assumptions on the client models and the discriminator, based on which we analyze the convergence of \texttt{FedAL}. 

 Denote $g^{\mathrm{l}}_n\left(\boldsymbol{\theta}_n\right)$, $g^{\mathrm{k}}_n\left(\boldsymbol{\theta}_n\right)$, $g^{\mathrm{u}}_n\left(\boldsymbol{\theta}_n\right)$, $g^{\mathrm{rl}}_n\left(\boldsymbol{\theta}_n\right)$, $g^{\mathrm{rg}}_n\left(\boldsymbol{\theta}_n\right)$ as the stochastic gradients over $\boldsymbol{\theta}_n$ of losses $\ell\left(y,f_n\left(\boldsymbol{x},\boldsymbol{\theta}_n\right)\right)$, $\mathscr{K}\left(\bar{p}_{-n}\left(\boldsymbol{x}\right), p_n\left(\boldsymbol{x},{\boldsymbol{\theta}}_{n}\right)\right)$, $U_n\left(\boldsymbol{x},{\boldsymbol{\theta}}_{n},\boldsymbol{w}\right)$ (i.e., $-\mathscr{E} \left(n,  h\left(p_n\left(\boldsymbol{\theta}_n\right),\boldsymbol{w}\right)\right)$), $r^{\mathrm{loc}}\left(\boldsymbol{x},\boldsymbol{\theta}_n\right)$, and $r^{\mathrm{glo}}\left(\boldsymbol{x},\boldsymbol{\theta}_n\right)$, respectively, for a mini-batch of data.
Let $\varphi_n\left(\boldsymbol{w}\right)$ be the stochastic gradient of $U_n\left(\boldsymbol{x},{\boldsymbol{\theta}}_{n},\boldsymbol{w}\right)$ over $\boldsymbol{w}$.
We have the following assumption for client models and the discriminator, which are commonly used in the literature~\cite{bistritz2020distributedDF, cho2021personalized}.

\begin{assumption}[Lipschitz properties of clients]
	\label{asmp:local_models}
  For each client $n$, its model $f_n\left(\boldsymbol{x}, \boldsymbol{\theta}_n\right)$, output probability distribution $p_n\left(\boldsymbol{x}, \boldsymbol{\theta}_n\right)$, and the loss function $\ell\left(y,f_n\left(\boldsymbol{x},\boldsymbol{\theta}_n\right)\right)$ are Lipschitz continuous in $\boldsymbol{\theta}_n$.
		The model output $f_n\left(\boldsymbol{x}, \boldsymbol{\theta}_n\right)$ is also Lipschitz smooth in $\boldsymbol{\theta}_n$. 

\end{assumption}

\begin{assumption}[Lipschitz properties of the discriminator]
	\label{asmp:discriminator}
	The output of the discriminator model,  i.e., $h\!\left(p_n\left(\boldsymbol{x}, \boldsymbol{\theta}_n\right), \boldsymbol{w}\right)$, is Lipschitz continuous in $p_n$ and  $\boldsymbol{w}$.
\end{assumption}
\begin{assumption}[Unbiased  gradients] \label{asm:Unbiased_stochastic_gradients} For a mini-batch $\mathcal{B}$, the stochastic gradients are unbiased, 
\begin{align}
    &\mathbb{E}_{\mathcal{B}} \left[g_n^{\mathrm{l}}\left(\boldsymbol{\theta}_n\right)\right]  =\nabla_{\boldsymbol{\theta}_n} \ell\left(y,f_n\left(\boldsymbol{x},\boldsymbol{\theta}_n\right)\right),\forall n,\\
    &\mathbb{E}_{\mathcal{B}} \left[g_n^{\mathrm{k}}\left(\boldsymbol{\theta}_n\right)\right]  =\nabla_{\boldsymbol{\theta}_n} \mathscr{K}\!\left(\bar{p}_{-n}\!\left(\boldsymbol{x}\right)\!,\!p_n\!\left(\boldsymbol{x}\!,\!{\boldsymbol{\theta}}_{n}\right)\right),\forall n,\\
    &\mathbb{E}_{\mathcal{B}} \left[g_n^{\mathrm{u}}\left(\boldsymbol{\theta}_n\right)\right]  =\nabla_{\boldsymbol{\theta}_n} U_n\left(\boldsymbol{x},{\boldsymbol{\theta}}_{n},\boldsymbol{w}\right),\forall n,\\
    &\mathbb{E}_{\mathcal{B}} \left[g_n^{\mathrm{rl}}\left(\boldsymbol{\theta}_n\right)\right]  =\nabla_{\boldsymbol{\theta}_n} r^{\mathrm{loc}}\left(\boldsymbol{x},\boldsymbol{\theta}_n\right),\forall n,\\
    &\mathbb{E}_{\mathcal{B}} \left[g_n^{\mathrm{rg}}\left(\boldsymbol{\theta}_n\right)\right]  =\nabla_{\boldsymbol{\theta}_n} r^{\mathrm{glo}}\left(\boldsymbol{x},\boldsymbol{\theta}_n\right),\forall n,\\
    &\mathbb{E}_{\mathcal{B}} \left[\varphi_n\left(\boldsymbol{w}\right)\right]  =\nabla_{\boldsymbol{w}} U_n\left(\boldsymbol{x},{\boldsymbol{\theta}}_{n},\boldsymbol{w}\right),\forall n.
\end{align}
\end{assumption}

\begin{assumption}[Bounded variance  of stochastic gradients]\label{asm:bounded_variance_o} There exists constants $\sigma_{\mathrm{l}}$, $\sigma_{\mathrm{k}}$,  $\sigma_{\mathrm{u}}$, $\sigma_{\mathrm{rl}}$, $\sigma_{\mathrm{rg}}$, and $\sigma_{\mathrm{n}}$, such that the variance of the stochastic gradients of loss functions on a mini-batch ${\mathcal{B}}$ are bounded as
\begin{align}
    &\mathbb{E}_{\mathcal{B}}\left[\left\Vert \nabla_{\boldsymbol{\theta}_n}\ell\left(y,f_n\left(\boldsymbol{x},\boldsymbol{\theta}_n\right)\right)- g^{\mathrm{l}}_n\left(\boldsymbol{\theta}_n\right)\right\Vert^2\right]\leq \sigma_{\mathrm{l}}^2,\forall n,\\
    &\mathbb{E}_{\mathcal{B}}\!\left[\!\left\Vert\! \nabla_{\boldsymbol{\theta}_n}\!\mathscr{K}\!\left(\bar{p}_{-n}\!\left(\boldsymbol{x}\right)\!,\!p_n\!\left(\boldsymbol{x}\!,\!{\boldsymbol{\theta}}_{n}\right)\right)\!-\! g^{\mathrm{k}}_n\!\left(\boldsymbol{\theta}_n\right)\!\right\Vert^2\!\right]\!\leq\!\sigma_{\mathrm{k}}^2,\forall n,\\
    &\mathbb{E}_{\mathcal{B}}\left[\left\Vert \nabla_{\boldsymbol{\theta}_n}U_n\left(\boldsymbol{x},{\boldsymbol{\theta}}_{n},\boldsymbol{w}\right)- g^{\mathrm{u}}_n\left(\boldsymbol{\theta}_n\right)\right\Vert^2\right]\leq \sigma_{\mathrm{u}}^2,\forall n,\\
     &\mathbb{E}_{\mathcal{B}}\left[\left\Vert \nabla_{\boldsymbol{\theta}_n}r^{\mathrm{loc}}\left(\boldsymbol{x},\boldsymbol{\theta}_n\right)- g^{\mathrm{rl}}_n\left(\boldsymbol{\theta}_n\right)\right\Vert^2\right]\leq \sigma_{\mathrm{rl}}^2,\forall n,\\
     &\mathbb{E}_{\mathcal{B}}\left[\left\Vert \nabla_{\boldsymbol{\theta}_n}r^{\mathrm{glo}}\left(\boldsymbol{x},\boldsymbol{\theta}_n\right)- g^{\mathrm{rg}}_n\left(\boldsymbol{\theta}_n\right)\right\Vert^2\right]\leq \sigma_{\mathrm{rg}}^2,\forall n,\\
   &\mathbb{E}_{\mathcal{B}}\left[\left\Vert \nabla_{\boldsymbol{w}}U_n\left(\boldsymbol{x},{\boldsymbol{\theta}}_{n},\boldsymbol{w}\right)- \varphi_n\left(\boldsymbol{w}\right)\right\Vert^2\right]\leq \sigma_{\mathrm{n}}^2,\forall n.
\end{align}
\end{assumption}
Note that we do not assume the convexity of  $\ell\left(y,f_n\left(\boldsymbol{x},\boldsymbol{\theta}_n\right)\right)$ indicating that our results hold for non-convex loss functions.

With a slight abuse of notation, we let $\mathcal{V}_n\left(\boldsymbol{\theta}_n\right)$  and $\mathcal{U}\left(\boldsymbol{w}\right)$ be equivalent to $\mathcal{V}_n\left(\boldsymbol{x},y,\boldsymbol{\theta}_n, \boldsymbol{w}\right)$ and $ \mathcal{U}\left(\boldsymbol{x},\boldsymbol{\Theta},\boldsymbol{w}\right)$, respectively.  Denote $\mathcal{F}\,\left(\,\boldsymbol{\Theta},\boldsymbol{w}\,\right)\,:=\,\left\{\,\mathcal{V}_1\left(\boldsymbol{\theta}_1\right)\,, \ldots, \mathcal{V}_N\left(\boldsymbol{\theta}_N\right), \mathcal{U}\left(\boldsymbol{w}\right)\right\}$. To compute the convergence error of \texttt{FedAL}, we first show that $V^{\mathrm{loc}}_n\left(\boldsymbol{\theta}_n\right)$, $V^{\mathrm{glo}}_n\left(\boldsymbol{\theta}_n\right)$, $\mathcal{U}\left(\boldsymbol{w}\right)$, and $\mathcal{F}\left(\boldsymbol{\Theta},\boldsymbol{w}\right)$ are Lipschitz-smooth functions with constants $L_{\rm v}$, $L_{\rm w}$, $L_{\mathrm{u}}$, and $L_{\mathrm{o}}$, respectively.
Then, we bound the variance of the stochastic gradients of  $V^{\mathrm{loc}}_n\left(\boldsymbol{\theta}_n\right)$, $V^{\mathrm{glo}}_n\left(\boldsymbol{\theta}_n\right)$,  and $\mathcal{U}\left(\boldsymbol{w}\right)$ with constants $\sigma_{\rm v}$, $\sigma_{\rm w}$, and $\sigma_{\mathrm{u}}$, respectively. We have $\sigma_{\mathrm{v}}^2:=2\sigma_{\mathrm{l}}^2 + 2\sigma_{\mathrm{rl}}^2$ and $ \sigma_{\mathrm{u}}^2:= N\sigma_{\mathrm{n}}^2$. Define $\sigma:=\max\left\{\sigma_{\rm v},\sigma_{\mathrm{u}}\right\}$ and $L:=\max\{L_{\mathrm{o}},\sqrt{3L_{\mathrm{v}}^2+L_{\mathrm{w}}^2},L_{\mathrm{u}}\}$, we have Theorem~\ref{thm:convergence}.
\begin{theorem} [Convergence error of \texttt{FedAL}]\label{thm:convergence} 
    Under Assumptions~\ref{asmp:local_models} - \ref{asm:bounded_variance_o}, if $\eta_l^t$ and $\eta_d^t$ satisfy
    $\eta_l^t\leq \frac{1}{8L\tau}$ and $\eta_d^t\leq \frac{1}{8L\tau}$,
    then the average squared gradient norm over all clients and the discriminator and $T$ rounds of \texttt{FedAL} algorithm is bounded by
\begin{align}
    &\frac{1}{S}\!\sum_{t=0}^{T-1}\!\left(\!\sum_{n=1}^N\!\eta_l^t\mathbb{E}\!\left[\left\|\nabla_{\boldsymbol{\theta}_n^{t,0}}\mathcal{V}_n\!\left(\!\boldsymbol{\theta}_n^{t,0}\!\right)\!\right\|^2\!\right]\!+\!\eta_d^t\mathbb{E}\!\left[\!\left\|\nabla_{\boldsymbol{w}^{t,0}}\mathcal{U}\left(\boldsymbol{w}^{t,0}\!\right)\!\right\|^2\!\right]\!\right)\nonumber\\
    & \leq \frac{4}{S \tau}\left(\mathcal{F}\left(\boldsymbol{\Theta}^{0,0},\boldsymbol{w}^{0,0}\right)-\mathcal{F}\left(\boldsymbol{\Theta}^{\ast},\boldsymbol{w}^{\ast}\right)\right)\nonumber\\
    &  \quad + \frac{20L\tau\sigma^2}{S}\sum_{t=0}^{T-1}\left(\sum_{n=1}^N\left(\eta_l^t\right)^2+\left(\eta_d^t\right)^2\right), \label{eq:convergence-error}
\end{align}
where $S:=\sum_{t=0}^{T-1}\left(\sum_{n=1}^N\eta_l^t+\eta_d^t\right)$.
\end{theorem}
We have included the proof of Theorem~\ref{thm:convergence} in our supplementary material. Briefly, we demonstrate the Lipschitz smoothness and bounded variance of client and discriminator objectives, considering various loss function components. We then examine the gradient changes for both clients and the discriminator across each training round, factoring in multiple local iterations. Unlike typical FL algorithms, our \texttt{FedAL} approach integrates local training with global KD for convergence analysis. This integration poses analytical challenges as the global KD process for each client is influenced by their local training outcomes. To address this, we analyze gradient changes post-local training and global KD from the end of the last training round. We then calculate the convergence rate of all clients' local models and the discriminator by summing up the gradient changes over several training rounds.

We have the following observations from  Theorem~\ref{thm:convergence}:
\begin{itemize}
    \item \textbf{Composition of convergence error upper bound}: In Theorem~\ref{thm:convergence}, we analyze the upper bound for the accumulated gradient norm of all clients and the discriminator. The right-hand side of \eqref{eq:convergence-error} reveals that the convergence error of \texttt{FedAL} correlates with two factors: the initialization error (the first term of \eqref{eq:convergence-error}) and the error stemming from variance in stochastic gradient steps and multiple local iterations (the second term of \eqref{eq:convergence-error}).
    \item \textbf{Convergence rate of \texttt{FedAL}}: Based on Theorem~\ref{thm:convergence}, if we let $\eta_l^t=\eta_l^d=\frac{1}{\sqrt{T\tau}}$, the above theorem gives a convergence rate of $O\left(\frac{1}{\sqrt{T\tau}}\right)$. This suggests that as the number of training rounds $T$ approaches infinity, the gradient norm of \texttt{FedAL} approaches zero, signifying algorithm convergence.
    \item \textbf{Impact of stochastic gradient updates}: Assumption \ref{asm:bounded_variance_o} characterizes the bound for the variance of clients' stochastic gradients,  denoted by $\sigma$. Theorem~\ref{thm:convergence} demonstrates that the convergence error increases with the variance $\sigma$.
    \item \textbf{Choose of learning rate}:  As the number of local iterations $\tau$ increases, it is beneficial to use a smaller learning rate to optimize convergence.
\end{itemize}

\section{Experimental Results}\label{sec:experimentation}

\subsection{Setup}

\textbf{Datasets and models.} 
The datasets include MNIST, 
SVHN~\cite{37648}, CIFAR-10~\cite{CIFAR10}, CINIC-10~\cite{darlow2018cinic}, and CelebA. 
For SVHN, CIFAR-10, and CINIC-10, we split the whole training dataset for clients using Dirichlet distribution~\cite{lin2020ensemble} with a parameter $\alpha$. 
The Dirichlet distribution, commonly employed to model non-IID client distributions, has been utilized in various studies \cite{lin2020ensemble,hsu2019measuring,yurochkin2019bayesian}. To implement this, we first generate a distribution of data across different classes, denoted as $q_n$, for each client $n$ using the Dirichlet distribution with parameter $\alpha$, where $q_n \in \mathbb{R}^K$. Subsequently, we select data samples for different classes of the client based on the distribution $q_n$.
Figure ~\ref{fig:svhn_data} illustrates the data distributions of SVHN among clients for different Dirichlet distributions with $\alpha$ equal to 5, 2, and 1. 
For each subfigure of Figure ~\ref{fig:svhn_data}, we only illustrate the data of 10 clients for presentation clarity, though more clients are evaluated in the experiments. A larger circle corresponds to more data samples.  Apparently, a lower $\alpha$ corresponds to a higher heterogeneity of clients' local data.  For CelebA, we use the original data partition in the LEAF benchmark as the local datasets of clients and train a CNN model for each client.

For neural network models, we use LeNet5 for MNIST,
WResNet-10-1, WResNet-10-2, and WResNet-16-1 for SVHN, ResNet18 and ResNet34 for CIFAR-10 and CINIC-10, and CNN for CelebA. We train 10, 20, 15, and 10 clients for MNIST, SVHN, CIFAR-10, and CINIC-10, respectively. We vary the number of clients for CelebA. Each client randomly selects a local model architecture from the candidate models of the corresponding dataset.
The discriminator has two fully-connected layers with 32 and 265 neurons.

\textbf{Baselines.} 
We consider two different settings. First, we consider the case where each client's model is a black box and not shared with other clients. In this case, we compare \texttt{FedAL} with 
\texttt{FedMD}~\cite{li2019fedmd}, \texttt{FedProto} \cite{tan2021fedproto},  \texttt{CFD}~\cite{sattler2020communication}, \texttt{MHAT}~\cite{hu2021mhat}, \texttt{FCCL}~\cite{huang2022learn}, and \texttt{FCCL+}~\cite{huang2023generalizable}, where \texttt{CFD} is the same as \texttt{MHAT}.
In \texttt{FedProto}, clients and the server conduct KD using the intermediate features before the fully-connected layer of client models, instead of model outputs in \texttt{FedMD}. For fair comparison, we adapt \texttt{FedProto} to incorporate unlabeled public data.
Note that \texttt{FedProto} only works for heterogeneous models with the same feature dimension. 
Second, we consider the case where clients' model parameters are allowed to be shared with other clients. 
In this case, it is possible to aggregate those client models that have identical architecture using parameter averaging (as in \texttt{FedAvg}). Simultaneously, KD is used to achieve knowledge transfer among models with different architectures. 
A baseline method that achieves this without our AL component is known as \texttt{FedDF}~\cite{lin2020ensemble}. For our \texttt{FedAL} method, we extend it to support parameter averaging among models with identical architecture and refer to this extended \texttt{FedAL} method as \texttt{FedDF-AL}.
In addition, we also consider a baseline in our ablation study, referred to as \texttt{FedMD-LF}, which applies the LF regularization in \texttt{FedMD} but does not include the AL component of \texttt{FedAL}.

\textbf{Hyper-parameter settings.}
We optimize $\eta_{l}, \eta_{d}$, $\tau$, and the temperature $E$ of discriminator's input $p_n\left(\boldsymbol{x},\boldsymbol{\theta}_n\right)$ by grid search.
In the grid search of hyper-parameters, we first find the best performing $\eta_{l}$ and $\tau$ by grid-search tuning. Based on the grid search results of $\eta_{l}$ and $\tau$, we further search for $\eta_{d}$ and $E$ for our proposed \texttt{FedAL} and its variants. The grid search is done by running algorithms for 120 rounds and finding the parameters with the highest training accuracy. Specifically, the grid for $\eta_{l}$ and $\eta_{d}$ are $\left\{10^{-4},10^{-3},10^{-2},10^{-1}\right\}$. The grid for $\tau$ is $\left\{1,5,10,15,20\right\}$. The grid for $E$ is $\left\{1,2,4,8\right\}$. 

The learning rate of client local models are $\eta_{l}=0.0001$ for CelebA and $\eta_{l}=0.001$ for other datasets. 
The learning rate for the discriminator is $\eta_{d}=0.0001$ for all datasets. For \texttt{FedMD} and \texttt{FedProto}, the number of consecutive iterations in each local training and global knowledge transfer stage is $\tau=1$. For \texttt{FedAL}, we have $\tau=5$. For \texttt{FedDF} and \texttt{FedDF-AL}, we have $\tau=10$, and $\tau=15$, respectively. The temperature of $p_n\left(\boldsymbol{x},\boldsymbol{\theta}_n\right)$ for discriminator's input for \texttt{FedAL} and \texttt{FedDF-AL} is $E=2$.
We use Adam optimizer~\cite{2014Adam} for training all client models. 
The batch size for the local training and global knowledge transfer is 32.
By default, we set the number of public data as $\left|\mathcal{P}\right|=1000$ and the temperature in KL divergence as $E=1$.

\textbf{Experimental environment.} We build up a Federated KD hardware prototype system to evaluate the performance of our \texttt{FedAL} algorithm on MNIST dataset. The prototype system, shown in Fig.~\ref{fig:platform}, includes 10 RaspberryPis, representing edge devices, and a server equipped with GPU RTX 3090. The communication between RaspberryPis with the server is WiFi.
For other datasets including SVHN, CIFAR-10, CINIC-10, CelebA, we conduct simulations on a server with GPU RTX 3090.

\begin{figure}[t]
 \centering \includegraphics[width=0.15\textwidth]{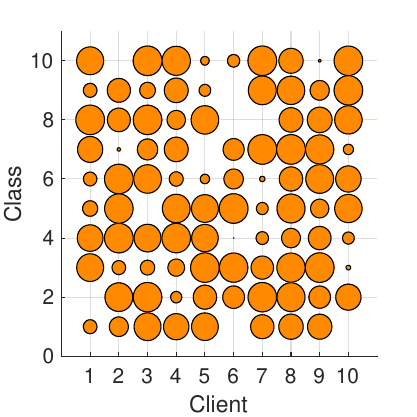}
\centering \includegraphics[width=0.15\textwidth]{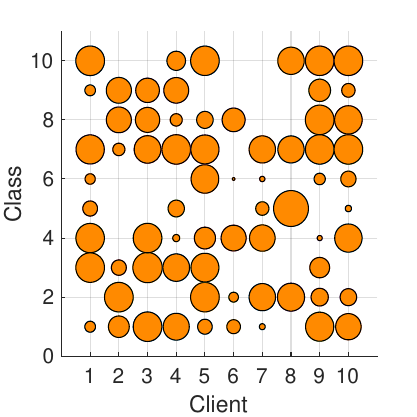}
 \centering \includegraphics[width=0.15\textwidth]{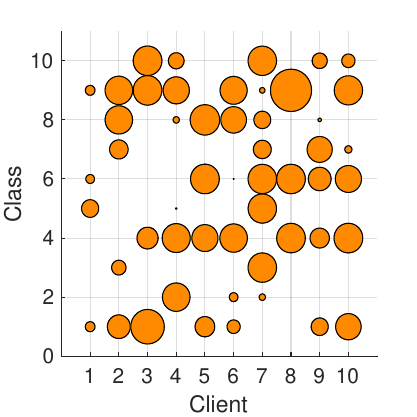}
\caption{Data distributions of SVHN for different $\alpha$, Left to right: $\alpha= 5, 2,$, and 1.}\label{fig:svhn_data}
\end{figure}

\begin{figure}[t]
 \centering \includegraphics[width=0.45\textwidth]{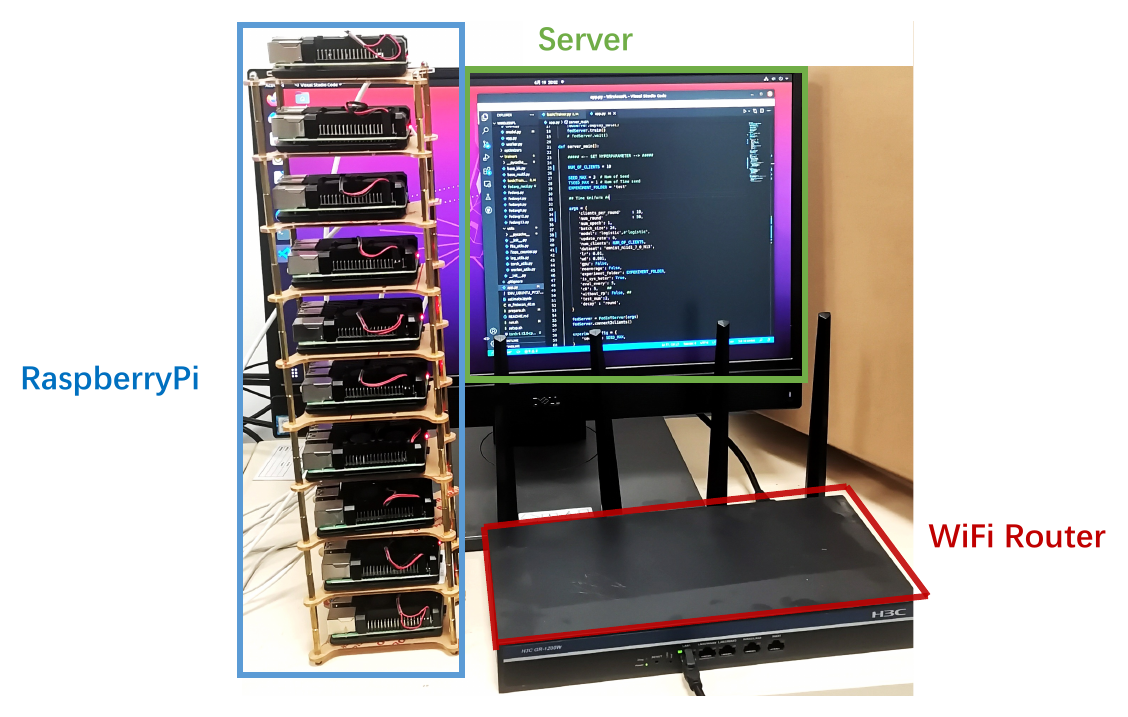}
\caption{Federated KD prototype system.}\label{fig:platform}
\end{figure}

\begin{figure}[t]
 \centering \includegraphics[width=0.3\textwidth]{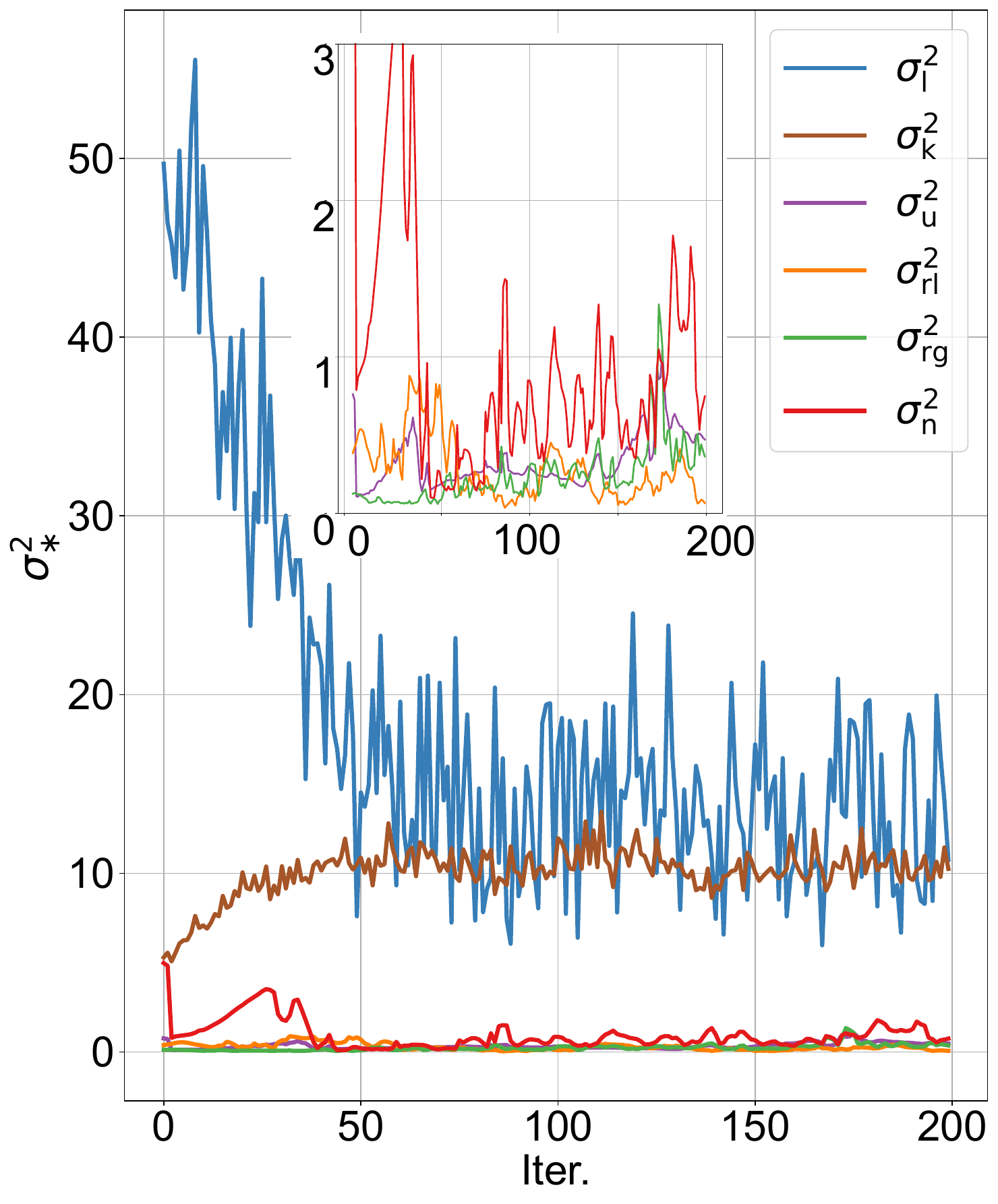}
\caption{Justification of Assumption \ref{asm:bounded_variance_o} on SVHN.}\label{fig:grad}
\end{figure}

\begin{figure}[t!]
\centering
	{
\includegraphics[width=0.38\linewidth]{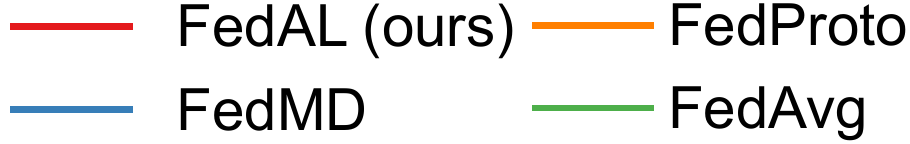}
 \includegraphics[width=0.3\linewidth]{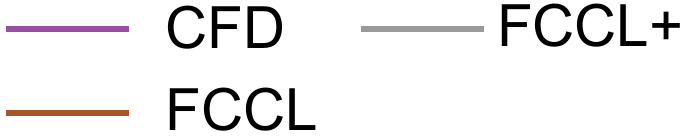}
 		\centering \includegraphics[width=0.3\linewidth]{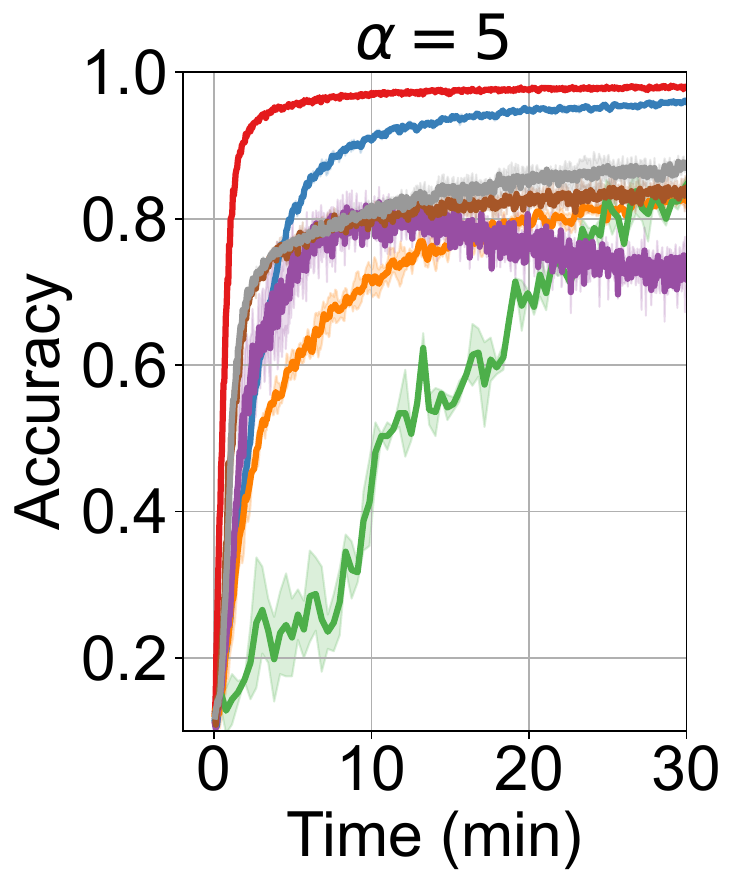}\centering \includegraphics[width=0.3\linewidth]{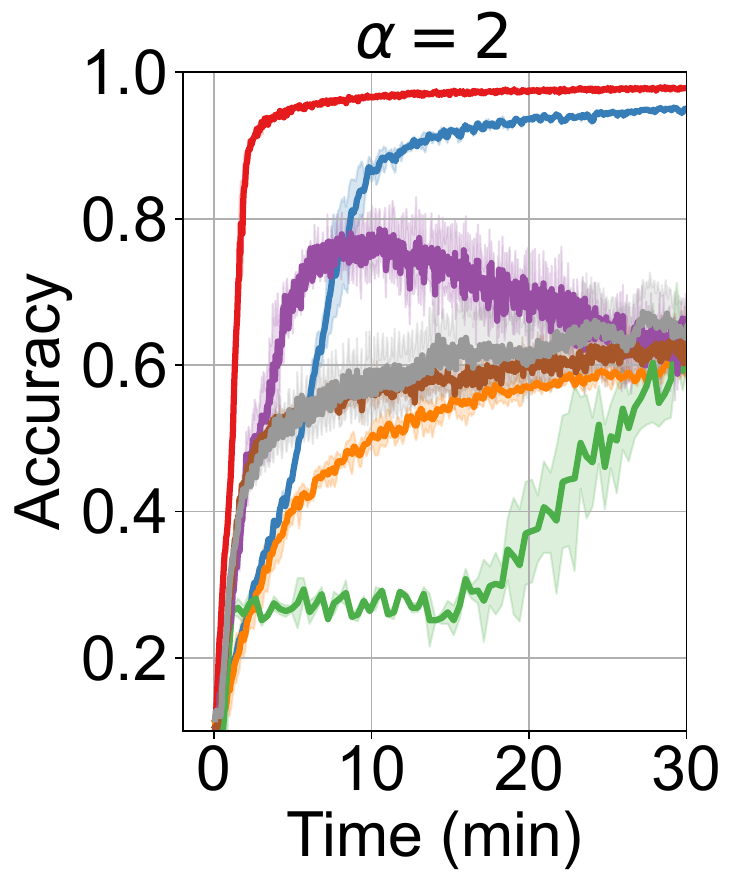}
        \includegraphics[width=0.3\linewidth]{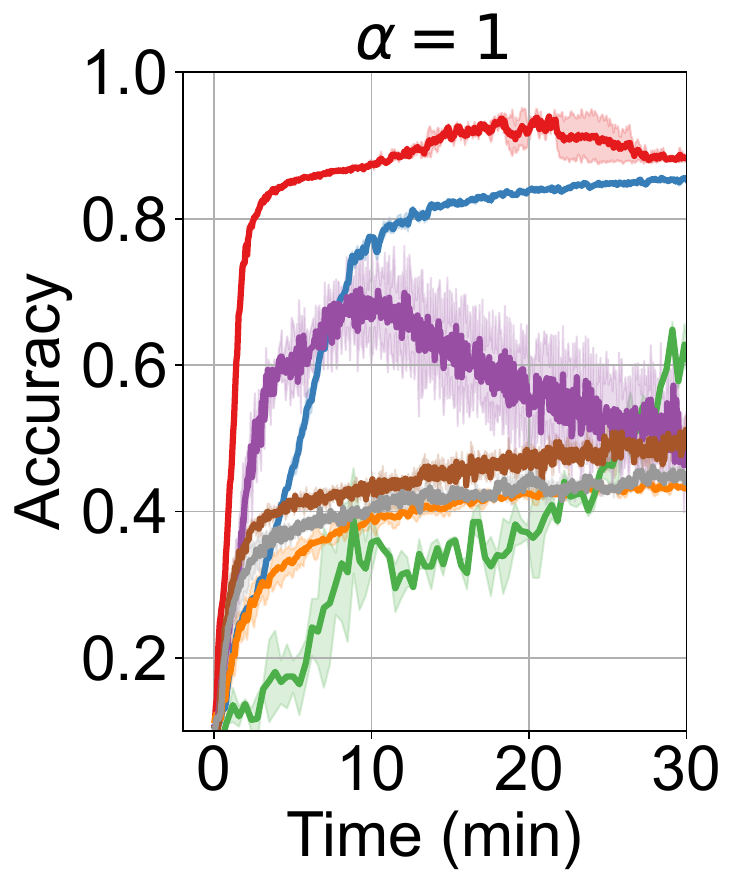}
        \caption{Accuracy vs. time of MNIST under different $\alpha$  (shareable client models).} 		\label{fig:mnist}
}
\end{figure}

\begin{figure}[t!]
\includegraphics[width=0.5\linewidth]{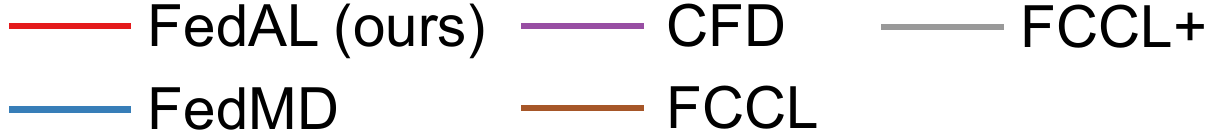}
	\centering
	{\centering \includegraphics[width=0.3\linewidth]{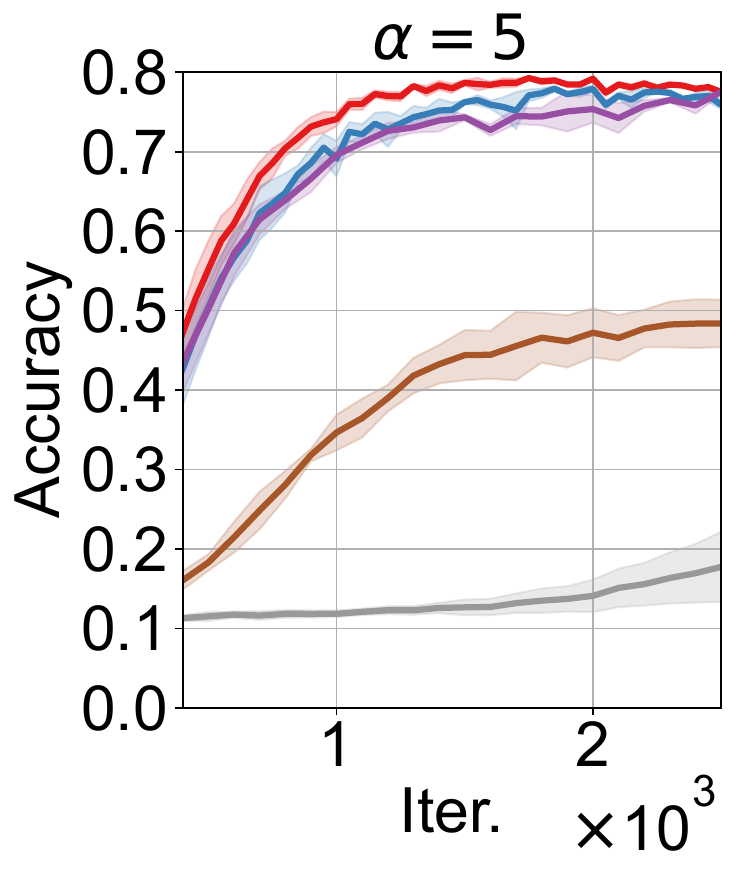}\includegraphics[width=0.3\linewidth]{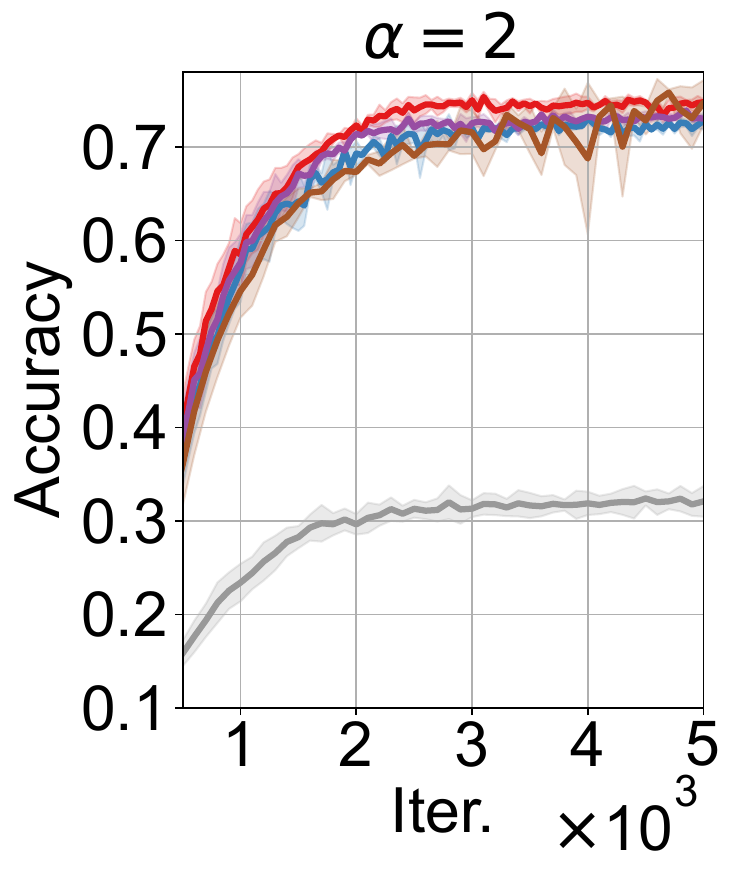}\centering \includegraphics[width=0.3\linewidth]{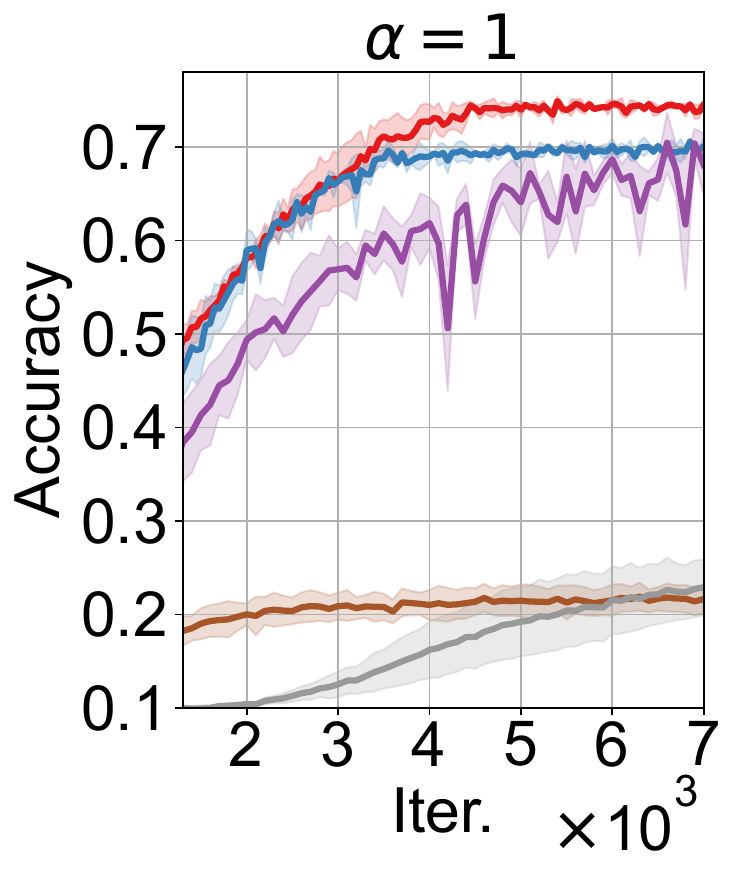}
 \caption{Accuracy vs. number of training iterations of SVHN under different $\alpha$ (black-box, i.e., non-shareable, client models).}	\label{fig:svhn}
}
\end{figure}

\begin{figure}[t!]
	\centering
	{
 \includegraphics[width=0.52\linewidth]{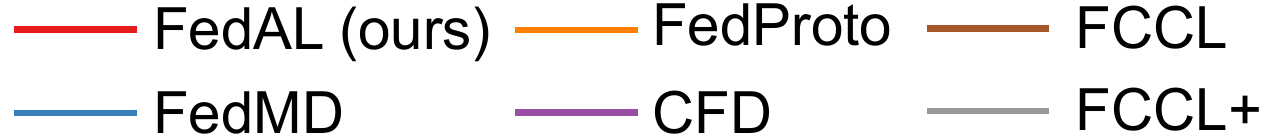}
 
		\centering \includegraphics[width=0.3\linewidth]{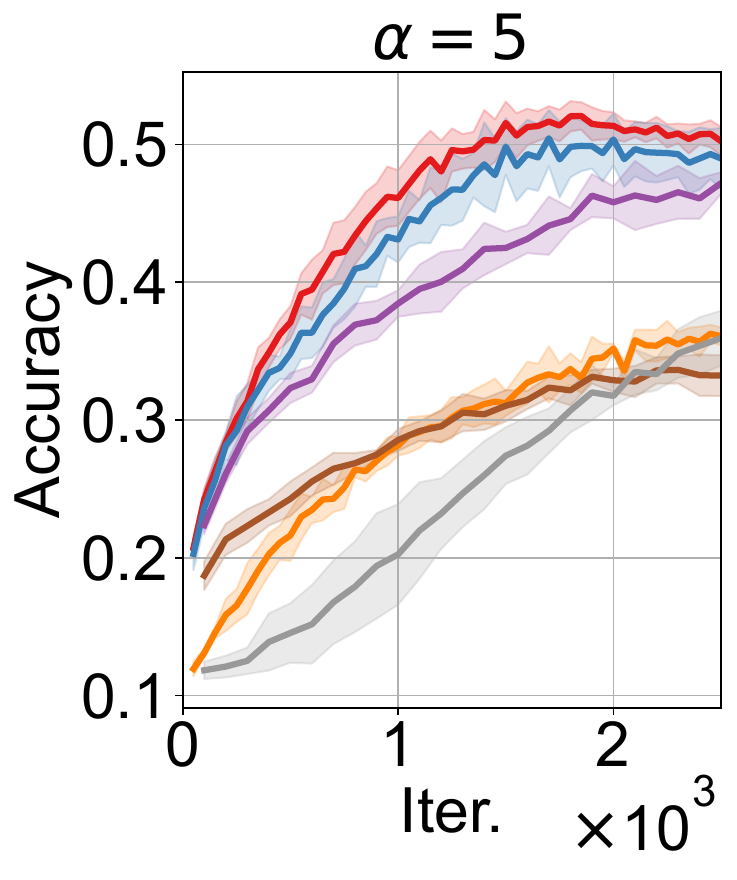}\includegraphics[width=0.3\linewidth]{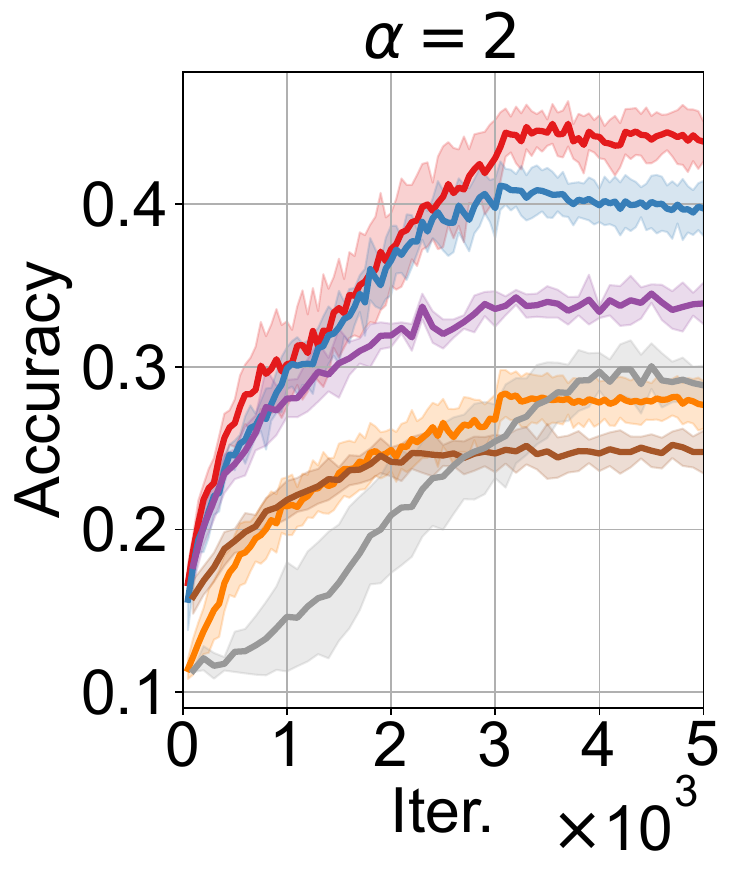}
        \includegraphics[width=0.3\linewidth]{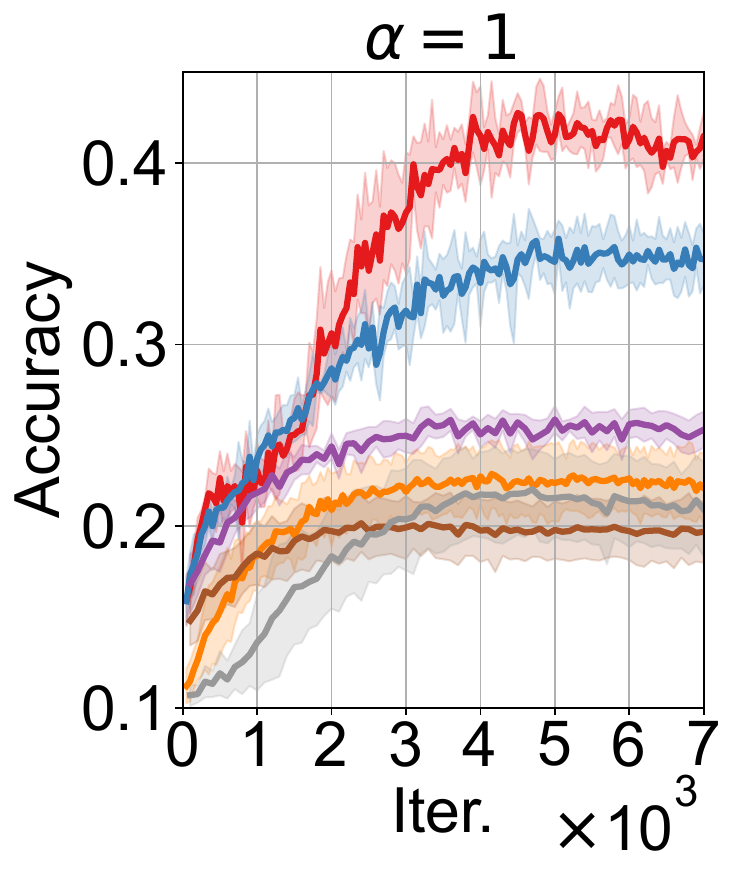}
        \caption{Accuracy vs. number of training iterations of CIFAR-10 under different $\alpha$ (black-box, i.e., non-shareable, client models).} \label{fig:cifar}
}
\end{figure}

\subsection{Results}\label{subsec:perm_homo}
\textbf{Justification of Assumptions.}
Most assumptions in this paper align with existing works. However, due to the introduction of new loss functions in our proposed \texttt{FedAL} algorithm, we substantiate the establishment of upper bounds on gradient variance for each loss function in Assumption \ref{asm:bounded_variance_o}. Figure \ref{fig:grad} illustrates the values of the variance between stochastic gradients and real gradients, denoted as $\sigma_{\rm l}^2$, $\sigma_{\rm k}^2$, $\sigma_{\rm u}^2$, $\sigma_{\rm rl}^2$, $\sigma_{\rm rg}^2$, and $\sigma_{\rm n}^2$, on the SVHN dataset. Notably, all stochastic gradient variances are confined to values smaller than 60. Particularly for the introduced loss functions in \texttt{FedAL}, namely $\sigma_{\rm u}^2$, $\sigma_{\rm rl}^2$, $\sigma_{\rm rg}^2$, and $\sigma_{\rm n}^2$, their stochastic gradient variances are limited to values smaller than 3. Hence, the feasibility of the assumption is empirically justified.

\textbf{Performance for black-box models.}
When clients have black-box models,  Figure~\ref{fig:mnist} depicts the accuracy of \texttt{FedAL}, \texttt{FedMD}, \texttt{FedProto}, and \texttt{FedAvg} under different heterogeneity of clients' local data, i.e., $\alpha= 5, 2,$ and 1 for MNIST. Here, we set all clients to have the same model architecture such that we can make a comparison with \texttt{FedAvg}. We can observe that \texttt{FedAL} achieves higher accuracy than other algorithms under the same running time.

Figure~\ref{fig:svhn} shows the accuracy of \texttt{FedAL} and other comparing algorithms for SVHN dataset under different heterogeneity. We can observe that \texttt{FedAL} outperforms \texttt{FedMD}, \texttt{CFD}, \texttt{FCCL}, and \texttt{FCCL+} for varying $\alpha$. Here, \texttt{FedProto} is not applicable for the model architectures used in SVHN, i.e., WResNet-10-1, WResnet-10-2, and WResNet-16-1, because they have different dimensions of features before the fully-connected layer.
For CIFAR-10, Figure~\ref{fig:cifar} depicts the accuracy of different algorithms under different data heterogeneity levels $\alpha=5, 2,$ and 1, where \texttt{FedAL} achieves higher accuracy than \texttt{FedMD}, \texttt{CFD}, \texttt{FCCL}, \texttt{FCCL+} and \texttt{FedProto} and is more robust across different values of $\alpha$. 
In addition, we have run the experiments on the CelebA and CINIC-10 datasets with 10,000 iterations. Table \ref{tab:rst_celeba} presents the model accuracy of different methods for CelebA. It demonstrates that \texttt{FedAL} can achieve a higher model accuracy compared to other algorithms across different numbers of clients $N$. 
Table \ref{tab:rst_cinic10} demonstrates the model accuracy of different methods on CINIC-10  under varying levels of data heterogeneity. We have the same observations and conclusions for CINIC-10 as other datasets.

To evaluate the performance of \texttt{FD} compared with our \texttt{FedAL}, we have conducted experiments on the SVHN and CIFAR-10 datasets under varying levels of clients' data heterogeneity. For a fair comparison with our existing results in Figures \ref{fig:svhn}-\ref{fig:svhndf}, we trained models with 3000, 5000, and 7000 iterations for $\alpha = 5, 2,$ and 1, respectively. We present the model accuracy of \texttt{FD} compared with \texttt{FedMD}, \texttt{FedProto}, and our \texttt{FedAL} in Tables \ref{tab:rst_fd} and \ref{tab:rst_fd_cifar10}. For FD, unfortunately, the authors of the paper did not make their code available. As a result, we were unable to replicate their reported results and can only achieve a low accuracy for FD.

{
\begin{table}[t] \centering
    \caption{Results of CelebA for 10, 000 iterations. }
    \label{tab:rst_celeba}
        \begin{center}
            \begin{small}
	\begin{tabular}{m{1.8cm}m{1.75cm}m{1.75cm}m{1.75cm}}
		\hline
   Method    &   $N=$ 20  &   $N=$ 50  &   $N=$ 100  \tabularnewline
\hline
\texttt{FedProto}&0.51 $\pm$ 0.005&0.53 $\pm$ 0.02&0.52 $\pm$ 0.008\tabularnewline
\texttt{FedMD}&0.51 $\pm$ 0.02&0.64 $\pm$ 0.01&0.66 $\pm$ 0.01\tabularnewline
\texttt{CFD}&0.49 $\pm$ 0.09&0.53 $\pm$ 0.15&0.55 $\pm$ 0.17\tabularnewline
\texttt{FCCL}&0.49 $\pm$ 0.08&0.57 $\pm$ 0.13&0.56 $\pm$ 0.15\tabularnewline
\texttt{FCCL+}& 0.53$\pm$ 0.10&0.54 $\pm$ 0.13&0.54 $\pm$ 0.13\tabularnewline
\texttt{FedAL} (ours)&\textbf{0.56} $\pm$ 0.03&\textbf{0.69} $\pm$ 0.02&\textbf{0.70} $\pm$ 0.01\tabularnewline
		\hline
	\end{tabular}
\end{small}
\end{center}
\end{table}
~
\begin{table}[t!] \centering
    \caption{Results of CINIC-10 for 10, 000 iterations. }
    \label{tab:rst_cinic10}
        \begin{center}
            \begin{small}
	\begin{tabular}{m{1.8cm}m{1.75cm}m{1.75cm}m{1.75cm}}
		\hline
   Method    &   $\alpha=$ 5  &   $\alpha=$ 2 &   $\alpha=$ 1 \tabularnewline
\hline
\texttt{FedProto}&0.43 $\pm$ 0.02&0.34 $\pm$ 0.03&0.25 $\pm$ 0.02\tabularnewline
\texttt{FedMD}&0.54 $\pm$ 0.02&0.43 $\pm$ 0.02&0.31 $\pm$ 0.04\tabularnewline
\texttt{CFD}&0.49$\pm$ 0.04& 0.38$\pm$ 0.03& 0.29$\pm$ 0.01\tabularnewline
\texttt{FCCL}& 0.37$\pm$ 0.05& 0.29$\pm$ 0.04&0.25 $\pm$ 0.03\tabularnewline
\texttt{FCCL+}& 0.41$\pm$ 0.05& 0.31$\pm$ 0.03&0.27 $\pm$ 0.01\tabularnewline
\texttt{FedAL} (ours)&\textbf{0.56} $\pm$ 0.02&\textbf{0.45} $\pm$ 0.03&\textbf{0.35} $\pm$ 0.05\tabularnewline
		\hline
	\end{tabular}
\end{small}
\end{center}
\end{table}
~
\begin{table}[t] \centering
    \caption{Results of FD for SHVN dataset. }
    \label{tab:rst_fd}
        \begin{center}
            \begin{small}
	\begin{tabular}{p{1.78cm}p{1.68cm}p{1.83cm}p{1.83cm}}
		\hline
    Method    &   $\alpha=$ 5  &   $\alpha=$ 2 &   $\alpha=$ 1 \tabularnewline
\hline
 \texttt{FD}          & 0.16 $\pm$ 0.01  & 0.14 $\pm$ 0.005  & 0.12 $\pm$ 0.004 \tabularnewline
 \texttt{FedMD} & 0.76 $\pm$ 0.006 & 0.73 $\pm$ 0.0009 & 0.71 $\pm$ 0.0008 \tabularnewline
\texttt{FedAL} (ours) & \textbf{0.78} $\pm$ 0.003 & \textbf{0.75} $\pm$ 0.0003 & \textbf{0.75} $\pm$ 0.0009 \tabularnewline
		\hline
	\end{tabular}
\end{small}
\end{center}
\end{table}
~
\begin{table}[t] \centering
    \caption{Results of FD for CIFAR-10 dataset. }
    \label{tab:rst_fd_cifar10}
        \begin{center}
            \begin{small}
	\begin{tabular}{p{1.78cm}p{1.68cm}p{1.83cm}p{1.83cm}}
		\hline
    Method    &   $\alpha=$ 5  &   $\alpha=$ 2 &   $\alpha=$ 1 \tabularnewline
\hline
 \texttt{FD}          &  0.15 $\pm$ 0.02 & 0.17 $\pm$ 0.007& 0.17 $\pm$ 0.02  \tabularnewline
\texttt{FedProto} & 0.38 $\pm$ 0.009 & 0.28 $\pm$ 0.01& 0.22 $\pm$ 0.02 \tabularnewline
\texttt{FedMD} & 0.49 $\pm$ 0.02 & 0.40 $\pm$ 0.02& 0.35 $\pm$ 0.01 \tabularnewline
\texttt{FedAL} (ours) & \textbf{0.51} $\pm$ 0.008 & \textbf{0.44} $\pm$ 0.01 & \textbf{0.41} $\pm$ 0.01 
\tabularnewline
		\hline
	\end{tabular}
\end{small}
\end{center}
\end{table}
}

\textbf{Performance for shareable models.} 
When clients agree to share their model parameters with the server, Figure~\ref{fig:svhndf} shows the comparison between vanilla \texttt{FedDF} and our \texttt{FedDF-AL}. 
We observe that \texttt{FedDF-AL} generally achieves higher model accuracy under different heterogeneity levels. When the local data of clients have high heterogeneity, e.g., $\alpha=1$, the accuracy of the proposed AL mechanism changes slowly at the beginning of model training. However, clients can have higher final model accuracy because our proposed AL mechanism alleviates the bias of clients toward different classes of data and facilitates knowledge transfer among clients.

\textbf{Performance of communication overheads.}
We evaluate the communication overhead, quantified by the volume of upstream transmitting data, across various algorithms. Fig. \ref{fig:svhncomm} depicts the accuracy concerning communication overhead for the SVHN dataset across different levels of data heterogeneity ($\alpha=5, 2$, and 1). We simulate 20 clients using homogeneous WResNet10-1 models to ensure the feasibility of all comparing algorithms.
Notably, \texttt{FedAL} induces significantly lower (two orders of magnitude) communication overhead than \texttt{FedAvg} to reach the same accuracy. Furthermore, despite incurring similar communication overhead levels, \texttt{FedAL} achieves higher accuracy than \texttt{FedMD}, \texttt{FedProto}, \texttt{CFD}, \texttt{FCCL}, and \texttt{FCCL+}.

\textbf{Ablation study.}
We conduct an ablation study to verify the usefulness of all the components in the optimization objectives of \texttt{FedAL}.
We note that \texttt{FedAL} without the AL component is equivalent to \texttt{FedMD-LF}, and \texttt{FedAL} with neither the AL component or LF regularization is equivalent to \texttt{FedMD}. For SVHN under $\alpha=2$ with 5000 iterations, 
we can observe from Table~\ref{tab:ablation} that, both LF regularization and the AL process contribute to a higher model accuracy for \texttt{FedAL}.

In addition, we evaluate the impact of public data on the superiority of \texttt{FedAL}. Table \ref{tab:ablation_pub} presents a comparison of the test accuracy between \texttt{FedMD} and \texttt{FedAL} across different numbers of public data samples, utilizing SVHN with $\alpha=2$ for 5000 iterations. We observe that while the number of public data samples significantly affects test accuracy, our proposed \texttt{FedAL} algorithm consistently outperforms \texttt{FedMD}.

\begin{figure}[t!]
\centering
	{
		\centering \includegraphics[width=0.315\linewidth]{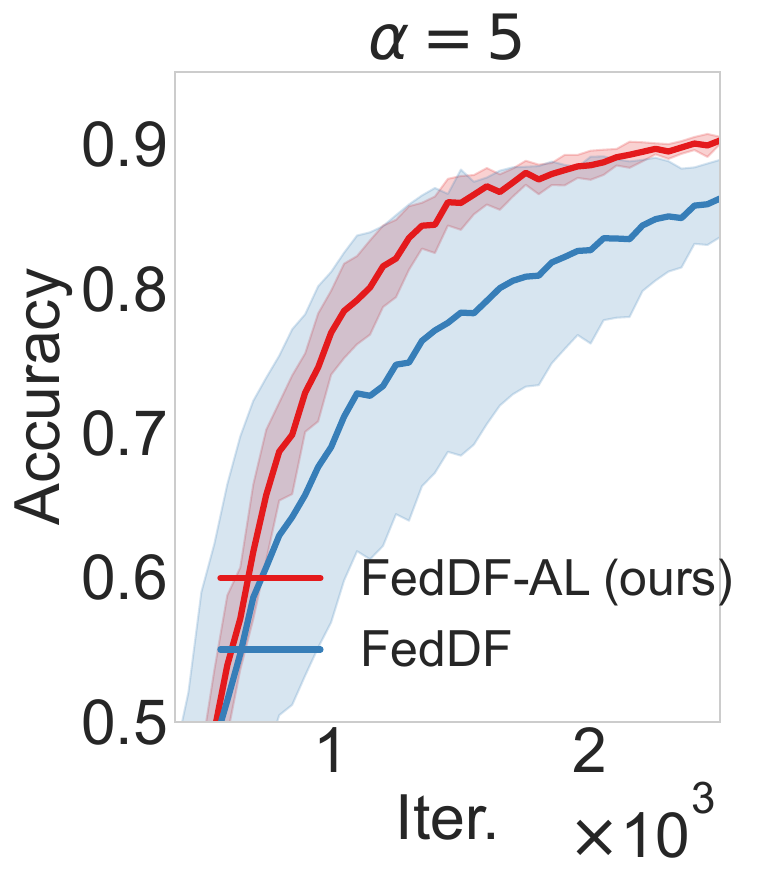}\centering \includegraphics[width=0.3\linewidth]{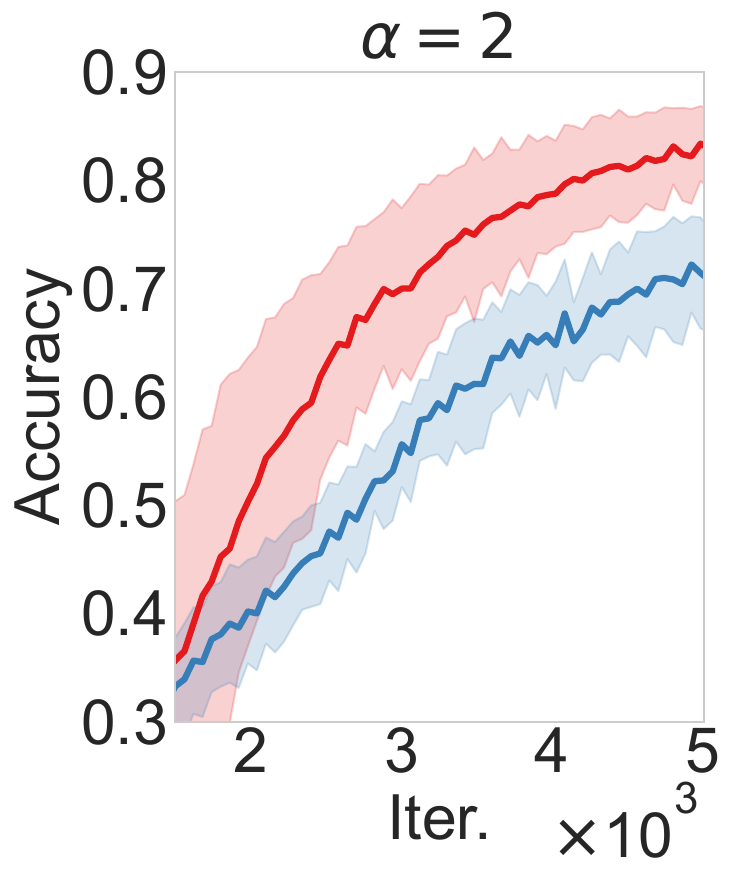}
        \includegraphics[width=0.3\linewidth]{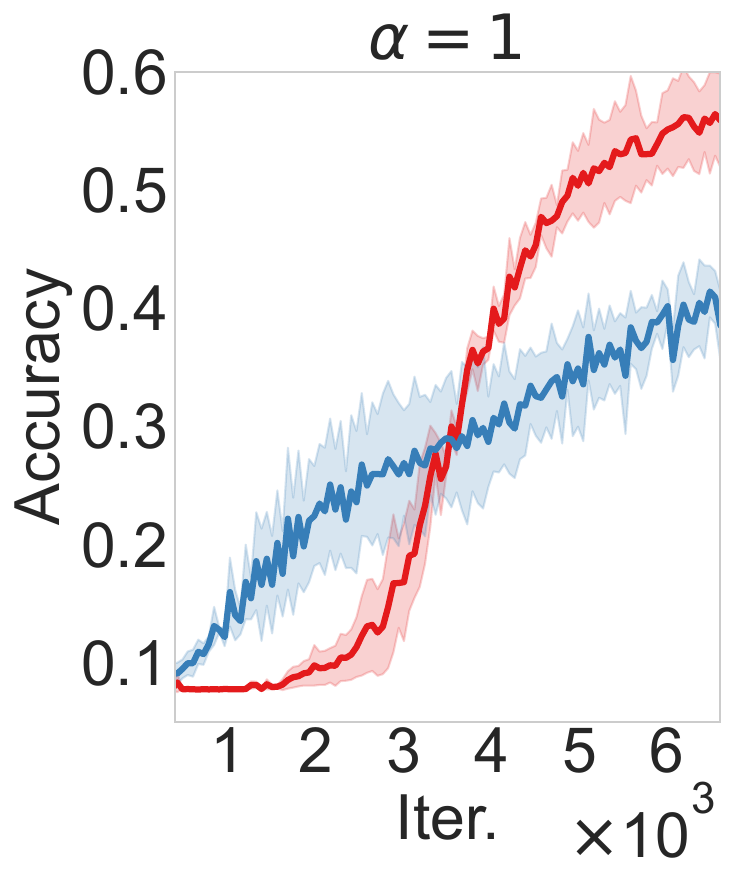}
        \caption{Accuracy vs. number of training iterations of SVHN under different $\alpha$  (shareable client models).} 		\label{fig:svhndf}
}
\end{figure}
\begin{table}[t] \centering
    \caption{Ablation study of loss components on SVHN with $\alpha=2$ for 5000 iterations. }
    \label{tab:ablation}
        \begin{center}
            \begin{small}
            \begin{sc}
	\begin{tabular}{p{1.2cm}p{1.75cm}p{1.75cm}p{1.75cm}}
		\hline
		Method&\texttt{FedMD}  & \texttt{FedMD-LF} & \texttt{FedAL}\tabularnewline
  \hline
  Accuracy & 0.722$\pm$0.005 & 0.732$\pm$0.004 & 0.744$\pm$0.007 \tabularnewline
		\hline
	\end{tabular}
\end{sc}
\end{small}
\end{center}
\end{table}

\begin{figure}[t!]
\centering
	{
 \includegraphics[width=0.38\linewidth]{figure/homo_svhn_legend1.pdf}
 \includegraphics[width=0.3\linewidth]{figure/homo_svhn_legend2.pdf}
 
		\centering \includegraphics[width=0.3\linewidth]{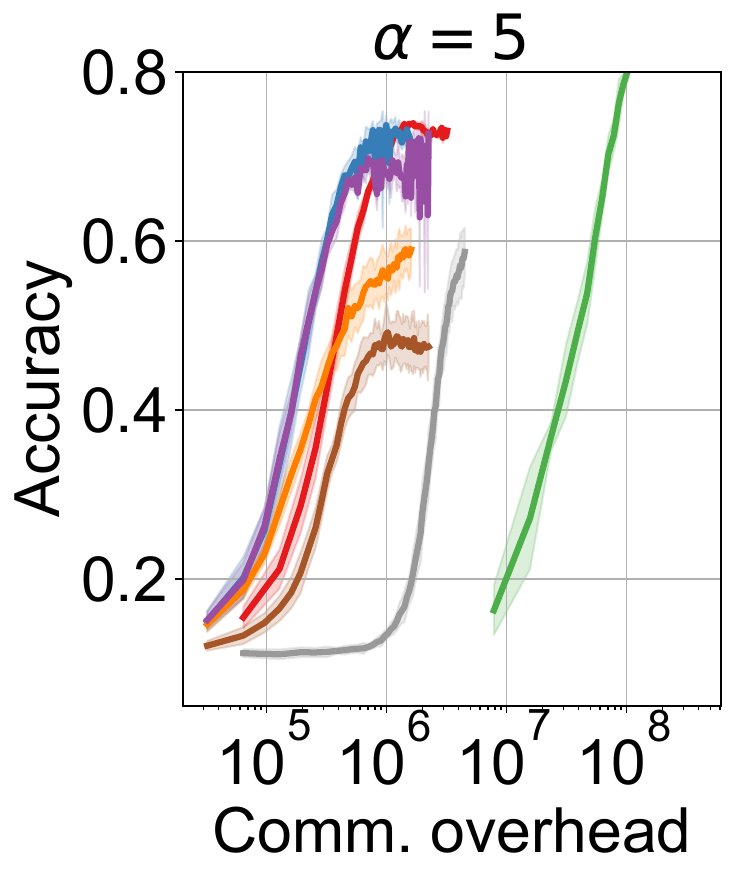}\centering \includegraphics[width=0.3\linewidth]{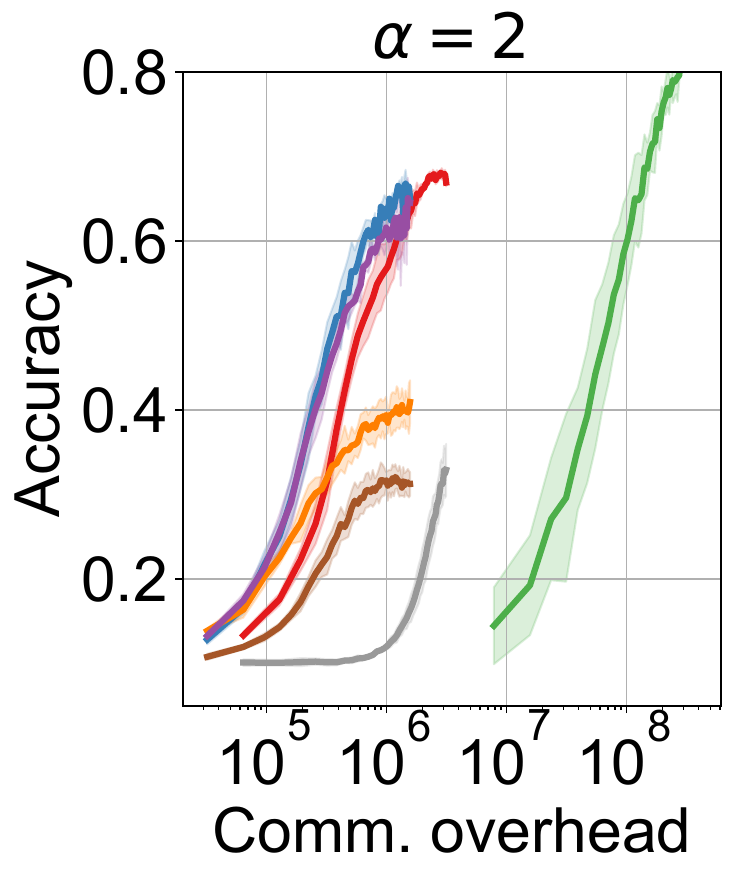}
        \includegraphics[width=0.3\linewidth]{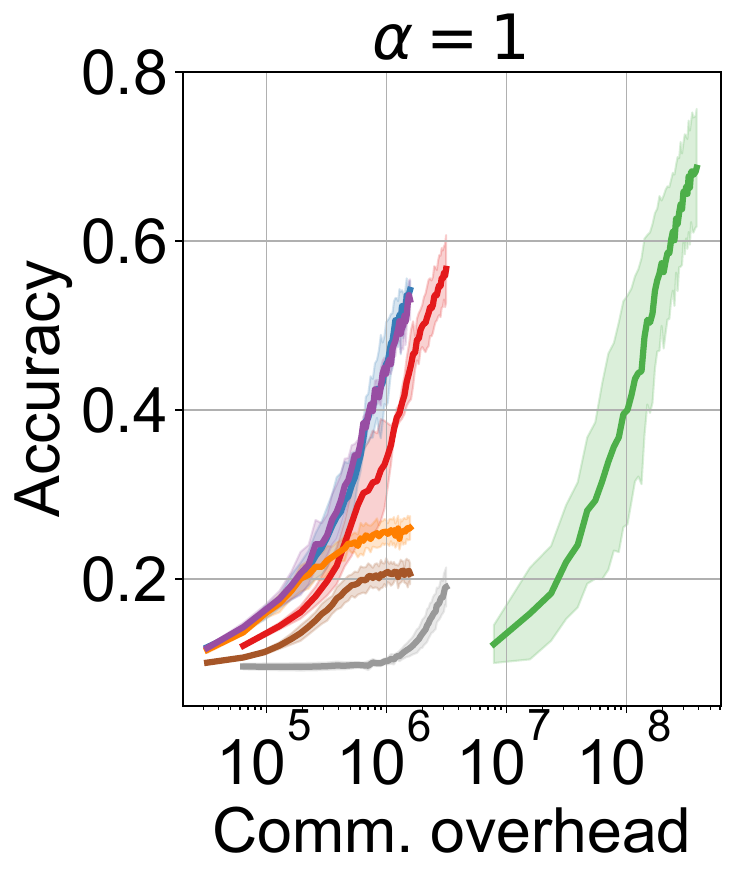}
        \caption{Accuracy vs. communication overhead of SVHN under different $\alpha$.}		\label{fig:svhncomm}
}
\end{figure}

\begin{table}[t] \centering
    \caption{Accuracy under the varying number of public data on SVHN with $\alpha=1$ for 5000 iterations. }
    \label{tab:ablation_pub}
        \begin{center}
            \begin{small}
            \begin{sc}
	\begin{tabular}{m{1cm}m{1.cm}m{1.cm}m{1.cm}m{1.cm}m{1.cm}}
		\hline
		$\left|\mathcal{P}\right|$ &200  & 400 & 600 & 800 &1000\tabularnewline
  \hline
  \texttt{FedMD} & 0.35 $\pm$0.001 & 0.54 $\pm$0.007 & 0.60 $\pm$0.006 & 0.65 $\pm$0.007 & 0.70 $\pm$0.006 
 \tabularnewline
\texttt{FedAL} & 0.43 $\pm$0.003 & 0.57 $\pm$0.003 & 0.64 $\pm$0.006  & 0.72 $\pm$0.002 & 0.75 $\pm$0.002 \tabularnewline
\hline
	\end{tabular}
\end{sc}
\end{small}
\end{center}
\end{table}

\section{Conclusion}\label{sec:conclusion}
In this paper, we proposed \texttt{FedAL} to achieve efficient knowledge transfer in federated KD, considering black-box client local models trained using heterogeneous local data. \texttt{FedAL} trains a discriminator on the server to stimulate the knowledge transfer among clients.
We have formulated a min-max game between clients and the discriminator 
to promote all clients producing the same model output close to the true data label. 
We have further designed the LF regularization for both local training and global knowledge transfer to improve client model accuracy.
Experimental results have demonstrated that \texttt{FedAL} and its variants achieve higher accuracy than other federated KD baselines.

\bibliographystyle{IEEEtran}
\bibliography{ref}

\begin{thebibliography}{10}
\providecommand{\url}[1]{#1}
\csname url@samestyle\endcsname
\providecommand{\newblock}{\relax}
\providecommand{\bibinfo}[2]{#2}
\providecommand{\BIBentrySTDinterwordspacing}{\spaceskip=0pt\relax}
\providecommand{\BIBentryALTinterwordstretchfactor}{4}
\providecommand{\BIBentryALTinterwordspacing}{\spaceskip=\fontdimen2\font plus
\BIBentryALTinterwordstretchfactor\fontdimen3\font minus
  \fontdimen4\font\relax}
\providecommand{\BIBforeignlanguage}[2]{{%
\expandafter\ifx\csname l@#1\endcsname\relax
\typeout{** WARNING: IEEEtran.bst: No hyphenation pattern has been}%
\typeout{** loaded for the language `#1'. Using the pattern for}%
\typeout{** the default language instead.}%
\else
\language=\csname l@#1\endcsname
\fi
#2}}
\providecommand{\BIBdecl}{\relax}
\BIBdecl

\bibitem{kairouz2019advances}
P.~Kairouz, H.~B. McMahan, B.~Avent, A.~Bellet, M.~Bennis, A.~N. Bhagoji,
  K.~Bonawitz, Z.~Charles, G.~Cormode, R.~Cummings \emph{et~al.}, ``Advances
  and open problems in federated learning,'' \emph{Foundations and
  Trends{\textregistered} in Machine Learning}, vol.~14, no. 1--2, pp. 1--210,
  2021.

\bibitem{li2019fedmd}
D.~Li and J.~Wang, ``{Fedmd}: Heterogenous federated learning via model
  distillation,'' in \emph{NeurIPS Workshop on Federated Learning for Data
  Privacy and Confidentiality}, 2019.

\bibitem{pmlr-v97-hoang19a}
M.~Hoang, N.~Hoang, B.~K.~H. Low, and C.~Kingsford, ``Collective model fusion
  for multiple black-box experts,'' in \emph{Proceedings of International
  Conference on Machine Learning}, vol.~97, 2019, pp. 2742--2750.

\bibitem{mcmahan2017communication}
B.~McMahan, E.~Moore, D.~Ramage, S.~Hampson, and B.~A. y~Arcas,
  ``Communication-efficient learning of deep networks from decentralized
  data,'' in \emph{Artificial Intelligence and Statistics}.\hskip 1em plus
  0.5em minus 0.4em\relax PMLR, 2017, pp. 1273--1282.

\bibitem{sun2021pain}
P.~Sun, H.~Che, Z.~Wang, Y.~Wang, T.~Wang, L.~Wu, and H.~Shao, ``Pain-fl:
  Personalized privacy-preserving incentive for federated learning,''
  \emph{IEEE Journal on Selected Areas in Communications}, vol.~39, no.~12, pp.
  3805--3820, 2021.

\bibitem{sun2022profit}
P.~Sun, X.~Chen, G.~Liao, and J.~Huang, ``A profit-maximizing model marketplace
  with differentially private federated learning,'' in \emph{IEEE INFOCOM
  2022-IEEE Conference on Computer Communications}.\hskip 1em plus 0.5em minus
  0.4em\relax IEEE, 2022, pp. 1439--1448.

\bibitem{jiao2024provably}
Y.~Jiao, K.~Yang, T.~Wu, C.~Jian, and J.~Huang, ``Provably convergent federated
  trilevel learning,'' in \emph{Proceedings of the AAAI Conference on
  Artificial Intelligence}, vol.~38, no.~11, 2024, pp. 12\,928--12\,937.

\bibitem{ba2014deep}
J.~Ba and R.~Caruana, ``Do deep nets really need to be deep?'' \emph{Advances
  in neural information processing systems}, vol.~27, 2014.

\bibitem{zhang2021fedzkt}
L.~Zhang, D.~Wu, and X.~Yuan, ``{FedZKT}: Zero-shot knowledge transfer towards
  resource-constrained federated learning with heterogeneous on-device
  models,'' in \emph{IEEE International Conference on Distributed Computing
  Systems (ICDCS)}, 2022, pp. 928--938.

\bibitem{8904164}
J.-H. Ahn, O.~Simeone, and J.~Kang, ``Wireless federated distillation for
  distributed edge learning with heterogeneous data,'' in \emph{2019 IEEE 30th
  Annual International Symposium on Personal, Indoor and Mobile Radio
  Communications (PIMRC)}, 2019, pp. 1--6.

\bibitem{jeong2018communication}
E.~Jeong, S.~Oh, H.~Kim, J.~Park, M.~Bennis, and S.-L. Kim,
  ``Communication-efficient on-device machine learning: Federated distillation
  and augmentation under non-iid private data,'' in \emph{NeurIPS Workshop on
  Machine Learning on the Phone and other Consumer Devices}, 2018.

\bibitem{tan2021fedproto}
Y.~Tan, G.~Long, L.~Liu, T.~Zhou, Q.~Lu, J.~Jiang, and C.~Zhang, ``{FedProto}:
  Federated prototype learning across heterogeneous clients,'' in \emph{AAAI
  Conference on Artificial Intelligence}, 2022.

\bibitem{lin2020ensemble}
T.~Lin, L.~Kong, S.~U. Stich, and M.~Jaggi, ``Ensemble distillation for robust
  model fusion in federated learning,'' in \emph{Conference on Neural
  Information Processing Systems}, 2021.

\bibitem{zhu2021data}
Z.~Zhu, J.~Hong, and J.~Zhou, ``Data-free knowledge distillation for
  heterogeneous federated learning,'' in \emph{International Conference on
  Machine Learning}, 2021, pp. 12\,878--12\,889.

\bibitem{9879661}
L.~Zhang, L.~Shen, L.~Ding, D.~Tao, and L.-Y. Duan, ``Fine-tuning global model
  via data-free knowledge distillation for non-iid federated learning,'' in
  \emph{IEEE/CVF Conference on Computer Vision and Pattern Recognition (CVPR)},
  2022, pp. 10\,164--10\,173.

\bibitem{he2020group}
C.~He, M.~Annavaram, and S.~Avestimehr, ``Group knowledge transfer: Federated
  learning of large cnns at the edge,'' \emph{Advances in Neural Information
  Processing Systems}, vol.~33, pp. 14\,068--14\,080, 2020.

\bibitem{yao2021local}
D.~Yao, W.~Pan, Y.~Dai, Y.~Wan, X.~Ding, H.~Jin, Z.~Xu, and L.~Sun,
  ``Local-global knowledge distillation in heterogeneous federated learning
  with non-iid data,'' \emph{arXiv preprint arXiv:2107.00051}, 2021.

\bibitem{wu2021fedkd}
C.~Wu, F.~Wu, R.~Liu, L.~Lyu, Y.~Huang, and X.~Xie, ``{FedKD}: Communication
  efficient federated learning via knowledge distillation,'' \emph{Nature
  Communications}, 2022.

\bibitem{li2020practical}
Q.~Li, B.~He, and D.~Song, ``Practical one-shot federated learning for
  cross-silo setting,'' in \emph{Proceedings of the Thirtieth International
  Joint Conference on Artificial Intelligence}, 2021.

\bibitem{cho2021personalized}
Y.~J. Cho, J.~Wang, T.~Chiruvolu, and G.~Joshi, ``Personalized federated
  learning for heterogeneous clients with clustered knowledge transfer,''
  \emph{arXiv preprint arXiv:2109.08119}, 2021.

\bibitem{ozkara2021quped}
K.~Ozkara, N.~Singh, D.~Data, and S.~Diggavi, ``Quped: Quantized
  personalization via distillation with applications to federated learning,''
  \emph{Advances in Neural Information Processing Systems}, vol.~34, pp.
  3622--3634, 2021.

\bibitem{heinbaughdata}
C.~Heinbaugh, E.~Luz-Ricca, and H.~Shao, ``Data-free one-shot federated
  learning under very high statistical heterogeneity,'' in \emph{International
  Conference on Learning Representations}, 2023.

\bibitem{sattler2020communication}
F.~Sattler, A.~Marban, R.~Rischke, and W.~Samek, ``Communication-efficient
  federated distillation,'' \emph{arXiv preprint arXiv:2012.00632}, 2020.

\bibitem{hu2021mhat}
L.~Hu, H.~Yan, L.~Li, Z.~Pan, X.~Liu, and Z.~Zhang, ``Mhat: An efficient
  model-heterogenous aggregation training scheme for federated learning,''
  \emph{Information Sciences}, vol. 560, pp. 493--503, 2021.

\bibitem{huang2022learn}
W.~Huang, M.~Ye, and B.~Du, ``Learn from others and be yourself in
  heterogeneous federated learning,'' in \emph{Proceedings of the IEEE/CVF
  Conference on Computer Vision and Pattern Recognition}, 2022, pp.
  10\,143--10\,153.

\bibitem{huang2023generalizable}
W.~Huang, M.~Ye, Z.~Shi, and B.~Du, ``Generalizable heterogeneous federated
  cross-correlation and instance similarity learning,'' \emph{IEEE Transactions
  on Pattern Analysis and Machine Intelligence}, 2023.

\bibitem{lee2021preservation}
G.~Lee, Y.~Shin, M.~Jeong, and S.-Y. Yun, ``Preservation of the global
  knowledge by not-true self knowledge distillation in federated learning,'' in
  \emph{Conference on Neural Information Processing Systems}, 2021.

\bibitem{tang2023fedrad}
J.~Tang, X.~Ding, D.~Hu, B.~Guo, Y.~Shen, P.~Ma, and Y.~Jiang, ``Fedrad:
  Heterogeneous federated learning via relational adaptive distillation,''
  \emph{Sensors}, vol.~23, no.~14, p. 6518, 2023.

\bibitem{he2022learning}
Y.~He, Y.~Chen, X.~Yang, H.~Yu, Y.-H. Huang, and Y.~Gu, ``Learning critically:
  Selective self-distillation in federated learning on non-iid data,''
  \emph{IEEE Transactions on Big Data}, 2022.

\bibitem{xu2022acceleration}
C.~Xu, Z.~Hong, M.~Huang, and T.~Jiang, ``Acceleration of federated learning
  with alleviated forgetting in local training,'' \emph{arXiv preprint
  arXiv:2203.02645}, 2022.

\bibitem{wang2024dfrd}
S.~Wang, Y.~Fu, X.~Li, Y.~Lan, M.~Gao \emph{et~al.}, ``Dfrd: Data-free
  robustness distillation for heterogeneous federated learning,''
  \emph{Advances in Neural Information Processing Systems}, vol.~36, 2024.

\bibitem{aljahdali2024flashback}
M.~Aljahdali, A.~M. Abdelmoniem, M.~Canini, and S.~Horv{\'a}th, ``Flashback:
  Understanding and mitigating forgetting in federated learning,'' \emph{arXiv
  preprint arXiv:2402.05558}, 2024.

\bibitem{kim2024federated}
S.~Kim, H.~Park, M.~Kang, K.~H. Jin, E.~Adeli, K.~M. Pohl, and S.~H. Park,
  ``Federated learning with knowledge distillation for multi-organ segmentation
  with partially labeled datasets,'' \emph{Medical Image Analysis}, p. 103156,
  2024.

\bibitem{liu2023adaptive}
J.~Liu, Q.~Zeng, H.~Xu, Y.~Xu, Z.~Wang, and H.~Huang, ``Adaptive block-wise
  regularization and knowledge distillation for enhancing federated learning,''
  \emph{IEEE/ACM Transactions on Networking}, 2023.

\bibitem{9964434}
H.~Jin, D.~Bai, D.~Yao, Y.~Dai, L.~Gu, C.~Yu, and L.~Sun, ``Personalized edge
  intelligence via federated self-knowledge distillation,'' \emph{IEEE
  Transactions on Parallel and Distributed Systems}, vol.~34, no.~2, pp.
  567--580, 2023.

\bibitem{shoham1910overcoming}
N.~Shoham, T.~Avidor, A.~Keren, N.~Israel, D.~Benditkis, L.~Mor-Yosef, and
  I.~Zeitak, ``Overcoming forgetting in federated learning on non-iid data,''
  in \emph{NeurIPS Workshop}, 2019.

\bibitem{xu2017training}
Z.~Xu, Y.~C. Hsu, and J.~Huang, ``Training shallow and thin networks for
  acceleration via knowledge distillation with conditional adversarial
  networks,'' in \emph{International Conference on Learning Representations
  Workshop}, 2018.

\bibitem{liu2020learning}
Y.~Liu, W.~Zhang, and J.~Wang, ``Learning from a lightweight teacher for
  efficient knowledge distillation,'' \emph{arXiv preprint arXiv:2005.09163},
  2020.

\bibitem{gao2020private}
C.~Zhuo, D.~Gao, and L.~Liu, ``Pkdgan: Private knowledge distillation with
  generative adversarial networks,'' \emph{IEEE Transactions on Big Data}, pp.
  1--14, 2022.

\bibitem{wang2018kdgan}
X.~Wang, R.~Zhang, Y.~Sun, and J.~Qi, ``{KDGAN}: Knowledge distillation with
  generative adversarial networks,'' \emph{Advances in Neural Information
  Processing Systems}, vol.~31, 2018.

\bibitem{wang2018adversarial}
Y.~Wang, C.~Xu, C.~Xu, and D.~Tao, ``Adversarial learning of portable student
  networks,'' in \emph{Proceedings of the AAAI Conference on Artificial
  Intelligence}, vol.~32, 2018.

\bibitem{Distilling}
G.~Hinton, O.~Vinyals, and J.~Dean, ``Distilling the knowledge in a neural
  network,'' in \emph{Conference on Neural Information Processing Systems
  (NeurIPS)}, 2014.

\bibitem{han2023fedal}
P.~Han, X.~Shi, and J.~Huang, ``Fedal: Black-box federated knowledge
  distillation enabled by adversarial learning,'' \emph{arXiv preprint
  arXiv:2311.16584}, 2024.

\bibitem{ben2010theory}
S.~Ben-David, J.~Blitzer, K.~Crammer, A.~Kulesza, F.~Pereira, and J.~W.
  Vaughan, ``A theory of learning from different domains,'' \emph{Machine
  learning}, vol.~79, no.~1, pp. 151--175, 2010.

\bibitem{Understanding}
S.~Shalev~Shwartz and S.~Ben~David, \emph{Understanding Machine Learning: From
  Theory to Algorithms}.\hskip 1em plus 0.5em minus 0.4em\relax USA: Cambridge
  University Press, 2014.

\bibitem{bistritz2020distributedDF}
I.~Bistritz, A.~Mann, and N.~Bambos, ``Distributed distillation for on-device
  learning,'' \emph{Advances in Neural Information Processing Systems},
  vol.~33, pp. 22\,593--22\,604, 2020.

\bibitem{37648}
Y.~Netzer, T.~Wang, A.~Coates, A.~Bissacco, B.~Wu, and A.~Y. Ng, ``Reading
  digits in natural images with unsupervised feature learning,'' in
  \emph{NeurIPS Workshop on Deep Learning and Unsupervised Feature Learning},
  2011.

\bibitem{CIFAR10}
A.~Krizhevsky and G.~Hinton, ``Learning multiple layers of features from tiny
  images,'' University of Toronto, Tech. Rep., 2009.

\bibitem{darlow2018cinic}
L.~N. Darlow, E.~J. Crowley, A.~Antoniou, and A.~J. Storkey, ``Cinic-10 is not
  imagenet or cifar-10,'' \emph{arXiv preprint arXiv:1810.03505}, 2018.

\bibitem{hsu2019measuring}
T.-M.~H. Hsu, H.~Qi, and M.~Brown, ``Measuring the effects of non-identical
  data distribution for federated visual classification,'' \emph{arXiv preprint
  arXiv:1909.06335}, 2019.

\bibitem{yurochkin2019bayesian}
M.~Yurochkin, M.~Agarwal, S.~Ghosh, K.~Greenewald, N.~Hoang, and Y.~Khazaeni,
  ``Bayesian nonparametric federated learning of neural networks,'' in
  \emph{International conference on machine learning}.\hskip 1em plus 0.5em
  minus 0.4em\relax PMLR, 2019, pp. 7252--7261.

\bibitem{2014Adam}
D.~Kingma and J.~Ba, ``Adam: A method for stochastic optimization,''
  \emph{Computer Science}, 2014.

\bibitem{ghosh2018multi}
A.~Ghosh, V.~Kulharia, V.~P. Namboodiri, P.~H. Torr, and P.~K. Dokania,
  ``Multi-agent diverse generative adversarial networks,'' in \emph{Proceedings
  of the IEEE conference on computer vision and pattern recognition}, 2018, pp.
  8513--8521.

\end{thebibliography}

\begin{IEEEbiography}
	[{\includegraphics[width=1in,height=1.25in,clip,keepaspectratio]{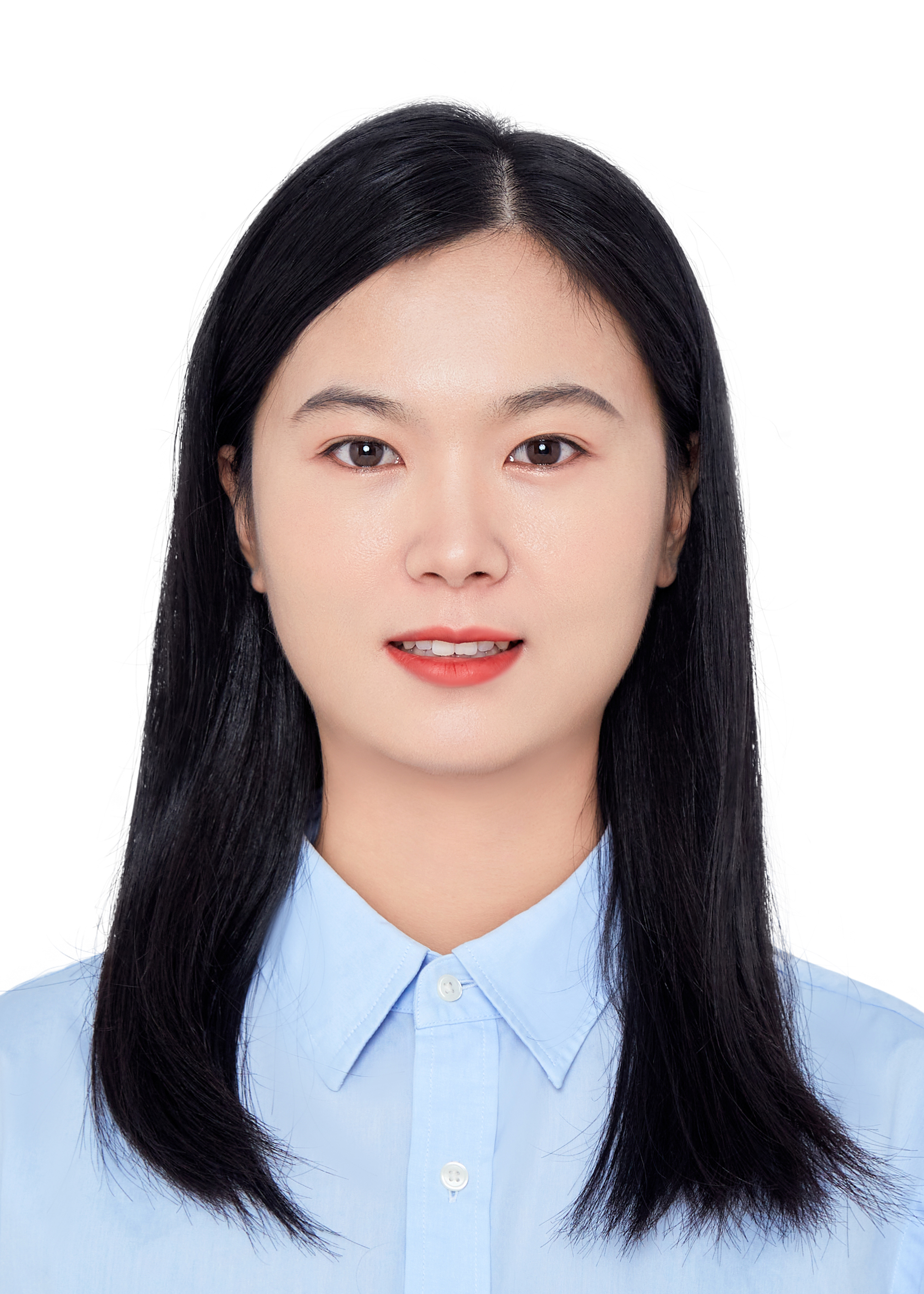}}]{Pengchao Han} received the Ph.D. degree in communication and information systems at Northeastern University, China, in 2021. She has conducted research at Imperial College London from 2018 to 2019. She was a  Postdoc research associate at The Chinese University of Hong Kong, Shenzhen, China, from 2021 to 2023. She is currently an associate professor at Guangdong University of Technology. Her research interests include wireless and optical networks, mobile edge computing, federated learning, and knowledge distillation. 
\end{IEEEbiography}

\begin{IEEEbiography}[{\includegraphics[width=1in,height=1.25in,clip,keepaspectratio]{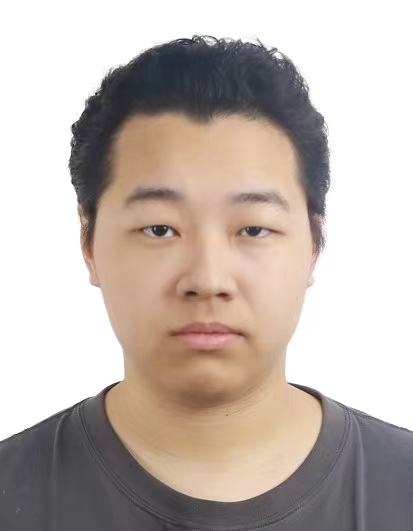}}]{Xingyan Shi} is an undergraduate student at The Chinese University of Hong Kong, Shenzhen. His research interests include federated learning, reinforcement learning, multi-agent learning, and computer vision. 
\end{IEEEbiography}

\begin{IEEEbiography}
	[{\includegraphics[width=1in,height=1.25in,clip,keepaspectratio]{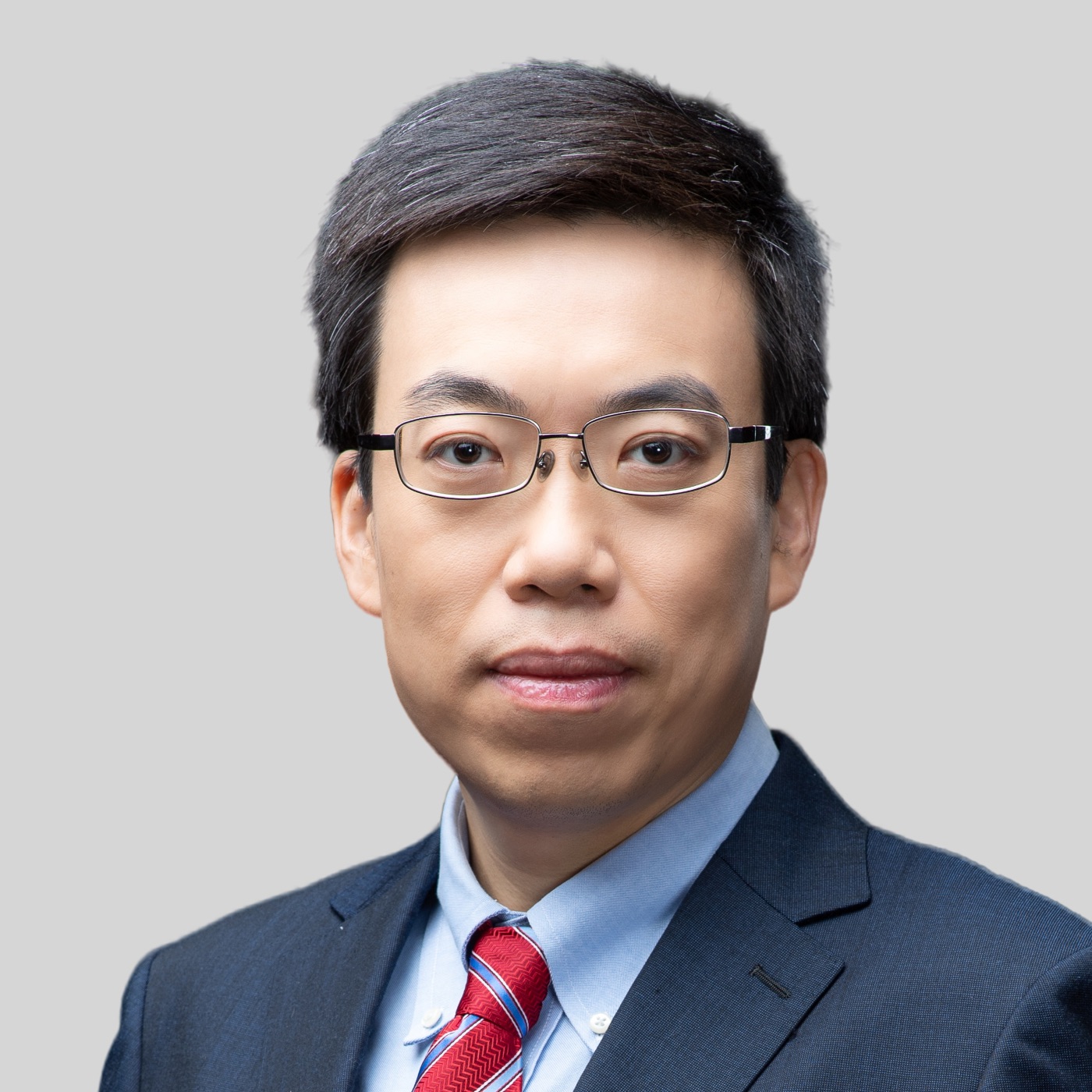}}]	
{Jianwei Huang} is a Presidential Chair Professor and Associate Vice President (Institutional Development) of the Chinese University of Hong Kong, Shenzhen, and the Associate Director of Shenzhen Institute of Artificial Intelligence and Robotics for Society. He received his Ph.D. from Northwestern University in 2005 and worked as a Postdoc Research Associate at Princeton University during 2005-2007. His research interests are network optimization and economics, with applications in communication networks, energy networks, data markets, and crowd intelligence. He has published 350+ papers in leading international venues, with a Google Scholar citation of 17,000+ and an H-index of 68. He has co-authored 11 Best Paper Awards, including the 2011 IEEE Marconi Prize Paper Award in Wireless Communications. He has co-authored seven books, including the textbook "Wireless Network Pricing." He has been an IEEE Fellow, an IEEE ComSoc Distinguished Lecturer, a Clarivate Web of Science Highly Cited Researcher, and an Elsevier Most Cited Chinese Researcher. He is the Editor-in-Chief of IEEE Transactions on Network Science and Engineering and Associate Editor-in-Chief of the IEEE Open Journal of the Communications Society.
\end{IEEEbiography}

\vfill

\clearpage
\onecolumn
\section{Supplementary Material}
\subsection{Paramilitaries}\label{sup:pre}
\begin{proposition} (Proposition 1 in \cite{ghosh2018multi})
	\label{prop:lpCons}
	Given ${\bf y} = (y_1, \cdots, y_n)$, $y_i \geq 0$, and $a_i \in \mathbb{R}$, the optimal solution for the problem defined below is achieved at $y_i^* = \frac{a_i}{\sum_{i' = 1}^n a_{i'}}, \forall i$,
	\begin{align}
		\max_{\bf y} \sum_{i=1}^n a_i\log y_i,  \;\; s.t. \; \sum_i^n y_i = 1.
	\end{align}
\end{proposition}

\begin{theorem}[Uniform Convergence~\cite{Understanding}]~\label{thm:uniform_convergence}
	Let $\mathcal{F}$ be a class and let $\tau_{\mathcal{F}}$ be its growth function. Denote $L_{\mathcal{D}}(h)$ the expected risk of $h$ over all data points that follow distribution $\mathcal{D}$.  Then, for every $\mathcal{D}$ and every $\delta \in (0, 1)$, with probability of at least $1 - \delta$ over the choice of $ S \sim \mathcal{D}^{m} $ with $m$ denoting the number of data samples, we have
	\begin{equation}
		 \left|L_{\mathcal{D}} (h) - L_{S} (h)\right|  \leq \frac{ 4 + \sqrt{ \log ( \tau_{\mathcal{F}} (2m) ) } }{ \delta \sqrt{ 2m } }. 
	\end{equation}
\end{theorem}

\begin{theorem}[Domain adaptation \cite{ben2010theory}] \label{thm:simplified_adaptation_stochastic_bound} Let $\mathcal{F}$ be the hypothesis set. 
	Denote $L_{\mathcal{D}}(h)$ the expected risk of $h$ over all data points that follow distribution $\mathcal{D}$. Considering the distributions $\mathcal{D}_S$ and $\mathcal{D}_T$,
	for every $h \in \mathcal{F}$ and any $\delta \in (0, 1)$,
	with probability at least $1 - \delta$ (over the choice of the samples),
	there exists:
	\begin{align}
		\textstyle
		L_{\mathcal{D}_T}(h)
		\leq L_{\mathcal{D}_S}(h) + \frac{1}{2} d_{\mathcal{H} \Delta \mathcal{H}} (\mathcal{D}_S, \mathcal{D}_T) + \lambda^\ast \,,
	\end{align}
	where $\lambda^\ast = L_{\mathcal{D}_S}(h^ \ast) + L_{\mathcal{D}_T}(h^ \ast)$ and
	$h^* := \arg\min_{h \in \mathcal{F}} L_{\mathcal{D}_S} (h) + L_{\mathcal{D}_T} (h) $
	corresponds to \emph{ideal joint hypothesis} that minimizes the combined error. $d_{\mathcal{H} \Delta \mathcal{H}} (\cdot)$ denotes the domain discrepancy between two distributions.
\end{theorem}

\subsection{Proof of Lemma~\ref{lem:dis_best}}
{
\begin{proof}
The objective of the discriminator in \eqref{eq:min-max-obj} is
	\begin{align}
		  &- 
   \frac{1}{N}\sum_{n=1}^N\mathscr{E} \left(n,  h\left(p_n\left(\boldsymbol{\theta}_n\right),\boldsymbol{w}\right)\right) \nonumber\\
		& = -
  \sum_{n=1}^{N} \mathbb{E}_{p \sim p_n \left(\boldsymbol{\theta}_n\right)} \mathscr{E}\left(n, h\left(p,\boldsymbol{w}\right)\right) \nonumber\\
  & = 
  \sum_{n=1}^{N}   p_n\left(\boldsymbol{\theta}_n\right) \log \left([h\left(p_n\left(\boldsymbol{\theta}_n\right),\boldsymbol{w}\right)]_n\right).\label{eq:dis_trans}
	\end{align}
Using Proposition 1 in \cite{ghosh2018multi}, i.e., Proposition \ref{prop:lpCons}, maximizing \eqref{eq:dis_trans} over ${\boldsymbol{w}}$ with the fact that $\sum_{n=1}^N [h\left(p_n\left(\boldsymbol{\theta}_n\right),\boldsymbol{w}\right)]_n=1$ gives the optimal distribution learned by the discriminator as \eqref{eq:best_dis}.
\end{proof}
}

\subsection{Proof of Lemma \ref{lem:client_best}}

{
\begin{proof}
	Given $\boldsymbol{\Theta}_{-n}$ and $\boldsymbol{w}$, applying Lemma~\ref{lem:dis_best} for client $n$ gives
	\begin{align} 
	&  -
 \mathscr{E}\left(n,h\left(p_n\left(\boldsymbol{\theta}_n\left( \boldsymbol{\Theta}_{-n} , \boldsymbol{w}\right)\right),\boldsymbol{w}\right) \right)	\nonumber\\
	&  =\sum_{m=1,m\neq n}^{N}  p_m\left(\boldsymbol{\theta}_m\right) \log \frac{p_m\left(\boldsymbol{\theta}_m\right)}{\sum_{m'=1}^{N} p_{m'}\left(\boldsymbol{\theta}_{m'}\right)/N}\nonumber\\
  &+ p_n\left(\boldsymbol{\theta}_n\left( \boldsymbol{\Theta}_{-n} , \boldsymbol{w}\right)\right) \log \frac{p_n\left(\boldsymbol{\theta}_n\left( \boldsymbol{\Theta}_{-n} , \boldsymbol{w}\right)\right)}{\sum_{m'=1}^{N} p_{m'}\left(\boldsymbol{\theta}_{m'}\right)/N}- N\log N \nonumber
 \end{align}
 \begin{align}
	&  =\sum_{m=1,m\neq n}^{N} \mathscr{K}\left(p_m\left(\boldsymbol{\theta}_m\right), \sum_{m'=1}^{N}\frac{ p_{m'}\left(\boldsymbol{x},\boldsymbol{\theta}_{m'}\right)}{N}\right) \nonumber\\
 &+ \mathscr{K}\left(p_n\left(\boldsymbol{\theta}_n\left( \boldsymbol{\Theta}_{-n}, \boldsymbol{w}\right)\right), \sum_{m'=1}^{N} \frac{p_{m'}\left(\boldsymbol{x},\boldsymbol{\theta}_{m'}\right)}{N}\right)
	- N\log N. \label{eq:reversk_kl}
	\end{align}
	The KL divergence in \eqref{eq:reversk_kl}  reaches its minimum when $p_n\left(\boldsymbol{\theta}_n^\ast\left( \boldsymbol{\Theta}_{-n} , \boldsymbol{w}\right)\right)={\sum_{n=1}^{N} p_n\left(\boldsymbol{x},\boldsymbol{\theta}_n\right)}/{N}$. 
\end{proof}
}

\subsection{Proof of Lemma \ref{lem:consensus}}
\begin{proof}
	Let $\tilde{p}:=\sum_{n=1}^N{p_n}/N$. Applying Pinsker inequality~\cite{pmlr-v97-hoang19a} and Assumption~\ref{asmp:al} 
	to the reverse KL divergence in \eqref{eq:reversk_kl} gives 
	\begin{align}
	\left\Vert\!p_n\!-\!\tilde{p}\!\right\Vert _1 \leq \sqrt{2\log 2 \mathscr{K}\left(p_n,\tilde {p}\right)}\!\leq\!\sqrt{2\log 2 \zeta_n}, \forall n \in \mathcal{N}.\label{eq:pinkser}
	\end{align} 
Recall that $p_{n,k}$ denotes the probability of client $n$'s model output for classes $k$.
	Using \eqref{eq:pinkser} for two clients $n$ and $m$, we have for any $k\in \left\{1, \ldots, K\right\}$ 
	\begin{align} 
		&\left|p_{n,k} - \tilde {p}_{k}\right| + \left|p_{m,k} - \tilde {p}_{k}\right| \nonumber\\
  &\leq \sum_k \left|p_{n,k} - \tilde {p}_{k}\right| + \sum_k \left|p_{m,k} - \tilde {p}_{k}\right| \nonumber\\
  &\leq \sqrt{2\log 2 \zeta_n} + \sqrt{2\log 2 \zeta_m}.\label{eq:prediction_dif}
	\end{align}
	The first-order Taylor's expansion on $-\log\left(\cdot\right)$ gives  
	\begin{align}
		&-\log p_{m,k}\nonumber\\
  &=-\log p_{n,k}-\frac{1}{p_{n,k}}\left(p_{m,k}-p_{n,k}\right) + \frac{1}{2p_{c,k}^2}\left(p_{m,k}-p_{n,k}\right)^2 \label{eq:taylor}\\
		&
		\leq-\log p_{n,k}+\frac{1}{p_{n,k}}\left(\sqrt{2\log 2 \zeta_n} + \sqrt{2\log 2 \zeta_m}\right),\label{eq:taylor_log}
		\end{align}
  where \eqref{eq:taylor} is due to Taylor remainder theorem for some $p_{c,k}$ between $p_{m,k}$ and $p_{n,k}$.
  
	For the cross-entropy loss function, using the first-order Taylor's expansion on $-\log\left(\cdot\right)$, i.e., \eqref{eq:taylor_log}, and the fact that $\mathscr{E}_{\mathcal{P}}\left(f_n\right)\geq0$, we can obtain
	\begin{align}
	&\mathscr{E}_{\mathcal{P}}\left(f_m\right) - \mathscr{E}_{\mathcal{P}}\left(f_n\right) 
 = \log p_{n,y} - \log p_{m,y}  \nonumber\\
	&\leq  \frac{1}{p_{n,y}}  \left(\sqrt{2\log 2 \zeta_n} + \sqrt{2\log 2 \zeta_m}\right) \nonumber\\
	&=e^{-\mathscr{E}_{\mathcal{P}}\left(f_n\right)}\left(\sqrt{2\log 2 \zeta_n} + \sqrt{2\log 2 \zeta_m}\right).
	\end{align}
	where the last line is due to the definition of cross-entropy loss function.
	Considering $\mathscr{E}_{\mathcal{P}}\left(f_n\right)\geq0$, indicating that $e^{-\mathscr{E}_{\mathcal{P}}\left(f_n\right)}\leq 1$, we have
	\begin{align}	
	& \mathscr{E}_{\mathcal{P}}\left(f_m\right)\leq  \mathscr{E}_{\mathcal{P}}\left(f_n\right) +\sqrt{2\log 2 \zeta_n} + \sqrt{2\log 2 \zeta_m}. \label{eq:consensun_public}
	\end{align}
	Taking into account the domain discrepancy between the local dataset of client $n$ and the public dataset, we have
	\begin{align}
	\mathscr{E}_{\mathcal{D}_n}(f_m)
	&\leq \mathscr{E}_{\mathcal{P}}(f_m) + \frac{1}{2} d_{\mathcal{H} \Delta \mathcal{H}} (\mathcal{P}, \mathcal{D}_n) + \lambda_n^\ast \label{eq:line1}\\
	&\leq \mathscr{E}_{\mathcal{P}}\left(f_n\right) +\sqrt{2\log 2 \zeta_n} + \sqrt{2\log 2 \zeta_m}\nonumber\\
 &\ \ \ + \frac{1}{2} d_{\mathcal{H} \Delta \mathcal{H}} (\mathcal{P}, \mathcal{D}_n) + \lambda_n^\ast \label{eq:line3}\\
	&\leq \mathscr{E}_{\mathcal{D}_n}(f_n) + \frac{1}{2} d_{\mathcal{H} \Delta \mathcal{H}} (\mathcal{P}, \mathcal{D}_n) + \lambda_n^\ast + \sqrt{2\log 2 \zeta_n} \nonumber\\
 &\ \ \ + \sqrt{2\log 2 \zeta_m}+ \frac{1}{2} d_{\mathcal{H} \Delta \mathcal{H}} (\mathcal{P}, \mathcal{D}_n) + \lambda_n^\ast \label{eq:line2}\\ 
	&\leq \mathscr{E}_{\mathcal{D}_n}(f_n) + \sqrt{2\log 2 \zeta_n} + \sqrt{2\log 2 \zeta_m}\nonumber\\
 &\ \ \ + 
	d_{\mathcal{H} \Delta \mathcal{H}} (\mathcal{P}, \mathcal{D}_n) + 2\lambda_n^\ast.
	\end{align}
Here, \eqref{eq:line1} and \eqref{eq:line2} are due to the domain adaptation theorem in \cite{ben2010theory}, i.e., Theorem~\ref{thm:simplified_adaptation_stochastic_bound}. Equation \eqref{eq:line3} is from \eqref{eq:consensun_public}.
This completes the proof. 
\end{proof}

\subsection{Proof of Theorem ~\ref{thm:generalization_bound}}
{

\begin{proof}
The generalization bound of \texttt{FedAL} is as follows: 
	\begin{align}
		&\frac{1}{N}\sum_{n=1}^N \mathscr{E}_{\mathcal{D}}\left(f_n\right)
		=\frac{1}{N}\sum_{n=1}^N \sum_{m=1}^N \frac{\left|\mathcal{D}_m\right|}{\left|\mathcal{D}\right|} \mathscr{E}_{\mathcal{D}_m}\left(f_n\right) \nonumber\\
		& \leq \frac{1}{N^2} \sum_{n=1}^{N} \left(\sum_{m=1, m\neq n}^{N} \left[\mathscr{E}_{\mathcal{D}_m}\left(f_m\right) +\sqrt{2\log 2 \zeta_m} \right.\right.\nonumber\\
  &\left.\left.+ \sqrt{2\log 2 \zeta_n}+d_{\mathcal{H} \Delta \mathcal{H}} (\mathcal{P}, \mathcal{D}_m) + 2\lambda_m^\ast\right]+ \mathscr{E}_{\mathcal{D}_n}\left(f_n\right) \right) \label{eq:line4} \\
		& = \frac{1}{N} \sum_{m=1}^{N}  \mathscr{E}_{\mathcal{D}_m}\left(f_m\right) +\frac{1}{N^2} \sum_{n=1}^{N}\sum_{m=1, m\neq n}^{N} \nonumber\\
  &\left(\sqrt{2\log 2 \zeta_n} + \sqrt{2\log 2 \zeta_m}+d_{\mathcal{H} \Delta \mathcal{H}} (\mathcal{P}, \mathcal{D}_m) + 2\lambda_m^\ast\right)   \nonumber \\
		& \leq  \frac{1}{N} \sum_{n=1}^{N}  \mathscr{E}_{\hat{\mathcal{D}}_n}\left(f_n\right)  + \frac{ 4 + \sqrt{ \log ( \epsilon_{\mathcal{F}} (2\phi) ) } }{ \delta \sqrt{ 2\phi }} \nonumber\\
  &+
		\frac{N-1}{N^2}\sum_{n=1}^N\left[2\sqrt{2\log 2 \zeta_n}+d_{\mathcal{H} \Delta \mathcal{H}} (\mathcal{P}, \mathcal{D}_n)+2\lambda_{n}^\ast\right].  \label{eq:line55} 
	\end{align}
Here, \eqref{eq:line4} is due to Lemma \ref{lem:consensus}.
The last line uses the uniform convergence theorem of deep learning~\cite{Understanding}, i.e., the Theorem~\ref{thm:uniform_convergence}. 
We notice that \eqref{eq:line55} is equivalent to \eqref{eq:bound_FDAL}. Thus, we finish the proof.
\end{proof}
}

\subsection{Proof of Theorem~\ref{thm:convergence}}\label{apx:theorem_convergence}
For simplicity, we use $f_n\left(\boldsymbol{\theta}_n\right)$, $p_n\left(\boldsymbol{\theta}_n\right)$, and $\bar{p}_{-n}$ for short of $f_n\left(\boldsymbol{x},\boldsymbol{\theta}_n\right)$, $p_n\left(\boldsymbol{x},\boldsymbol{\theta}_n\right)$, and $\bar{p}_{-n}\left(\boldsymbol{x}\right)$ respectively. Throughout this proof, we define the gradient of a function with the output vector $\boldsymbol{a}$ over its variable $\boldsymbol{x}$ as $\nabla_{\boldsymbol{x}}\boldsymbol{a}$ to represent the Jacobian matrix of $\boldsymbol{a}$. Assumption~\ref{asmp:local_models} and \ref{asmp:discriminator} give the following inequalities:
\begin{itemize}
		\item $L_{\mathrm{f}}-$Lipschitz continuity of $f_n\left(\boldsymbol{\theta}_n\right)$: 
		\begin{equation}\label{eq:asmf2}
			\left\Vert\nabla_{\boldsymbol{\theta}_n}f_n\left(\boldsymbol{\theta}_n\right)\right\Vert \leq L_{\mathrm{f}}, \forall n \in \mathcal{N},
		\end{equation} 
		\item $L_{\mathrm{f}}-$Lipschitz smooth of $f_n\left(\boldsymbol{\theta}_n\right)$: 
		\begin{equation}\label{eq:asmf3}
			\left\Vert\nabla_{\boldsymbol{\theta}_n}f_n\left(\boldsymbol{\theta}_n\right)-\nabla_{\boldsymbol{\theta}'_n}f_n\left(\boldsymbol{\theta}'_n\right)\right\Vert^2 \leq L_{\mathrm{f}} \left\Vert\boldsymbol{\theta}-\boldsymbol{\theta'}\right\Vert^2, \forall n \in \mathcal{N},
		\end{equation}
		\item $L_{\mathrm{p}}-$Lipschitz continuity of $p_n\left(\boldsymbol{\theta}_n\right)$:
		\begin{equation}\label{eq:asmp1}
			\left\Vert p_n\left(\boldsymbol{\theta}_n\right)-p_n\left(\boldsymbol{\theta}'_n\right)\right\Vert  \leq L_{\mathrm{p}} \left\Vert\boldsymbol{\theta}_n-\boldsymbol{\theta}'_n\right\Vert, \forall n \in \mathcal{N},
		\end{equation} 	
		\begin{equation}	\label{eq:asmp2}
			\left\Vert\nabla_{\boldsymbol{\theta}_n}p_n\left(\boldsymbol{\theta}_n\right)\right\Vert \leq L_{\mathrm{p}}, \forall n \in \mathcal{N},
		\end{equation} 		
		\item $L_{\mathrm{h}}-$Lipschitz continuity of $h\left(p_n,\boldsymbol{w}\right)$: 
		\begin{equation}\label{eq:asma2}
			\left\Vert h\left(p_n,\boldsymbol{w}\right)-h\left(p'_n,\boldsymbol{w}\right)\right\Vert  \leq L_{\mathrm{h}} \left\Vert p_n-p'_n\right\Vert,
		\end{equation}
		\begin{equation}\label{eq:asma3}
	\left\Vert\nabla_{p_n}h\left(p_n,\boldsymbol{w}\right)\right\Vert\leq L_{\mathrm{h}},
		\end{equation}
  \begin{equation}\label{eq:asma2w}
			\left\Vert h\left(p_n,\boldsymbol{w}\right)-h\left(p_n,\boldsymbol{w}'\right)\right\Vert  \leq L_{\mathrm{h}} \left\Vert \boldsymbol{w}-\boldsymbol{w}'\right\Vert,
		\end{equation}
		\begin{equation}\label{eq:asma3w}
	\left\Vert\nabla_{\boldsymbol{w}}h\left(p_n,\boldsymbol{w}\right)\right\Vert\leq L_{\mathrm{h}}.
		\end{equation}
\end{itemize}

The objective of client $n$ at training round $t$ is to minimize
\begin{align}
	 \mathcal{V}_n\left(\boldsymbol{x},y,\boldsymbol{\theta}_n^t, \boldsymbol{w}^t\right) &:=  \ell \left(y,f_n\left(\boldsymbol{\theta}_n^t\right)\right)+\mathscr{K}\left({p}_{n}\left({\boldsymbol{\theta}}^{t,0}_n\right), p_n\left(\boldsymbol{\theta}_n^t\right)\right) \nonumber\\
  &+\mathscr{K}\left(\bar{p}_{-n}, p_n\left(\boldsymbol{\theta}_n^t\right)\right) 
	- \mathscr{E} \left(n,  h\left(p_n\left(\boldsymbol{\theta}_n^t\right),\boldsymbol{w}^t\right)\right)+\mathscr{K}\left({p}_{n}\left({\boldsymbol{\theta}}^{t+\frac{1}{2},0}_n\right), p_n\left(\boldsymbol{\theta}_n^t\right)\right)
\end{align}

The objective of the discriminator at time $t$ is to maximize
\begin{equation}
    \mathcal{U}\left(\boldsymbol{x},\boldsymbol{\Theta}^t,\boldsymbol{w}^t\right):=-\frac{1}{N}\sum_{n=1}^N\mathscr{E} \left(n,  h\left(p_n\left(\boldsymbol{\theta}_n^t\right),\boldsymbol{w}^t\right)\right)
\end{equation} 
We start by showing the Lipschitz property of $\mathcal{V}_n\left(\boldsymbol{\theta}_n\right)$ and $\mathcal{U}\left(\boldsymbol{w}\right)$, and the unbiased gradients with bounded variances for all loss functions. Then, we prove the convergence error of \texttt{FedAL}.

\textbf{
\begin{lemma}\label{lem:L_continuity U} Under Assumptions~\ref{asmp:local_models} - \ref{asm:bounded_variance_o},
	$V_n^{\mathrm{loc}}\left(\boldsymbol{\theta}_n\right)$, $V_n^{\mathrm{glo}}\left(\boldsymbol{\theta}_n\right)$ and $\mathcal{U}\left(\boldsymbol{w}\right)$ are $L_{\mathrm{v}}$, $L_{\mathrm{w}}$, and $L_{\mathrm{u}}-$ Lipschitz smooth with some constants $L_{\mathrm{v}}$, $L_{\mathrm{w}}$, and $L_{\mathrm{u}}$. 
\end{lemma}}
\begin{proof}
	\textbf{
The gradient of $\mathscr{K}\left(\bar{p}_{-n},p_n\left(\boldsymbol{\theta}_n\right)\right)$} is:
\begin{align}	
	&\nabla_{\boldsymbol{\theta}_n}\mathscr{K}\left(\bar{p}_{-n},p_n\left(\boldsymbol{\theta}_n\right)\right)=\left[\left[ \nabla_{\boldsymbol{\theta}_n} {f_{n,k}\left(\boldsymbol{\theta}_n\right)}\right]^T \nabla_{f_{n,k}}{\mathscr{K}\left(\psi\left(\bar{f}_{-n}\right),\psi\left(f_n\left(\boldsymbol{\theta}_n\right)\right)\right)},\forall k \right]\nonumber\\
	&= \left[ \left[ \nabla_{\boldsymbol{\theta}_n} {f_{n,k}\left(\boldsymbol{\theta}_n\right)}\right]^T \nabla_{f_{n,k}}{\left[- \sum_{k'} \frac{\exp\left({\bar f_{-n,k'}}\right)}{\sum_{k''}\exp\left({\bar{f}_{-n,k''}}\right)}\log \frac{\exp\left({f_{n,k'}}\right)}{\sum_{k''}\exp\left({{f}_{n,k''}}\right)} \right]},\forall k\right]\\
	& = \left[ \underset{\mathrm{The}\ k\mathrm{th}\ \mathrm{element}}{\underbrace{\left[ \nabla_{\boldsymbol{\theta}_n} {f_{n,k}\left(\boldsymbol{\theta}_n\right)}\right]^T\left[-\frac{\exp\left(\bar{f}_{-n,k}\right)}{\sum_{k''}\exp\left(\bar{f}_{-n,k''}\right)}
 {{\frac{\sum_{k'}\exp\left(f_{n,k''}\right)}{\exp\left(f_{n,k}\right)}}}\left[
 {{\frac{\exp\left(f_{n,k}\right)}{\sum_{k''}\exp\left(f_{n,k''}\right)}}}
 {{-\frac{\exp\left(f_{n,k}\right)}{\left(\sum_{k''}\exp\left(f_{n,k''}\right)\right)^{2}}\exp\left(f_{n,k}\right)}}\right]\right.}}\right.+ \nonumber\\
 &\left.\underset{\mathrm{Other}\ k'\neq k\ \mathrm{elements}}{\underbrace{\left.\sum_{k'\neq k}\left[-\frac{\exp\left(\bar{f}_{-n,k'}\right)}{\sum_{k''}\exp\left(\bar{f}_{-n,k''}\right)}\frac{\sum_{k''}\exp\left(f_{n,k''}\right)}{\exp\left(f_{n,k'}\right)}\frac{-\exp\left(f_{n,k'}\right)}{\left(\sum_{k''}\exp\left(f_{n,k''}\right)\right)^{2}}\exp\left(f_{n,k}\right)\right]\right]}}, \forall k\right]\\
	& = \left[ \left[\nabla_{\boldsymbol{\theta}_n} {f_{n,k}\left(\boldsymbol{\theta}_n\right)}\right]^T\left(\frac{\exp\left(f_{n,k}\right)}{\sum_{k'}\exp\left(f_{n,k'}\right)}-\frac{\exp\left(\bar{f}_{-n,k}\right)}{\sum_{k'}\exp\left(\bar{f}_{-n,k'}\right)}\right) , \forall k \right] 
	=  \left[\nabla_{\boldsymbol{\theta}_n} {f_{n}\left(\boldsymbol{\theta}_n\right)}\right]^T\left(p_{n}\left(\boldsymbol{\theta}_n\right)-\bar{p}_{-n}\right) \label{eq:gradient_kl}
\end{align}
We can show  \textbf{the
Lipschitz smoothness of $\mathscr{K}\left(\bar{p}_{-n},p_n\left(\boldsymbol{\theta}_n\right)\right)$} as:
\begin{align}	&\left\Vert \nabla_{\boldsymbol{\theta}_n}\mathscr{K}\left(\bar{p}_{-n},p_n\left(\boldsymbol{\theta}_n\right)\right) - 	\nabla_{\boldsymbol{\theta}'_n}\mathscr{K}\left(\bar{p}_{-n},p_n\left(\boldsymbol{\theta}'_n\right)\right)\right\Vert ^2 \nonumber\\
	& = \left\Vert \left[\nabla_{\boldsymbol{\theta}_n} {f_n\left(\boldsymbol{\theta}_n\right)}\right]^T\left(p_{n}\left(\boldsymbol{\theta}_n\right)-\bar{p}_{-n}\right) -  \left[\nabla_{\boldsymbol{\theta}'_n} {f_n\left(\boldsymbol{\theta}'_n\right)}\right]^T\left(p_{n}\left(\boldsymbol{\theta}'_n\right)-\bar{p}_{-n}\right)\right\Vert ^2 \nonumber\\
	& \leq 2\left\Vert \left[\nabla_{\boldsymbol{\theta}_n} {f_n\left(\boldsymbol{\theta}_n\right)}\right]^T\left(p_{n}\left(\boldsymbol{\theta}_n\right)-\bar{p}_{-n}\right)  - \left[\nabla_{\boldsymbol{\theta}_n} {f_n\left(\boldsymbol{\theta}_n\right)}\right]^T\left(p_{n}\left(\boldsymbol{\theta}'_n\right)-\bar{p}_{-n}\right)\right\Vert ^2 \nonumber\\& +2\left\Vert  \left[\nabla_{\boldsymbol{\theta}_n} {f_n\left(\boldsymbol{\theta}_n\right)}\right]^T\left(p_{n}\left(\boldsymbol{\theta}'_n\right)-\bar{p}_{-n}\right) -  \left[\nabla_{\boldsymbol{\theta}'_n} {f_n\left(\boldsymbol{\theta}'_n\right)}\right]^T\left(p_{n}\left(\boldsymbol{\theta}'_n\right)-\bar{p}_{-n}\right)\right\Vert ^2 
	\tag{By Jensen's inequality}\\
	& \leq 2\left\Vert \nabla_{\boldsymbol{\theta}_n} {f_n\left(\boldsymbol{\theta}_n\right)}\right\Vert ^2 \left\Vert p_{n}\left(\boldsymbol{\theta}_n\right)-p_{n}\left(\boldsymbol{\theta}'_n\right)\right\Vert ^2+2 \left\Vert \nabla_{\boldsymbol{\theta}_n} {f_n\left(\boldsymbol{\theta}_n\right)} - \nabla_{\boldsymbol{\theta}'_n} {f_n\left(\boldsymbol{\theta}'_n\right)}\right\Vert ^2\left\Vert   p_{n}\left(\boldsymbol{\theta}'_n\right)-\bar{p}_{-n}\right\Vert ^2
	\tag{By submultiplicativity of norm: $\left\Vert AB\right\Vert ^{2}\leq\left\Vert A\right\Vert ^{2}\left\Vert B\right\Vert ^{2}$}\\
	& \leq 2L_{\mathrm{p}}^2 L_{\mathrm{f}}^2 \left\Vert \boldsymbol{\theta}_n-\boldsymbol{\theta}'_n\right\Vert ^2+4L_{\mathrm{f}}^2 \left\Vert \boldsymbol{\theta}_n-\boldsymbol{\theta}'_n\right\Vert ^2
 \end{align}
where the last line is due to Assumption~\ref{asmp:local_models}, i.e., \eqref{eq:asmf2}, \eqref{eq:asmf3}, \eqref{eq:asmp1}, and the fact that the difference between two probability vectors is bounded by $\sqrt{2}$.

Similarly, we have \textbf{
the gradient of $\mathscr{K}\left({p}_{n}\left({\boldsymbol{\theta}}^{*,0}_n\right), p_n\left(\boldsymbol{\theta}_n\right)\right)$} as:
\begin{align}	
\nabla_{\boldsymbol{\theta}_n}\mathscr{K}\left({p}_{n}\left({\boldsymbol{\theta}}^{*,0}_n\right), p_n\left(\boldsymbol{\theta}_n\right)\right) = \left[\nabla_{\boldsymbol{\theta}_n} {f_n}\left(\boldsymbol{\theta}_n\right)\right]^T\left(p_{n}\left(\boldsymbol{\theta}_n\right)-{p}_{n}\left({\boldsymbol{\theta}}^{*,0}_n\right)\right) . 
\end{align}

We can show  \textbf{the Lipschitz smoothness of $\mathscr{K}\left({p}_{n}\left({\boldsymbol{\theta}}^{*,0}_n\right), p_n\left(\boldsymbol{\theta}_n\right)\right)$} as:
\begin{align}	&\left\Vert \nabla_{\boldsymbol{\theta}_n}\mathscr{K}\left({p}_{n}\left({\boldsymbol{\theta}}^{*,0}_n\right), p_n\left(\boldsymbol{\theta}_n\right)\right) - 	\nabla_{\boldsymbol{\theta}'_n}\mathscr{K}\left({p}_{n}\left({\boldsymbol{\theta}}^{*,0}_n\right), p_n\left(\boldsymbol{\theta}'\right)\right)\right\Vert ^2 \nonumber\\
	& \leq 2L_{\mathrm{p}}^2 L_{\mathrm{f}}^2 \left\Vert \boldsymbol{\theta}_n-\boldsymbol{\theta}'_n\right\Vert ^2+4L_{\mathrm{f}}^2 \left\Vert \boldsymbol{\theta}_n-\boldsymbol{\theta}'_n\right\Vert ^2
\end{align}

\textbf{
The gradient of $\mathscr{K}\left({p}_{n}\left({\boldsymbol{\theta}}^{*+\frac{1}{2},0}_n\right), p_n\left(\boldsymbol{\theta}_n\right)\right)$} is:
\begin{align}	
\nabla_{\boldsymbol{\theta}_n}\mathscr{K}\left({p}_{n}\left({\boldsymbol{\theta}}^{*,0}_n\right), p_n\left(\boldsymbol{\theta}_n\right)\right) = \left[\nabla_{\boldsymbol{\theta}_n} {f_n}\left(\boldsymbol{\theta}_n\right)\right]^T\left(p_{n}\left(\boldsymbol{\theta}_n\right)-{p}_{n}\left({\boldsymbol{\theta}}^{*+\frac{1}{2},0}_n\right)\right) . 
\end{align}

\textbf{
The Lipschitz smoothness of $\mathscr{K}\left({p}_{n}\left({\boldsymbol{\theta}}^{*+\frac{1}{2},0}_n\right), p_n\left(\boldsymbol{\theta}_n\right)\right)$} is expressed as:
\begin{align}	&\left\Vert \nabla_{\boldsymbol{\theta}_n}\mathscr{K}\left({p}_{n}\left({\boldsymbol{\theta}}^{*+\frac{1}{2},0}_n\right), p_n\left(\boldsymbol{\theta}_n\right)\right) - 	\nabla_{\boldsymbol{\theta}'_n}\mathscr{K}\left({p}_{n}\left({\boldsymbol{\theta}}^{*+\frac{1}{2},0}_n\right), p_n\left(\boldsymbol{\theta}'\right)\right)\right\Vert ^2 \nonumber\\
	& \leq 2L_{\mathrm{p}}^2 L_{\mathrm{f}}^2 \left\Vert \boldsymbol{\theta}_n-\boldsymbol{\theta}'_n\right\Vert ^2+4L_{\mathrm{f}}^2 \left\Vert \boldsymbol{\theta}_n-\boldsymbol{\theta}'_n\right\Vert ^2
\end{align}

\textbf{
The gradient of $-\mathscr{E}\left(n,h\left(p_n\left(\boldsymbol{\theta}_n\right),\boldsymbol{w}\right)\right)$} over $\boldsymbol{\theta}_n$ is:
\begin{align}
	&-\nabla_{\boldsymbol{\theta}_n}\mathscr{E}\left(n,h\left(p_n\left(\boldsymbol{\theta}_n\right),\boldsymbol{w}\right)\right)\nonumber\\
	&= \left(\nabla_{p_n}h\left(p_n\left(\boldsymbol{\theta}_n\right),\boldsymbol{w}\right)\nabla_{\boldsymbol{\theta}_n}p_n\left(\boldsymbol{\theta}_n\right) \right)^T \nabla_{h} \log \frac{\exp\left({\left(h\left(p_n\left(\boldsymbol{\theta}_n\right),\boldsymbol{w}\right)\right)_n}\right)}{\sum_{m=1}^N \exp\left({\left(h\left(p_n\left(\boldsymbol{\theta}_n\right),\boldsymbol{w}\right)\right)_m}\right)} \nonumber\\
	& = \left(\nabla_{p_n}h\left(p_n\left(\boldsymbol{\theta}_n\right),\boldsymbol{w}\right)\nabla_{\boldsymbol{\theta}_n}p_n\left(\boldsymbol{\theta}_n\right) \right)^T  \left(1-h\left(p_n\left(\boldsymbol{\theta}_n\right),\boldsymbol{w}\right)_{n}\right) \label{eq:gradient_ha}
\end{align}

\textbf{
The gradient of $-\mathscr{E}\left(n,h\left(p_n\left(\boldsymbol{\theta}_n\right),\boldsymbol{w}\right)\right)$} over $\boldsymbol{w}$ is:
\begin{align}
	&-\nabla_{\boldsymbol{w}}\mathscr{E}\left(n,h\left(p_n\left(\boldsymbol{\theta}_n\right),\boldsymbol{w}\right)\right)\nonumber\\
	&= \left(\nabla_{\boldsymbol{w}}h\left(p_n\left(\boldsymbol{\theta}_n\right),\boldsymbol{w}\right) \right)^T \nabla_{h} \log \frac{\exp\left({\left(h\left(p_n\left(\boldsymbol{\theta}_n\right),\boldsymbol{w}\right)\right)_n}\right)}{\sum_{m=1}^N \exp\left({\left(h\left(p_n\left(\boldsymbol{\theta}_n\right),\boldsymbol{w}\right)\right)_m}\right)} \nonumber\\
	& = \left(\nabla_{\boldsymbol{w}}h\left(p_n\left(\boldsymbol{\theta}_n\right),\boldsymbol{w}\right) \right)^T  \left(1-h\left(p_n\left(\boldsymbol{\theta}_n\right),\boldsymbol{w}\right)_{n}\right)
\end{align}

We can show  \textbf{the Lipschitz smoothness of $-\mathscr{E}\left(n,h\left(p_n\left(\boldsymbol{\theta}_n\right),\boldsymbol{w}\right)\right)$} as:
\begin{align}
	& \left\Vert \nabla_{\boldsymbol{\theta}_n}\mathscr{E}\left(n,h\left(p_n\left(\boldsymbol{\theta}_n\right),\boldsymbol{w}\right)\right)-\nabla_{\boldsymbol{\theta}'_n}\mathscr{E}\left(n,h\left(p_n\left(\boldsymbol{\theta}'_n\right),\boldsymbol{w}\right)\right)\right\Vert ^2 \nonumber\\
	&=\!\left\Vert \left(\!\nabla_{p_n}\!h\!\left(\!p_n\left(\boldsymbol{\theta}_n\!\right),\!\boldsymbol{w}\!\right)\!\nabla_{\boldsymbol{\theta}_n}\!p_n\!\left(\!\boldsymbol{\theta}_n\!\right)\!\right)^T\left(\!1\!-\!h\left(p_n\left(\boldsymbol{\theta}_n\right),\boldsymbol{w}\right)_n\right)\!-\!\left(\nabla_{p_n}\!h\left(p_n\!\left(\!\boldsymbol{\theta}'_n\right),\!\boldsymbol{w}\!\right)\nabla_{\boldsymbol{\theta}'_n}p_n\left(\boldsymbol{\theta}'_n\right)\!\right)^T\!\left(\!1\!-\!h\!\left(\!p_n\!\left(\boldsymbol{\theta}'_n\!\right)\!,\!\boldsymbol{w}\right)_n\right)\right\Vert ^2 \nonumber\\
	& \leq L_{\mathrm{h}}^2L_{\mathrm{p}}^2 \left\Vert h\left(p_n\left(\boldsymbol{\theta}_n\right),\boldsymbol{w}\right)_n-h\left(p_n\left(\boldsymbol{\theta}'_n\right),\boldsymbol{w}\right)_n\right\Vert ^2 \tag{By Assumptions~\ref{asmp:local_models} and ~\ref{asmp:discriminator}, i.e., \eqref{eq:asmp2} and \eqref{eq:asma3}}\\
	& \leq  L_{\mathrm{h}}^2L_{\mathrm{p}}^2 L_{\mathrm{h}}^2 \left\Vert p_n\left(\boldsymbol{\theta}_n\right)-p_n\left(\boldsymbol{\theta}'_n\right)\right\Vert ^2 \leq  L_{\mathrm{h}}^4L_{\mathrm{p}}^4 \left\Vert \boldsymbol{\theta}_n-\boldsymbol{\theta}'_n\right\Vert ^2 \\
 \tag{By Assumptions~\ref{asmp:local_models} and ~\ref{asmp:discriminator}, i.e., \eqref{eq:asmp1} and \eqref{eq:asma2}}
\end{align}

\begin{align}
	& \left\Vert \nabla_{\boldsymbol{w}}\mathscr{E}\left(n,h\left(p_n\left(\boldsymbol{\theta}_n\right),\boldsymbol{w}\right)\right)-\nabla_{\boldsymbol{w}'_n}\mathscr{E}\left(n,h\left(p_n\left(\boldsymbol{\theta}'_n\right),\boldsymbol{w}\right)\right)\right\Vert ^2 \nonumber\\
	&=\!\left\Vert \left(\!\nabla_{\boldsymbol{w}}\!h\!\left(\!p_n\left(\boldsymbol{\theta}_n\!\right),\!\boldsymbol{w}\!\right)\!\right)^T\left(\!1\!-\!h\left(p_n\left(\boldsymbol{\theta}_n\right),\boldsymbol{w}\right)_n\right)-\left(\nabla_{\boldsymbol{w}'}\!h\left(p_n\!\left(\!\boldsymbol{\theta}_n\right),\!\boldsymbol{w}'\!\right)\right)^T\!\left(\!1\!-\!h\!\left(\!p_n\!\left(\boldsymbol{\theta}_n\!\right)\!,\!\boldsymbol{w}'\right)_n\right)\right\Vert^2 \nonumber\\
	& \leq L_{\mathrm{h}}^2 \left\Vert h\left(p_n\left(\boldsymbol{\theta}_n\right),\boldsymbol{w}\right)_n-h\left(p_n\left(\boldsymbol{\theta}_n\right),\boldsymbol{w}'\right)_n\right\Vert ^2 \tag{By Assumption~\ref{asmp:discriminator}, i.e., \eqref{eq:asma3w}}\\
	& \leq  L_{\mathrm{h}}^4 \left\Vert \boldsymbol{w}-\boldsymbol{w}'\right\Vert ^2 \\
 \tag{By Assumption~\ref{asmp:discriminator}, i.e., \eqref{eq:asma2w}}
\end{align}

Based on the Lipschitz smoothness of all the components, the objective functions $V_n^{\mathrm{loc}}\left(\boldsymbol{\theta}_n\right)$ and $V_n^{\mathrm{glo}}\left(\boldsymbol{\theta}_n\right)$ are also Lipschitz smooth functions in $\boldsymbol{\theta}_n$ with some constants $L_{\mathrm{v}}$ and $L_{\mathrm{w}}$.
Similarly, $\mathcal{U}\left(\boldsymbol{w}\right)$ is a Lipschitz-smooth function in $\boldsymbol{w}$ with a constant $L_{\mathrm{u}}$.
\end{proof}

Let $\tilde{g}_n\left(\boldsymbol{\theta}_n\right) := g^{\mathrm{l}}_n\left(\boldsymbol{\theta}_n\right) + g_n^{\mathrm{rl}}\left(\boldsymbol{\theta}_n\right)$ denotes the stochastic gradient of local training at client $n$ and  $\hat{g} _n\left(\boldsymbol{\theta}_n\right):=g^{\mathrm{k}}_n\left(\boldsymbol{\theta}_n\right)+g^{\mathrm{u}}_n\left(\boldsymbol{\theta}_n\right)+g^{\mathrm{rg}}_n\left(\boldsymbol{\theta}_n\right)$ indicates the stochastic gradient of global knowledge transfer at client $n$. Moreover, let $\varphi\left(\boldsymbol{w}\right)$ be the stochastic gradient of $\mathcal{U}\left(\boldsymbol{w}\right)$ over $\boldsymbol{w}$.
We have the following lemmas.

\begin{lemma}[Unbiased stochastic gradients] \label{lem:umbiased_grad} Under Assumptions~\ref{asmp:local_models} - \ref{asm:bounded_variance_o}, for a mini-batch ${\mathcal{B}}$, the following hold: $\mathbb{E}_{\mathcal{B}} \left[\tilde{g}_n\left(\boldsymbol{\theta}_n\right)\right]  =\nabla_{\boldsymbol{\theta}_n} V_n^{\mathrm{loc}}\left(\boldsymbol{\theta}_n\right)$,  $\mathbb{E}_{\mathcal{B}} \left[\hat{g}_n\left(\boldsymbol{\theta}_n\right)\right]  =\nabla_{\boldsymbol{\theta}_n} V_n^{\mathrm{glo}}\left(\boldsymbol{\theta}_n\right)$, and $\mathbb{E}_{\mathcal{B}} \left[\varphi\left(\boldsymbol{w}\right)\right]  =\nabla_{\boldsymbol{w}} \mathcal{U}\left(\boldsymbol{w}\right)$.
\end{lemma}
\begin{proof}

The expected stochastic gradient of $V_n^{\mathrm{loc}}\left(\boldsymbol{\theta}_n\right)$ over $\boldsymbol{\theta}_n$ is	
\begin{align}
		&\mathbb{E}_{\mathcal{B}} \left[\tilde{g}_n\left(\boldsymbol{\theta}_n\right)\right]  \nonumber\\
		&\!=\!\mathbb{E}_{\mathcal{B}}\!\left[\nabla_{\boldsymbol{\theta}_n}\ell\left(y,\left(f_n\left(\boldsymbol{\theta}_n\right)\right)\right)\!+\nabla_{\boldsymbol{\theta}_n}r^{\mathrm{loc}}\left(\boldsymbol{x},\boldsymbol{\theta}_n\right)\right] \nonumber\\
		& = \mathbb{E}_{\left\{\boldsymbol{x},y\right\}\in \mathcal{D}_n}\left[\nabla_{\boldsymbol{\theta}_n}\ell\left(y,\left(f_n\left(\boldsymbol{\theta}_n\right)\right)\right)+\nabla_{\boldsymbol{\theta}_n}r^{\mathrm{loc}}\left(\boldsymbol{x},\boldsymbol{\theta}_n\right)\right]  \nonumber\\
		& = \nabla_{\boldsymbol{\theta}_n} V_n^{\mathrm{loc}}\left(\boldsymbol{\theta}_n^{t}\right)
	\end{align}

 The expected stochastic gradient of $V_n^{\mathrm{glo}}\left(\boldsymbol{\theta}_n\right)$ over $\boldsymbol{\theta}_n$ is	
\begin{align}
		&\mathbb{E}_{\mathcal{B}} \left[\hat{g}_n\left(\boldsymbol{\theta}_n\right)\right]  \nonumber\\
		&\!=\!\mathbb{E}_{\mathcal{B}}\!\left[\!\nabla_{\boldsymbol{\theta}_n}\mathscr{K}\left(\bar{p}_{-n},p_n\left(\boldsymbol{\theta}_n\right)\right)\!-\!\nabla_{\boldsymbol{\theta}_n}\mathscr{E} \left(n,  h\left(p_n\left(\boldsymbol{\theta}_n\right),\boldsymbol{w}^t\right)\right)+\nabla_{\boldsymbol{\theta}_n}r^{\mathrm{glo}}\left(\boldsymbol{x},\boldsymbol{\theta}_n\right) \right] \nonumber\\
		& =\mathbb{E}_{\boldsymbol{x}\in \mathcal{P}}\left[\nabla_{\boldsymbol{\theta}_n}\mathscr{K}\left(\bar{p}_{-n},p_n\left(\boldsymbol{\theta}_n\right)\right) - \nabla_{\boldsymbol{\theta}_n}\mathscr{E} \left(n,  h\left(p_n\left(\boldsymbol{\theta}_n\right),\boldsymbol{w}^t\right)\right) +\nabla_{\boldsymbol{\theta}_n}r^{\mathrm{glo}}\left(\boldsymbol{x},\boldsymbol{\theta}_n\right)\right]\nonumber\\
		& = \nabla_{\boldsymbol{\theta}_n} V_n^{\mathrm{glo}}\left(\boldsymbol{\theta}_n^{t}\right)
	\end{align}
 
The expected stochastic gradient of $\mathcal{U}_n\left(\boldsymbol{w}\right)$ over $\boldsymbol{w}$ is	
\begin{align}
		\mathbb{E}_{\mathcal{B}} \left[\varphi\left(\boldsymbol{w}^t\right)\right]  & = \mathbb{E}_{\mathcal{B}}\left[-\frac{1}{N}\sum_{n=1}^N   \nabla_{\boldsymbol{w}^t}\mathscr{E} \left(n,  h\left(p_n\left(\boldsymbol{\theta}_n\right),\boldsymbol{w}^t\right)\right) \right] \nonumber\\
		& = \mathbb{E}_{\boldsymbol{x}\in \mathcal{P}} \left[ --\frac{1}{N}\sum_{n=1}^N  \nabla_{\boldsymbol{w}_n}\mathscr{E} \left(n,  h\left(p_n\left(\boldsymbol{\theta}_n\right),\boldsymbol{w}^t\right)\right) \right] = \nabla_{\boldsymbol{w}} \mathcal{U}\left(\boldsymbol{w}\right)
	\end{align}

 Overall, we have proved the lemma.
\end{proof}

\begin{lemma}[Bounded variance of stochastic gradients]\label{lem:bounded_variance} Under Assumptions~\ref{asmp:local_models} - \ref{asm:bounded_variance_o},
    there exist constants $\sigma_{\mathrm{v}}, \sigma_{\mathrm{w}}, \sigma_{\mathrm{u}}\leq \infty$ such that the variance of the stochastic gradients $\tilde{g}_n\left(\boldsymbol{\theta}_n\right)$, $\hat{g}_n\left(\boldsymbol{\theta}_n\right)$, and  $\varphi\left(\boldsymbol{w}\right)$ are bounded as
    \begin{equation}
        \mathbb{E}_{\mathcal{B}}\left[\left\Vert \nabla_{\boldsymbol{\theta}_n}V_n^{\mathrm{loc}}\left(\boldsymbol{\theta}_n\right)- \tilde{g}_n\left(\boldsymbol{\theta}_n\right)\right\Vert^2\right]\leq \sigma_{\mathrm{v}}^2,
    \end{equation}
     \begin{equation}
        \mathbb{E}_{\mathcal{B}}\left[\left\Vert \nabla_{\boldsymbol{\theta}_n}V_n^{\mathrm{glo}}\left(\boldsymbol{\theta}_n\right)- \hat{g}_n\left(\boldsymbol{\theta}_n\right)\right\Vert^2\right]\leq \sigma_{\mathrm{w}}^2,
    \end{equation}
    \begin{equation}
        \mathbb{E}_{\mathcal{B}}\left[\left\Vert \nabla_{\boldsymbol{w}}\mathcal{U}\left(    \boldsymbol{w}\right)- \varphi\left(\boldsymbol{w}\right)\right\Vert^2\right]\leq \sigma_{\mathrm{u}}^2,
    \end{equation}
    for a mini-batch ${\mathcal{B}}$.
\end{lemma}
\begin{proof}
For the loss function $V_n^{\mathrm{loc}}\left(\boldsymbol{\theta}_n\right)$, there exists

\begin{align}
       & \mathbb{E}_{\mathcal{B}}\left[\left\Vert \nabla_{\boldsymbol{\theta}_n}V_n^{\mathrm{loc}}\left(\boldsymbol{\theta}_n\right)- \tilde{g}_n\left(\boldsymbol{\theta}_n\right)\right\Vert^2 \right]\nonumber \\
       	&\!=\!\mathbb{E}_{\mathcal{B}}\!\left[\left\Vert\!\nabla_{\boldsymbol{\theta}_n}\ell\left(y,\left(f_n\left(\boldsymbol{x},\boldsymbol{\theta}_n\right)\right)\right)\!+\nabla_{\boldsymbol{\theta}_n}r^{\mathrm{loc}}\left(\boldsymbol{x},\boldsymbol{\theta}_n\right)-g^{\mathrm{l}}_n\left(\boldsymbol{\theta}_n\right)-g^{\mathrm{rl}}_n\left(\boldsymbol{\theta}_n\right)\right\Vert^2
        \right] \nonumber\\
		& \leq 2\mathbb{E}_{\mathcal{B}}\!\left[\left\Vert\nabla_{\boldsymbol{\theta}_n}\ell\left(y,\left(f_n\left(\boldsymbol{x},\boldsymbol{\theta}_n\right)\right)\right)- g^{\mathrm{l}}_n\left(\boldsymbol{\theta}_n\right)\right\Vert^2
        \right] +2\mathbb{E}_{\mathcal{B}}\!\left[\left\Vert\nabla_{\boldsymbol{\theta}_n}r^{\mathrm{loc}}\left(\boldsymbol{x},\boldsymbol{\theta}_n\right)- g^{\mathrm{rl}}_n\left(\boldsymbol{\theta}_n\right)\right\Vert^2
        \right] \nonumber\\
        \tag{By Jensen's inequality}\\
        & \leq 2\sigma_{\mathrm{l}}^2 + 2\sigma_{\mathrm{rl}}^2  := \sigma_{\mathrm{v}}^2
    \end{align}
    Similarly we have,
    \begin{align}
       & \mathbb{E}_{\mathcal{B}}\left[\left\Vert \nabla_{\boldsymbol{\theta}_n}V_n^{\mathrm{glo}}\left(\boldsymbol{\theta}_n\right)- \hat{g}_n\left(\boldsymbol{\theta}_n\right)\right\Vert^2 \right] \leq 3\sigma_{\mathrm{k}}^2 + 3\sigma_{\mathrm{n}}^2+ 3\sigma_{\mathrm{rg}}^2  := \sigma_{\mathrm{w}}^2
    \end{align}
    \begin{align}
        &\mathbb{E}_{\mathcal{B}}\left[\left\Vert \nabla_{\boldsymbol{w}}\mathcal{U}\left(    \boldsymbol{w}\right)- \varphi\left(\boldsymbol{w}\right)\right\Vert^2\right]
        \leq \sum_{n=1}^N \mathbb{E}_{\mathcal{B}}\left[\left\Vert \nabla_{\boldsymbol{w}}U_n\left(\boldsymbol{x},{\boldsymbol{\theta}}_{n},\boldsymbol{w}\right)- \varphi_n\left(\boldsymbol{w}\right)\right\Vert^2\right]\leq N\sigma_{\mathrm{n}}^2:= \sigma_{\mathrm{u}}^2
    \end{align}
\end{proof}
\begin{lemma}[Model update in one round] \label{lem:one-round-update}
    Under Assumptions~\ref{asmp:local_models} - \ref{asm:bounded_variance_o}, if $\eta_l^t\leq \frac{1}{2\max\left\{L_{\mathrm{v}},L_{\mathrm{w}}\right\}\tau}$ and $\eta_d^t\leq \frac{1}{2L_{\mathrm{u}}\tau}$, then the conditional expected squared norm differences of gradients for multiple iterations are bounded as follows:
    \begin{align}
    \sum_{i=0}^{\tau-1}\mathbb{E}^t\left[\left\Vert \tilde{g}_n\left(\boldsymbol{\theta}_n^{t,i}\right)-\tilde{g}_n\left(\boldsymbol{\theta}_n^{t,0}\right)\right\Vert^2\right] \leq 8\tau^3\left(\eta_l^t\right)^2L_{\mathrm{v}}^2\left(\left\Vert \nabla_{\boldsymbol{\theta}_n^{t,0}}V_n^{\mathrm{loc}}\left(\boldsymbol{\theta}_n^{t,0}\right)\right\Vert^2+\sigma_{\mathrm{v}}^2\right) 
\end{align}
\begin{align}
     \sum_{i=0}^{\tau-1}\mathbb{E}^t\left[\left\Vert \hat{g}_n\left(\boldsymbol{\theta}_n^{t+\frac{1}{2},i}\right)-\tilde{g}_n\left(\boldsymbol{\theta}_n^{t,0}\right)\right\Vert^2\right] \leq 8 \tau^3\left(\eta_l^t\right)^2\left(2L_{\mathrm{v}}^2+L_{\mathrm{w}}^2\right)\left(\left\Vert \nabla_{\boldsymbol{\theta}_n^{t,0}}V_n^{\mathrm{loc}}\left(\boldsymbol{\theta}_n^{t,0}\right)\right\Vert^2+\sigma_{\mathrm{v}}^2\right) 
\end{align}
\begin{align}
    \sum_{i=0}^{\tau-1}\mathbb{E}^t\left[\left\Vert \varphi_n\left(\boldsymbol{w}^{t,i}\right)-\varphi\left(\boldsymbol{w}^{t,0}\right)\right\Vert^2\right] 
    &\leq 8\tau^3\left(\eta_d^t\right)^2L_{\mathrm{u}}^2\left(\left\Vert \nabla_{\boldsymbol{w}_n^{t,0}}\mathcal{U}\left(\boldsymbol{w}^{t,0}\right)\right\Vert^2+\sigma_{\mathrm{u}}^2\right)
\end{align}
where  $\mathbb{E}_t[\cdot] := \mathbb{E}[\cdot|\mathbf{x}_t]$.
\end{lemma}

 \begin{proof}
 
In \texttt{FedAL}, each training round $t$ contains two processes, i.e., local training and global knowledge transfer. There exist multiple, indicated by $\tau$, iterations in both local training and global knowledge transfer. 
We first compute the expected squared norm difference of the client $n$' gradient from the start of a training round $t$ to the $i$th iteration of local training in this training round. Then, we derive the expected squared norm difference of the client $n$' gradient from the start of the training round $t$ to the $i$th iteration of the global knowledge transfer of this training round.

For local training, we have
\begin{align}
    &\mathbb{E}^t\left[\left\Vert \tilde{g}_n\left(\boldsymbol{\theta}_n^{t,i}\right)-\tilde{g}_n\left(\boldsymbol{\theta}_n^{t,0}\right)\right\Vert^2\right]\nonumber\\
    &\leq \mathbb{E}^t\left[\left\Vert \tilde{g}_n\left(\boldsymbol{\theta}_n^{t,i}\right)-\tilde{g}_n\left(\boldsymbol{\theta}_n^{t,i-1}\right)+\tilde{g}_n\left(\boldsymbol{\theta}_n^{t,i-1}\right)-\tilde{g}_n\left(\boldsymbol{\theta}_n^{t,0}\right)\right\Vert^2\right]\nonumber\\
    & \leq \left(1+\tau\right)\mathbb{E}^t\left[\left\Vert \tilde{g}_n\left(\boldsymbol{\theta}_n^{t,i}\right)-\tilde{g}_n\left(\boldsymbol{\theta}_n^{t,i-1}\right)\right\Vert^2\right]+ 
    \left(1+\frac{1}{\tau}\right)\mathbb{E}^t\left[\left\Vert \tilde{g}_n\left(\boldsymbol{\theta}_n^{t,i-1}\right)-\tilde{g}_n\left(\boldsymbol{\theta}_n^{t,0}\right)\right\Vert^2\right] \nonumber\\
    \tag{By $\left(X+Y\right)^2\leq\left(1+a\right)X^2+\left(1+\frac{1}{a}\right)Y^2$ for some positive $a$ and we let $a=\tau$} \\
    & \leq \left(1+\tau\right)L_{\mathrm{v}}^2\mathbb{E}^t\left[\left\Vert \boldsymbol{\theta}_n^{t,i}-\boldsymbol{\theta}_n^{t,i-1}\right\Vert^2\right]+ 
    \left(1+\frac{1}{\tau}\right)\mathbb{E}^t\left[\left\Vert \tilde{g}_n\left(\boldsymbol{\theta}_n^{t,i-1}\right)-\tilde{g}_n\left(\boldsymbol{\theta}_n^{t,0}\right)\right\Vert^2\right] \nonumber\\
    \tag{By $L_{\mathrm{v}}$-smoothness of $V_n^{\mathrm{loc}}\left(\boldsymbol{\theta}_n\right)$ by Lemma \ref{lem:L_continuity U}}\\
    & \leq \left(1+\tau\right)\left(\eta_l^t\right)^2L_{\mathrm{v}}^2\mathbb{E}^t\left[\left\Vert \tilde{g}_n\left(\boldsymbol{\theta}_n^{t,i-1}\right)\right\Vert^2\right]+ 
    \left(1+\frac{1}{\tau}\right)\mathbb{E}^t\left[\left\Vert \tilde{g}_n\left(\boldsymbol{\theta}_n^{t,i-1}\right)-\tilde{g}_n\left(\boldsymbol{\theta}_n^{t,0}\right)\right\Vert^2\right] \tag{By gradient update rule}\\
    & \leq \left(1+\tau\right)\left(\eta_l^t\right)^2L_{\mathrm{v}}^2\mathbb{E}^t\left[\left\Vert \tilde{g}_n\left(\boldsymbol{\theta}_n^{t,i-1}\right)-\tilde{g}_n\left(\boldsymbol{\theta}_n^{t,0}\right)+\tilde{g}_n\left(\boldsymbol{\theta}_n^{t,0}\right)\right\Vert^2\right]+ 
    \left(1+\frac{1}{\tau}\right)\mathbb{E}^t\left[\left\Vert \tilde{g}_n\left(\boldsymbol{\theta}_n^{t,i-1}\right)-\tilde{g}_n\left(\boldsymbol{\theta}_n^{t,0}\right)\right\Vert^2\right] \nonumber \\
    & \leq 2\left(1+\tau\right)\left(\eta_l^t\right)^2L_{\mathrm{v}}^2\mathbb{E}^t\left[\left\Vert \tilde{g}_n\left(\boldsymbol{\theta}_n^{t,i-1}\right)-\tilde{g}_n\left(\boldsymbol{\theta}_n^{t,0}\right)\right\Vert^2\right] + 2\left(1+\tau\right)\left(\eta_l^t\right)^2L_{\mathrm{v}}^2\mathbb{E}^t\left[\left\Vert \tilde{g}_n\left(\boldsymbol{\theta}_n^{t,0}\right)\right\Vert^2\right] 
     \tag{By Jensen's inequality}\\
    & + \left(1+\frac{1}{\tau}\right)\mathbb{E}^t\left[\left\Vert \tilde{g}_n\left(\boldsymbol{\theta}_n^{t,i-1}\right)-\tilde{g}_n\left(\boldsymbol{\theta}_n^{t,0}\right)\right\Vert^2\right] \nonumber \\
    & \leq \left(1+\frac{2}{\tau}\right)\mathbb{E}^t\left[\left\Vert \tilde{g}_n\left(\boldsymbol{\theta}_n^{t,i-1}\right)-\tilde{g}_n\left(\boldsymbol{\theta}_n^{t,0}\right)\right\Vert^2\right] + 2\left(1+\tau\right)\left(\eta_l^t\right)^2L_{\mathrm{v}}^2\mathbb{E}^t\left[\left\Vert \tilde{g}_n\left(\boldsymbol{\theta}_n^{t,0}\right)\right\Vert^2\right] \tag{Let $\eta_l^t\leq \frac{1}{2L_{\mathrm{v}}\tau}$}
\end{align}

We define the following notation for simplicity:
\begin{align}
    &A^{t,i}:= \mathbb{E}^t\left[\left\Vert \tilde{g}_n\left(\boldsymbol{\theta}_n^{t,i}\right)-\tilde{g}_n\left(\boldsymbol{\theta}_n^{t,0}\right)\right\Vert^2\right]\\
    &B:=2\left(1+\tau\right)\left(\eta_l^t\right)^2L_{\mathrm{v}}^2\mathbb{E}^t\left[\left\Vert \tilde{g}_n\left(\boldsymbol{\theta}_n^{t,0}\right)\right\Vert^2\right]\\
    &C:=\left(1+\frac{2}{\tau}\right)
\end{align}

We have 
\begin{align}
    A^{t,i}\leq CA^{t,i-1}+B
\end{align}
We can show that
\begin{align}
   & A^{t,1}\leq CA^{t,0}+B \nonumber\\
   & A^{t,2}\leq CA^{t,1}+B \leq C^2A^{t,0}+CB+B\nonumber\\
   & A^{t,3}\leq CA^{t,2}+B \leq C^3A^{t,0}+C^2B+CB+B\nonumber\\
   & \ldots \nonumber \\
   & A^{t,i}\leq C^iA^{t,0}+B\sum_{j=0}^{i-1}C^{j}\nonumber
\end{align}

Note that $A^{t,0}=\mathbb{E}^t\left[\left\Vert \tilde{g}_n\left(\boldsymbol{\theta}_n^{t,0}\right)-\tilde{g}_n\left(\boldsymbol{\theta}_n^{t,0}\right)\right\Vert^2\right]=0$. For the second part, we have
\begin{align}
    &\sum_{i=0}^{\tau-1} B\sum_{j=0}^{i-1}C^{j} \nonumber\\
    &= B\sum_{i=0}^{\tau-1} \frac{C^i-1}{C-1} = \frac{B}{C-1}\sum_{i=0}^{\tau-1} \left(C^i-1\right)= \frac{B}{C-1} \frac{C^{\tau}-1}{C-1}-\tau\tag{$\sum_{i=0}^{N-1}x^i=\frac{x^N-1}{X-1}$}\\
    &=\frac{B}{\frac{2}{\tau}} \frac{\left(1+\frac{2}{\tau}\right)^{\tau}-1}{\frac{2}{\tau}}-\tau\tag{$C:=\left(1+\frac{2}{\tau}\right)$}\\
    &\leq \frac{\tau^2B}{2} \left(\frac{e^2-1}{2}-1\right) \tag{$(1+\frac{n}{x})^x\leq e^n$}\\
    & \leq 2\tau^2B \nonumber\\
    & \leq 4\tau^2\left(1+\tau\right)\left(\eta_l^t\right)^2L_{\mathrm{v}}^2\mathbb{E}^t\left[\left\Vert \tilde{g}_n\left(\boldsymbol{\theta}_n^{t,0}\right)\right\Vert^2\right]\tag{$B:=2\left(1+\tau\right)\left(\eta_l^t\right)^2L_{\mathrm{v}}^2\mathbb{E}^t\left[\left\Vert \tilde{g}_n\left(\boldsymbol{\theta}_n^{t,0}\right)\right\Vert^2\right]$}\\
    & \leq 8\tau^3\left(\eta_l^t\right)^2L_{\mathrm{v}}^2\mathbb{E}^t\left[\left\Vert \tilde{g}_n\left(\boldsymbol{\theta}_n^{t,0}\right)\right\Vert^2\right]\nonumber\\
    & \leq 8\tau^3\left(\eta_l^t\right)^2L_{\mathrm{v}}^2\left(\left\Vert \nabla_{\boldsymbol{\theta}_n^{t,0}}V_n^{\mathrm{loc}}\left(\boldsymbol{\theta}_n^{t,0}\right)\right\Vert^2+\sigma_{\mathrm{v}}^2\right) \tag{Lemmas \ref{lem:umbiased_grad} and \ref{lem:bounded_variance}, $\mathbb{E}^t\left[\left\Vert\mathbf{z}\right\Vert^2\right]\leq\left\Vert\mathbb{E}\left[\mathbf{z}\right]\right\Vert^2+\mathbb{E}^t\left[\left\Vert\mathbf{z}-\mathbb{E}^t\left[\mathbf{z}\right]\right\Vert^2\right]$ for any random variable $\mathbf{z}$}
\end{align}
Thus, we have
\begin{align}
    \sum_{i=0}^{\tau-1}\mathbb{E}^t\left[\left\Vert \tilde{g}_n\left(\boldsymbol{\theta}_n^{t,i}\right)-\tilde{g}_n\left(\boldsymbol{\theta}_n^{t,0}\right)\right\Vert^2\right] \leq 8\tau^3\left(\eta_l^t\right)^2L_{\mathrm{v}}^2\left(\left\Vert \nabla_{\boldsymbol{\theta}_n^{t,0}}V_n^{\mathrm{loc}}\left(\boldsymbol{\theta}_n^{t,0}\right)\right\Vert^2+\sigma_{\mathrm{v}}^2\right) 
\end{align}

For global knowledge transfer, we have
\begin{align}
    &\mathbb{E}^t\left[\left\Vert \hat{g}_n\left(\boldsymbol{\theta}_n^{t+\frac{1}{2},i}\right)-\tilde{g}_n\left(\boldsymbol{\theta}_n^{t,0}\right)\right\Vert^2\right]\nonumber\\
    & \leq \left(1+\frac{2}{\tau}\right)\mathbb{E}^t\left[\left\Vert \hat{g}_n\left(\boldsymbol{\theta}_n^{t+\frac{1}{2},i-1}\right)-\tilde{g}_n\left(\boldsymbol{\theta}_n^{t,0}\right)\right\Vert^2\right] + 2\left(1+\tau\right)\left(\eta_l^t\right)^2L_{\mathrm{w}}^2\mathbb{E}^t\left[\left\Vert \tilde{g}_n\left(\boldsymbol{\theta}_n^{t,0}\right)\right\Vert^2\right] \tag{Let $\eta_l^t\leq \frac{1}{2L_{\mathrm{w}}\tau}$}
\end{align}

We define the following notation for simplicity:
\begin{align}
    &A^{t+\frac{1}{2},i}:= \mathbb{E}^t\left[\left\Vert \hat{g}_n\left(\boldsymbol{\theta}_n^{t+\frac{1}{2},i}\right)-\tilde{g}_n\left(\boldsymbol{\theta}_n^{t,0}\right)\right\Vert^2\right]\\
    &B':=2\left(1+\tau\right)\left(\eta_l^t\right)^2L_{\mathrm{w}}^2\mathbb{E}^t\left[\left\Vert \tilde{g}_n\left(\boldsymbol{\theta}_n^{t,0}\right)\right\Vert^2\right]\\
    &C:=\left(1+\frac{2}{\tau}\right)
\end{align}

We have 
\begin{align}
    A^{t+\frac{1}{2},i}\leq CA^{t+\frac{1}{2},i-1}+B'
\end{align}
We can show that
\begin{align}
   & A^{t+\frac{1}{2},i}\leq C^iA^{t+\frac{1}{2},0}+B'\sum_{j=0}^{i-1}C^{j}
\end{align}

Note that $A^{t+\frac{1}{2},0}=\mathbb{E}^t\left[\left\Vert \hat{g}_n\left(\boldsymbol{\theta}_n^{t,0}\right)-\tilde{g}_n\left(\boldsymbol{\theta}_n^{t,0}\right)\right\Vert^2\right]=\mathbb{E}^t\left[\left\Vert \tilde{g}_n\left(\boldsymbol{\theta}_n^{t,\tau}\right)-\tilde{g}_n\left(\boldsymbol{\theta}_n^{t,0}\right)\right\Vert^2\right]=B\sum_{j=0}^{\tau-1}C^{j}=B\frac{C^{\tau}-1}{C-1}\leq B\frac{e^2-1}{2/\tau}\leq 4 B \tau $. We have
\begin{align}
    & \sum_{i=0}^{\tau-1}A^{t+\frac{1}{2},i} = \sum_{i=0}^{\tau-1} \left(4 B \tau +B'\sum_{j=0}^{i-1}C^{j}\right) \nonumber\\
   &\leq 4 B \tau^2 + 2B'\tau^2\nonumber\\
    & \leq 4\tau^2\left(1+\tau\right)\left(\eta_l^t\right)^2\left(2L_{\mathrm{v}}^2+L_{\mathrm{w}}^2\right)\mathbb{E}^t\left[\left\Vert \tilde{g}_n\left(\boldsymbol{\theta}_n^{t,0}\right)\right\Vert^2\right]\tag{$B:=2\left(1+\tau\right)\left(\eta_l^t\right)^2L_{\mathrm{v}}^2\mathbb{E}^t\left[\left\Vert \tilde{g}_n\left(\boldsymbol{\theta}_n^{t,0}\right)\right\Vert^2\right]$, $B':=2\left(1+\tau\right)\left(\eta_l^t\right)^2L_{\mathrm{w}}^2\mathbb{E}^t\left[\left\Vert \tilde{g}_n\left(\boldsymbol{\theta}_n^{t,0}\right)\right\Vert^2\right]$}\\
    & \leq 8\tau^3\left(\eta_l^t\right)^2\left(2L_{\mathrm{v}}^2+L_{\mathrm{w}}^2\right)\mathbb{E}^t\left[\left\Vert \tilde{g}_n\left(\boldsymbol{\theta}_n^{t,0}\right)\right\Vert^2\right]\nonumber\\
    & \leq 8\tau^3\left(\eta_l^t\right)^2\left(2L_{\mathrm{v}}^2+L_{\mathrm{w}}^2\right)\left(\left\Vert \nabla_{\boldsymbol{\theta}_n^{t,0}}V_n^{\mathrm{loc}}\left(\boldsymbol{\theta}_n^{t,0}\right)\right\Vert^2+\sigma_{\mathrm{v}}^2\right) \tag{Lemmas \ref{lem:umbiased_grad} and \ref{lem:bounded_variance}, $\mathbb{E}^t\left[\left\Vert\mathbf{z}\right\Vert^2\right]\leq\left\Vert\mathbb{E}\left[\mathbf{z}\right]\right\Vert^2+\mathbb{E}^t\left[\left\Vert\mathbf{z}-\mathbb{E}^t\left[\mathbf{z}\right]\right\Vert^2\right]$ for any random variable $\mathbf{z}$}
\end{align}
Thus, we have
\begin{align}
    \sum_{i=0}^{\tau-1}\mathbb{E}^t\left[\left\Vert \hat{g}_n\left(\boldsymbol{\theta}_n^{t+\frac{1}{2},i}\right)-\tilde{g}_n\left(\boldsymbol{\theta}_n^{t,0}\right)\right\Vert^2\right] \leq 8\tau^3\left(\eta_l^t\right)^2\left(2L_{\mathrm{v}}^2+L_{\mathrm{w}}^2\right)\left(\left\Vert \nabla_{\boldsymbol{\theta}_n^{t,0}}V_n^{\mathrm{loc}}\left(\boldsymbol{\theta}_n^{t,0}\right)\right\Vert^2+\sigma_{\mathrm{v}}^2\right) 
\end{align}

Similarly, for the discriminator we have

\begin{align}
    \sum_{i=0}^{\tau-1}\mathbb{E}^t\left[\left\Vert \varphi_n\left(\boldsymbol{w}^{t,i}\right)-\varphi\left(\boldsymbol{w}^{t,0}\right)\right\Vert^2\right] 
    &\leq 8\tau^3\left(\eta_d^t\right)^2L_{\mathrm{u}}^2\left(\left\Vert \nabla_{\boldsymbol{w}_n^{t,0}}\mathcal{U}\left(\boldsymbol{w}^{t,0}\right)\right\Vert^2+\sigma_{\mathrm{u}}^2\right)
\end{align}
    The above finishes the proof.
 \end{proof}

Now, we are ready to compute the convergence error of \texttt{FedAL}.
It is worth noting that at the beginning of each training round, we have $\nabla_{\boldsymbol{\theta}_n^{t,i}}\mathcal{V}_n\left(\boldsymbol{\theta}_n^{t,i}\right) = \nabla_{\boldsymbol{\theta}_n^{t,i}}V_n^{\mathrm{loc}}\left(\boldsymbol{\theta}_n^{t,i}\right)$. At the beginning of each global knowledge transfer, we have $\nabla_{\boldsymbol{\theta}_n^{t+\frac{1}{2},i}}\mathcal{V}_n\left(\boldsymbol{\theta}_n^{t+\frac{1}{2},i}\right) = \nabla_{\boldsymbol{\theta}_n^{t+\frac{1}{2},i}}V_n^{\mathrm{glo}}\left(\boldsymbol{\theta}_n^{t+\frac{1}{2},i}\right)$.
 Denote $\mathcal{F}\left(\boldsymbol{\Theta},\boldsymbol{w}\right):=\left\{\mathcal{V}_1\left(\boldsymbol{\theta}_1\right), \mathcal{V}_2\left(\boldsymbol{\theta}_2\right)\ldots, \mathcal{V}_N\left(\boldsymbol{\theta}_N\right), \mathcal{U}\left(\boldsymbol{w}\right)\right\}$. Obviously, $\mathcal{F}\left(\boldsymbol{\Theta},\boldsymbol{w}\right)$ is a Lipschitz-smooth function with a constant $L_{\mathrm{o}}$ by Lemma \ref{lem:L_continuity U}.
Let $\nabla\mathcal{F}\left(\boldsymbol{\Theta},\boldsymbol{w}\right)$ be the stacked gradient at iteration $i$ of training round $t$, we have
\begin{equation}
    \nabla\mathcal{F}\left(\boldsymbol{\Theta}^{t,i},\boldsymbol{w}^{t,i}\right):=\left\{\nabla_{\boldsymbol{\theta}_1^{t,i}}\mathcal{V}_1\left(\boldsymbol{\theta}_1^{t,i}\right),\ldots,\nabla_{\boldsymbol{\theta}_N^{t,i}}\mathcal{V}_N\left(\boldsymbol{\theta}_N^{t,i}\right),\nabla_{\boldsymbol{w}^{t,i}}\mathcal{U}\left(\boldsymbol{w}^{t,i}\right)\right\}
\end{equation}

\clearpage

\textbf{Proof of Theorem~\ref{thm:convergence}}:
\begin{proof}
By the smoothness of $\mathcal{F}\left(\boldsymbol{\theta},\boldsymbol{w}\right)$, we have
\begin{align}
    &\mathbb{E}^t\left[\mathcal{F}\left(\boldsymbol{\Theta}^{t+1,0},\boldsymbol{w}^{t+1,0}\right)\right]-\mathcal{F}\left(\boldsymbol{\Theta}^{t,0},\boldsymbol{w}^{t,0}\right)\nonumber\\
    & \leq\mathbb{E}^t\left[\left\langle\nabla \mathcal{F}\left(\boldsymbol{\Theta}^{t,0},\boldsymbol{w}^{t,0}\right), \left(\boldsymbol{\Theta}^{t+1,0}-\boldsymbol{\Theta}^{t, 0},\boldsymbol{w}^{t+1,0}-\boldsymbol{w}^{t, 0}\right)\right\rangle\right]+\frac{L_{\mathrm{o}}}{2}\mathbb{E}^t\left[\left\|\boldsymbol{\Theta}^{t+1,0}-\boldsymbol{\Theta}^{t, 0}\right\|^2 \right]+\frac{L_{\mathrm{o}}}{2}\mathbb{E}^t\left[\left\|\boldsymbol{w}^{t+1,0}-\boldsymbol{w}^{t, 0}\right\|^2\right]\nonumber\\
    \tag{By Smoothness: $f(\boldsymbol{y}) \leq f(\boldsymbol{x})+\langle\nabla f(\boldsymbol{x}), \boldsymbol{y}-\boldsymbol{x}\rangle+\frac{L}{2}\|\boldsymbol{y}-\boldsymbol{x}\|_2^2$ for Lipschitz constant $L$}\\
     & \leq-\sum_{n=1}^N \eta_l^t \sum_{i=0}^{\tau -1}\mathbb{E}^t\left[\left\langle\nabla_{\boldsymbol{\theta}_n^{t,0}} \mathcal{V}_n\left(\boldsymbol{\theta}_n^{t,0}\right),  \tilde{g}\left(\boldsymbol{\theta}_n^{t,i}\right)+\hat{g}\left(\boldsymbol{\theta}_n^{t+\frac{1}{2},i}\right)\right\rangle\right]
     -\eta_d^t\sum_{i=0}^{\tau -1}\mathbb{E}^t\left[\left\langle\nabla_{\boldsymbol{w}^{t,0}}\mathcal{U}\left(\boldsymbol{w}^{t,0}\right), \varphi\left(\boldsymbol{w}^{t,i}\right)\right\rangle\right]\nonumber \\
    & +\frac{L_{\mathrm{o}}\tau}{2}\sum_{n=1}^N\left(\eta_l^t\right)^2\sum_{i=0}^{\tau-1}\mathbb{E}^t\left[\left\|\tilde{g}\left(\boldsymbol{\theta}_n^{t,i}\right)+\hat{g}\left(\boldsymbol{\theta}_n^{t+\frac{1}{2},i}\right)\right\|^2 \right]
    +\frac{L_{\mathrm{o}}\tau}{2}\left(\eta_d^t\right)^2\sum_{i=0}^{\tau-1}\mathbb{E}^t\left[\left\|\varphi\left(\boldsymbol{w}^{t,i}\right)\right\|^2\right]\label{eq:overall_diff}\\
    \tag{By gradient update $\boldsymbol{\Theta}^{t+1,0}-\boldsymbol{\Theta}^{t, 0} = -\eta_l^t \sum_{i=0}^{\tau-1} \left(\tilde{g}\left(\boldsymbol{\theta}_n^{t,i}\right)+\hat{g}\left(\boldsymbol{\theta}_n^{t+\frac{1}{2},i}\right)\right)$ and $\left(\sum_{i=1}^{\tau}\boldsymbol{x}_i\right)^2\leq \tau\sum_{i=1}^{\tau}\boldsymbol{x}_i^2$}
\end{align}

For the first two terms, there exist
\begin{align}
    & -\mathbb{E}^t\left[\left\langle\nabla_{\boldsymbol{\theta}_n^{t,0}} \mathcal{V}_n\left(\boldsymbol{\theta}_n^{t,0}\right),  \tilde{g}\left(\boldsymbol{\theta}_n^{t,i}\right)+\hat{g}\left(\boldsymbol{\theta}_n^{t+\frac{1}{2},i}\right)\right\rangle\right] \nonumber\\
    & =-\mathbb{E}^t\left[\left\langle\nabla_{\boldsymbol{\theta}_n^{t,0}} \mathcal{V}_n\left(\boldsymbol{\theta}_n^{t,0}\right),  \tilde{g}\left(\boldsymbol{\theta}_n^{t,i}\right)\right\rangle\right] -\mathbb{E}^t\left[\left\langle\nabla_{\boldsymbol{\theta}_n^{t,0}} \mathcal{V}_n\left(\boldsymbol{\theta}_n^{t,0}\right),  \hat{g}\left(\boldsymbol{\theta}_n^{t+\frac{1}{2},i}\right)\right\rangle\right] \nonumber\\
    & = \mathbb{E}^t\left[\left\langle-\nabla_{\boldsymbol{\theta}_n^{t,0}} \mathcal{V}_n\left(\boldsymbol{\theta}_n^{t,0}\right),  \tilde{g}\left(\boldsymbol{\theta}_n^{t,i}\right)-\tilde{g}\left(\boldsymbol{\theta}_n^{t,0}\right)\right\rangle\right] - \mathbb{E}^t\left[\left\langle\nabla_{\boldsymbol{\theta}_n^{t,0}} \mathcal{V}_n\left(\boldsymbol{\theta}_n^{t,0}\right),  \tilde{g}\left(\boldsymbol{\theta}_n^{t,0}\right)\right\rangle \right]\nonumber\\
    &-\mathbb{E}^t\left[\left\langle-\nabla_{\boldsymbol{\theta}_n^{t,0}} \mathcal{V}_n\left(\boldsymbol{\theta}_n^{t,0}\right),  \hat{g}\left(\boldsymbol{\theta}_n^{t+\frac{1}{2},i}\right)-\tilde{g}\left(\boldsymbol{\theta}_n^{t,0}\right)\right\rangle\right] - \mathbb{E}^t\left[\left\langle\nabla_{\boldsymbol{\theta}_n^{t,0}} \mathcal{V}_n\left(\boldsymbol{\theta}_n^{t,0}\right),  \tilde{g}\left(\boldsymbol{\theta}_n^{t,0}\right)\right\rangle \right]\nonumber\\
    & \leq \frac{1}{2}\left\|\nabla_{\boldsymbol{\theta}_n^{t,0}} \mathcal{V}_n\left(\boldsymbol{\theta}_n^{t,0}\right)\right\|^2+\frac{1}{2}\mathbb{E}^t\left[\left\|\tilde{g}\left(\boldsymbol{\theta}_n^{t,i}\right)-\tilde{g}\left(\boldsymbol{\theta}_n^{t,0}\right)\right\|^2 \right]
    -\mathbb{E}^t\left[\left\langle\nabla_{\boldsymbol{\theta}_n^{t,0}} \mathcal{V}_n\left(\boldsymbol{\theta}_n^{t,0}\right),  \tilde{g}\left(\boldsymbol{\theta}_n^{t,0}\right)\right\rangle\right] \nonumber\\
    & + \frac{1}{2}\left\|\nabla_{\boldsymbol{\theta}_n^{t,0}} \mathcal{V}_n\left(\boldsymbol{\theta}_n^{t,0}\right)\right\|^2+\frac{1}{2}\mathbb{E}^t\left[\left\|\hat{g}\left(\boldsymbol{\theta}_n^{t+\frac{1}{2},i}\right)-\tilde{g}\left(\boldsymbol{\theta}_n^{t,0}\right)\right\|^2 \right]
    -\mathbb{E}^t\left[\left\langle\nabla_{\boldsymbol{\theta}_n^{t,0}} \mathcal{V}_n\left(\boldsymbol{\theta}_n^{t,0}\right),  \tilde{g}\left(\boldsymbol{\theta}_n^{t,0}\right)\right\rangle\right]\nonumber \\
    \tag{By $\left\langle A, B\right\rangle=\frac{1}{2}A^2+\frac{1}{2}B^2-\frac{1}{2}\left(A-B\right)^2\leq \frac{1}{2}A^2+\frac{1}{2}B^2$}\\
     &  = \left\|\nabla_{\boldsymbol{\theta}_n^{t,0}} \mathcal{V}_n\left(\boldsymbol{\theta}_n^{t,0}\right)\right\|^2+\frac{1}{2}\mathbb{E}^t\left[\left\|\tilde{g}\left(\boldsymbol{\theta}_n^{t,i}\right)-\tilde{g}\left(\boldsymbol{\theta}_n^{t,0}\right)\right\|^2 \right] 
   +\frac{1}{2}\mathbb{E}^t\left[\left\|\hat{g}\left(\boldsymbol{\theta}_n^{t+\frac{1}{2},i}\right)-\tilde{g}\left(\boldsymbol{\theta}_n^{t,0}\right)\right\|^2 \right]
   -2\mathbb{E}^t\left[\left\langle\nabla_{\boldsymbol{\theta}_n^{t,0}} \mathcal{V}_n\left(\boldsymbol{\theta}_n^{t,0}\right),  \tilde{g}\left(\boldsymbol{\theta}_n^{t,0}\right)\right\rangle\right] 
\end{align}
Similarly, 
\begin{align}
    & -\mathbb{E}^t\left[\left\langle\nabla_{\boldsymbol{w}^{t,0}} \mathcal{U}\left(\boldsymbol{w}^{t,0}\right), \varphi\left(\boldsymbol{w}^{t,i}\right)\right\rangle\right] \nonumber\\
    & \leq \frac{1}{2} \left\|\nabla_{\boldsymbol{w}^{t,0}}\mathcal{U}\left(\boldsymbol{w}^{t,0}\right)\right\|^2+\frac{1}{2}\mathbb{E}^t\left[\left\|\varphi\left(\boldsymbol{w}^{t,i}\right)-\varphi\left(\boldsymbol{w}^{t,0}\right)\right\|^2\right]
    -\mathbb{E}^t\left[\left\langle\nabla_{\boldsymbol{w}^{t,0}} \mathcal{U}\left(\boldsymbol{w}^{t,0}\right),  \varphi\left(\boldsymbol{w}^{t,0}\right)\right\rangle\right]
\end{align}
For the latter two terms, we have
\begin{align}
    &\mathbb{E}^t\left[\left\|\tilde{g}\left(\boldsymbol{\theta}_n^{t,i}\right)+\hat{g}\left(\boldsymbol{\theta}_n^{t+\frac{1}{2},i}\right)\right\|^2\right] 
    \leq 2\mathbb{E}^t\left[\left\|\tilde{g}\left(\boldsymbol{\theta}_n^{t,i}\right)-\tilde{g}\left(\boldsymbol{\theta}_n^{t,0}\right)+\tilde{g}\left(\boldsymbol{\theta}_n^{t,0}\right)\right\|^2\right]
    +2\mathbb{E}^t\left[\left\|\hat{g}\left(\boldsymbol{\theta}_n^{t+\frac{1}{2},i}\right)-\tilde{g}\left(\boldsymbol{\theta}_n^{t,0}\right)+\tilde{g}\left(\boldsymbol{\theta}_n^{t,0}\right)\right\|^2\right]\nonumber\\
    & \leq 4\left(\mathbb{E}^t\left[\left\|\tilde{g}\left(\boldsymbol{\theta}_n^{t,i}\right)-\tilde{g}\left(\boldsymbol{\theta}_n^{t,0}\right)\right\|^2\right]+\mathbb{E}^t\left[\left\|\hat{g}\left(\boldsymbol{\theta}_n^{t+\frac{1}{2},i}\right)-\tilde{g}\left(\boldsymbol{\theta}_n^{t,0}\right)\right\|^2\right]+2\mathbb{E}^t\left[\left\|\tilde{g}\left(\boldsymbol{\theta}_n^{t,0}\right)\right\|^2\right]\right) \\
    \tag{By Jensen's inequality}
\end{align}

\begin{align}
    &\mathbb{E}^t\left[\left\|\varphi\left(\boldsymbol{w}^{t,i}\right)\right\|^2\right] =\mathbb{E}^t\left[\left\|\varphi\left(\boldsymbol{w}^{t,i}\right)-\varphi\left(\boldsymbol{w}^{t,0}\right)+\varphi\left(\boldsymbol{w}^{t,0}\right)\right\|^2\right]\nonumber\\
    & \leq 2\left(\mathbb{E}^t\left[\left\|\varphi\left(\boldsymbol{w}^{t,i}\right)-\varphi\left(\boldsymbol{w}^{t,0}\right)\right\|^2\right]+\mathbb{E}^t\left[\left\|\varphi\left(\boldsymbol{w}^{t,0}\right)\right\|^2\right]\right) \\
    \tag{By Jensen's inequality}
\end{align}

Thus, \eqref{eq:overall_diff} becomes
\begin{align}
&\mathbb{E}^t\left[\mathcal{F}\left(\boldsymbol{\Theta}^{t+1,0},\boldsymbol{w}^{t+1,0}\right)\right]-\mathcal{F}\left(\boldsymbol{\Theta}^{t,0},\boldsymbol{w}^{t,0}\right)\nonumber\\
    & \leq \sum_{n=1}^N \eta_l^t \sum_{i=0}^{\tau -1} \left[\left\|\nabla_{\boldsymbol{\theta}_n^{t,0}} \mathcal{V}_n\left(\boldsymbol{\theta}_n^{t,0}\right)\right\|^2+\frac{1}{2}\mathbb{E}^t\left[\left\|\tilde{g}\left(\boldsymbol{\theta}_n^{t,i}\right)-\tilde{g}\left(\boldsymbol{\theta}_n^{t,0}\right)\right\|^2 \right]
    \right.\nonumber\\
    &\left.+\frac{1}{2}\mathbb{E}^t\left[\left\|\hat{g}\left(\boldsymbol{\theta}_n^{t+\frac{1}{2},i}\right)-\tilde{g}\left(\boldsymbol{\theta}_n^{t,0}\right)\right\|^2 \right]
    -2\mathbb{E}^t\left[\left\langle\nabla_{\boldsymbol{\theta}_n^{t,0}} \mathcal{V}_n\left(\boldsymbol{\theta}_n^{t,0}\right),  \tilde{g}\left(\boldsymbol{\theta}_n^{t,0}\right)\right\rangle\right]\right]\nonumber\\
    & +\frac{L_{\mathrm{o}}\tau}{2}\sum_{n=1}^N\left(\eta_l^t\right)^2\sum_{i=0}^{\tau-1}4\left(\mathbb{E}^t\left[\left\|\tilde{g}\left(\boldsymbol{\theta}_n^{t,i}\right)-\tilde{g}\left(\boldsymbol{\theta}_n^{t,0}\right)\right\|^2\right]+\mathbb{E}^t\left[\left\|\hat{g}\left(\boldsymbol{\theta}_n^{t+\frac{1}{2},i}\right)-\tilde{g}\left(\boldsymbol{\theta}_n^{t,0}\right)\right\|^2\right]+2\mathbb{E}^t\left[\left\|\tilde{g}\left(\boldsymbol{\theta}_n^{t,0}\right)\right\|^2\right]\right) \nonumber\\
     &+\eta_d^t\sum_{i=0}^{\tau -1}\left[\frac{1}{2} \left\|\nabla_{\boldsymbol{w}^{t,0}} \mathcal{U}\left(\boldsymbol{w}^{t,0}\right)\right\|^2+\frac{1}{2}\mathbb{E}^t\left[\left\|\varphi\left(\boldsymbol{w}^{t,i}\right)-\varphi\left(\boldsymbol{w}^{t,0}\right)\right\|^2\right]
    -\mathbb{E}^t\left[\left\langle\nabla_{\boldsymbol{w}^{t,0}} \mathcal{U}\left(\boldsymbol{w}^{t,0}\right),  \varphi\left(\boldsymbol{w}^{t,0}\right)\right\rangle\right]\right]\nonumber \\
    & +\frac{L_{\mathrm{o}}\tau}{2}\left(\eta_d^t\right)^2\sum_{i=0}^{\tau-1}2\left(\mathbb{E}^t\left[\left\|\varphi\left(\boldsymbol{w}^{t,i}\right)-\varphi\left(\boldsymbol{w}^{t,0}\right)\right\|^2\right]+\mathbb{E}^t\left[\left\|\varphi\left(\boldsymbol{w}^{t,0}\right)\right\|^2\right]\right)\\
     & \leq \sum_{n=1}^N \eta_l^t \sum_{i=0}^{\tau -1}\left[\left\|\nabla_{\boldsymbol{\theta}_n^{t,0}} \mathcal{V}_n\left(\boldsymbol{\theta}_n^{t,0}\right)\right\|^2+\frac{1}{2}\mathbb{E}^t\left[\left\|\tilde{g}\left(\boldsymbol{\theta}_n^{t,i}\right)-\tilde{g}\left(\boldsymbol{\theta}_n^{t,0}\right)\right\|^2 \right]
     +\frac{1}{2}\mathbb{E}^t\left[\left\|\hat{g}\left(\boldsymbol{\theta}_n^{t+\frac{1}{2},i}\right)-\tilde{g}\left(\boldsymbol{\theta}_n^{t,0}\right)\right\|^2 \right]
    -2\left\|\nabla_{\boldsymbol{\theta}_n^{t,0}} \mathcal{V}_n\left(\boldsymbol{\theta}_n^{t,0}\right)\right\|^2\right]\nonumber\\
    & +\frac{L_{\mathrm{o}}\tau}{2}\sum_{n=1}^N\left(\eta_l^t\right)^2\sum_{i=0}^{\tau-1}4\left(\mathbb{E}^t\left[\left\|\tilde{g}\left(\boldsymbol{\theta}_n^{t,i}\right)-\tilde{g}\left(\boldsymbol{\theta}_n^{t,0}\right)\right\|^2\right]+\mathbb{E}^t\left[\left\|\hat{g}\left(\boldsymbol{\theta}_n^{t+\frac{1}{2},i}\right)-\tilde{g}\left(\boldsymbol{\theta}_n^{t,0}\right)\right\|^2\right]+2\left\|\nabla_{\boldsymbol{\theta}_n^{t,0}} \mathcal{V}_n\left(\boldsymbol{\theta}_n^{t,0}\right)\right\|^2+2\sigma_{\mathrm{v}}^2\right) \nonumber\\
    \tag{By Lemmas \ref{lem:umbiased_grad} and \ref{lem:bounded_variance}, $\mathbb{E}^t\left[\left\Vert\mathbf{z}\right\Vert^2\right]\leq\left\Vert\mathbb{E}^t\left[\mathbf{z}\right]\right\Vert^2+\mathbb{E}^t\left[\left\Vert\mathbf{z}-\mathbb{E}^t\left[\mathbf{z}\right]\right\Vert^2\right]$ for any random variable $\mathbf{z}$}\\
    \tag{$\nabla_{\boldsymbol{\theta}_n^{t,0}} \mathcal{V}_n\left(\boldsymbol{\theta}_n^{t,0}\right)=\nabla_{\boldsymbol{\theta}_n^{t,0}} V_n^{\mathrm{loc}}\left(\boldsymbol{\theta}_n^{t,0}\right)$ at point $\boldsymbol{\theta}_n^{t,0}$}\\
     &+\eta_d^t\sum_{i=0}^{\tau -1}\left[\frac{1}{2}\left\|\nabla_{\boldsymbol{w}^{t,0}} \mathcal{U}\left(\boldsymbol{w}^{t,0}\right)\right\|^2+\frac{1}{2}\mathbb{E}^t\left[\left\|\varphi\left(\boldsymbol{w}^{t,i}\right)-\varphi\left(\boldsymbol{w}^{t,0}\right)\right\|^2\right]
    -\left\|\nabla_{\boldsymbol{w}^{t,0}} \mathcal{U}\left(\boldsymbol{w}^{t,0}\right)\right\|^2\right]\nonumber \\
    & +\frac{L_{\mathrm{o}}\tau}{2}\left(\eta_d^t\right)^2\sum_{i=0}^{\tau-1}2\left(\mathbb{E}^t\left[\left\|\varphi\left(\boldsymbol{w}^{t,i}\right)-\varphi\left(\boldsymbol{w}^{t,0}\right)\right\|^2\right]+\left\|\nabla_{\boldsymbol{w}^{t,0}} \mathcal{U}\left(\boldsymbol{w}^{t,0}\right)\right\|^2+\sigma_{\mathrm{u}}^2\right)\\
    & \leq -\frac{1}{2}\sum_{n=1}^N \eta_l^t \sum_{i=0}^{\tau -1} \left(2- 8L_{\mathrm{o}}\tau\eta_l^t\right)
   \left\|\nabla_{\boldsymbol{\theta}_n^{t,0}} \mathcal{V}_n\left(\boldsymbol{\theta}_n^{t,0}\right)\right\|^2 +4L_{\mathrm{o}}\tau\sum_{n=1}^N\left(\eta_l^t\right)^2\sum_{i=0}^{\tau-1}\sigma_{\mathrm{v}}^2
  \nonumber\\
    &  +\frac{1}{2}\sum_{n=1}^N \eta_l^t \sum_{i=0}^{\tau -1}\left(1+4L_{\mathrm{o}}\tau\eta_l^t\right)
    \left(\mathbb{E}^t\left[\left\|\tilde{g}\left(\boldsymbol{\theta}_n^{t,i}\right)-\tilde{g}\left(\boldsymbol{\theta}_n^{t,0}\right)\right\|^2 \right]+\mathbb{E}^t\left[\left\|\hat{g}\left(\boldsymbol{\theta}_n^{t,i}\right)-\tilde{g}\left(\boldsymbol{\theta}_n^{t,0}\right)\right\|^2 \right]\right) \nonumber\\
     &-\frac{1}{2} \eta_d^t\sum_{i=0}^{\tau -1}\left(1-2L_{\mathrm{o}}\tau\eta_d^t\right)
      \left\|\nabla_{\boldsymbol{w}^{t,0}} \mathcal{U}\left(\boldsymbol{w}^{t,0}\right)\right\|^2+\frac{1}{2}\eta_d^t\sum_{i=0}^{\tau -1}\left(1+2L_{\mathrm{o}}\tau\eta_d^t\right) \mathbb{E}^t\left[\left\|\varphi\left(\boldsymbol{w}^{t,i}\right)-\varphi\left(\boldsymbol{w}^{t,0}\right)\right\|^2\right]\nonumber \\
    & +L_{\mathrm{o}}\tau\left(\eta_d^t\right)^2\sum_{i=0}^{\tau-1}\sigma_{\mathrm{u}}^2
\end{align}

Applying Lemma \ref{lem:one-round-update}, we have
\begin{align}
    &\mathbb{E}^t\left[\mathcal{F}\left(\boldsymbol{\Theta}^{t+1,0},\boldsymbol{w}^{t+1,0}\right)\right]-\mathcal{F}\left(\boldsymbol{\Theta}^{t,0},\boldsymbol{w}^{t,0}\right)\nonumber\\
    & \leq -\frac{1}{2}\sum_{n=1}^N \eta_l^t \tau \left(2- 8L_{\mathrm{o}}\tau\eta_l^t\right)
    \left\|\nabla_{\boldsymbol{\theta}_n^{t,0}} \mathcal{V}_n\left(\boldsymbol{\theta}_n^{t,0}\right)\right\|^2+4L_{\mathrm{o}}\tau^2\sum_{n=1}^N\left(\eta_l^t\right)^2\sigma_{\mathrm{v}}^2 \nonumber\\
    &+4\sum_{n=1}^N \left(1+4L_{\mathrm{o}}\tau\eta_l^t\right)
    \tau^3\left(\eta_l^t\right)^3\left(3L_{\mathrm{v}}^2+L_{\mathrm{w}}^2\right)\left( \left\Vert \nabla_{\boldsymbol{\theta}_n^{t,0}}\mathcal{V}_n\left(\boldsymbol{\theta}_n^{t,0}\right)\right\Vert^2+\sigma_{\mathrm{v}}^2\right)\nonumber\\
     &-\frac{1}{2} \eta_d^t\tau\left(1-2L_{\mathrm{o}}\tau\eta_d^t\right) 
     \left\|\nabla_{\boldsymbol{w}^{t,0}} \mathcal{U}\left(\boldsymbol{w}^{t,0}\right)\right\|^2+L_{\mathrm{o}}\tau^2\left(\eta_d^t\right)^2\sigma_{\mathrm{u}}^2\nonumber\\
     &+4\left(1+2L_{\mathrm{o}}\tau\eta_d^t\right)\tau^3\left(\eta_d^t\right)^3L_{\mathrm{u}}^2\left( \left\Vert \nabla_{\boldsymbol{w}_n^{t,0}}\mathcal{U}\left(\boldsymbol{w}^{t,0}\right)\right\Vert^2+\sigma_{\mathrm{u}}^2\right)\nonumber\\
     & \leq -\frac{1}{2}\sum_{n=1}^N \eta_l^t \tau \left(2- 8L_{\mathrm{o}}\tau\eta_l^t-8 \left(1+4L_{\mathrm{o}}\tau\eta_l^t\right)
    \tau^2\left(\eta_l^t\right)^2\left(3L_{\mathrm{v}}^2+L_{\mathrm{w}}^2\right)\right)
     \left\|\nabla_{\boldsymbol{\theta}_n^{t,0}} \mathcal{V}_n\left(\boldsymbol{\theta}_n^{t,0},\boldsymbol{w}^{t,0}\right)\right\|^2\nonumber\\
    &+\left(4L_{\mathrm{o}}\tau^2\sum_{n=1}^N\left(\eta_l^t\right)^2+4\sum_{n=1}^N \left(1+4L_{\mathrm{o}}\tau\eta_l^t\right)
    \tau^3\left(\eta_l^t\right)^3\left(3L_{\mathrm{v}}^2+L_{\mathrm{w}}^2\right)\right)\sigma_{\mathrm{v}}^2 \nonumber\\
    &-\frac{1}{2} \eta_d^t\tau\left(1-2L_{\mathrm{o}}\tau\eta_d^t-8\left(1+2L_{\mathrm{o}}\tau\eta_d^t\right)\tau^2\left(\eta_d^t\right)^2L_{\mathrm{u}}^2\right)
      \left\|\nabla_{\boldsymbol{w}^{t,0}} \mathcal{U}\left(\boldsymbol{w}^{t,0}\right)\right\|^2\nonumber\\
     &+\left(L_{\mathrm{o}}\tau^2\left(\eta_d^t\right)^2+4\left(1+2L_{\mathrm{o}}\tau\eta_d^t\right)\tau^3\left(\eta_d^t\right)^3L_{\mathrm{u}}^2\right)\sigma_{\mathrm{u}}^2\\
      & \leq -\frac{1}{2}\sum_{n=1}^N \eta_l^t \tau \left(2- 1-\frac{1}{8}-\frac{1}{16}\right)
    \left\|\nabla_{\boldsymbol{\theta}_n^{t,0}} \mathcal{V}_n\left(\boldsymbol{\theta}_n^{t,0}\right)\right\|^2+\max\{L_{\mathrm{o}},\sqrt{3L_{\mathrm{v}}^2+L_{\mathrm{w}}^2}\}\tau^2\sum_{n=1}^N\left(\eta_l^t\right)^2\left(4+\frac{1}{2}+\frac{1}{4}\right)\sigma_{\mathrm{v}}^2
    \tag{Let $\eta_l^t\leq \frac{1}{8\tau \max\{L_{\mathrm{o}},\sqrt{3L_{\mathrm{v}}^2+L_{\mathrm{w}}^2}\}}$}\\
    &-\frac{1}{2} \eta_d^t\tau\left(1- \frac{1}{4}-\frac{1}{8}-\frac{1}{32}\right)
      \left\|\nabla_{\boldsymbol{w}^{t,0}} \mathcal{U}\left(\boldsymbol{w}^{t,0}\right)\right\|^2 +\max\{L_{\mathrm{o}},L_{\mathrm{u}}\}\tau^2\left(\eta_d^t\right)^2\left(1+\frac{1}{2}+\frac{1}{8}\right)\sigma_{\mathrm{u}}^2 \tag{Let $\eta_d^t\leq \frac{1}{8\max\{L_{\mathrm{o}},L_{\mathrm{u}}\}\tau }$}\\
     & \leq -\frac{1}{4}\sum_{n=1}^N \eta_l^t \tau 
   \left\|\nabla_{\boldsymbol{\theta}_n^{t,0}} \mathcal{V}_n\left(\boldsymbol{\theta}_n^{t,0}\right)\right\|^2+5\max\{L_{\mathrm{o}},\sqrt{3L_{\mathrm{v}}^2+L_{\mathrm{w}}^2}\}\tau^2\sum_{n=1}^N\left(\eta_l^t\right)^2\sigma_{\mathrm{v}}^2 \nonumber\\
    &-\frac{1}{4} \eta_d^t\tau
      \left\|\nabla_{\boldsymbol{w}^{t,0}} \mathcal{U}\left(\boldsymbol{w}^{t,0}\right)\right\|^2+5\max\{L_{\mathrm{o}},L_{\mathrm{u}}\}\tau^2\left(\eta_d^t\right)^2\sigma_{\mathrm{u}}^2
\end{align}
Rearranging the above gives
\begin{align}
    &\sum_{n=1}^N \eta_l^t \tau 
     \left\|\nabla_{\boldsymbol{\theta}_n^{t,0}} \mathcal{V}_n\left(\boldsymbol{\theta}_n^{t,0}\right)\right\|^2+\eta_d^t\tau
      \left\|\nabla_{\boldsymbol{w}^{t,0}} \mathcal{U}\left(\boldsymbol{w}^{t,0}\right)\right\|^2 \nonumber\\
     &\leq 4\left(\mathcal{F}\left(\boldsymbol{\Theta}^{t,0},\boldsymbol{w}^{t,0}\right)-\mathbb{E}^t\left[\mathcal{F}\left(\boldsymbol{\Theta}^{t+1,0},\boldsymbol{w}^{t+1,0}\right)\right]\right)
     + 20\max\{L_{\mathrm{o}},\sqrt{3L_{\mathrm{v}}^2+L_{\mathrm{w}}^2}\}\tau^2\sum_{n=1}^N\left(\eta_l^t\right)^2\sigma_{\mathrm{v}}^2
     +20\max\{L_{\mathrm{o}},L_{\mathrm{u}}\}\tau^2\left(\eta_d^t\right)^2\sigma_{\mathrm{u}}^2
\end{align}
We average over all training round $t=0,\cdots,T-1$ and take total expectation:
\begin{align}
    &\frac{1}{T}\sum_{t=0}^{T-1}\left(\sum_{n=1}^N\eta_l^t\tau\mathbb{E}\left[\left\|\nabla_{\boldsymbol{\theta}_n^{t,0}}\mathcal{V}_n\left(\boldsymbol{\theta}_n^{t,0}\right)\right\|^2\right]+\eta_d^t\tau\mathbb{E}\left[\left\|\nabla_{\boldsymbol{w}^{t,0}}\mathcal{U}\left(\boldsymbol{w}^{t,0}\right)\right\|^2\right]\right)\nonumber\\
    & \leq \frac{4\left(\mathcal{F}\left(\boldsymbol{\Theta}^{0,0},\boldsymbol{w}^{0,0}\right)-\mathbb{E}\left[\mathcal{F}\left(\boldsymbol{\Theta}^{T,0},\boldsymbol{w}^{T,0}\right)\right]\right)}{T}+ \frac{20\max\{L_{\mathrm{o}},\sqrt{3L_{\mathrm{v}}^2+L_{\mathrm{w}}^2},L_{\mathrm{u}}\}\tau^2\max\left\{\sigma_{\mathrm{v}}^2,\sigma_{\mathrm{u}}^2\right\}}{T}\sum_{t=0}^{T-1}\left(\sum_{n=1}^N\left(\eta_l^t\right)^2+\left(\eta_d^t\right)^2\right)
\end{align}
Let $\sigma:=\max\left\{\sigma_{\mathrm{v}},\sigma_{\mathrm{u}}\right\}$ and $L:=\max\{L_{\mathrm{o}},\sqrt{3L_{\mathrm{v}}^2+L_{\mathrm{w}}^2},L_{\mathrm{u}}\}$, thus,
\begin{align}
    &\frac{1}{T}\sum_{t=0}^{T-1}\left(\sum_{n=1}^N\eta_l^t\mathbb{E}\left[\left\|\nabla_{\boldsymbol{\theta}_n^{t,0}}\mathcal{V}_n\left(\boldsymbol{\theta}_n^{t,0}\right)\right\|^2\right]+\eta_d^t\mathbb{E}\left[\left\|\nabla_{\boldsymbol{w}^{t,0}}\mathcal{U}\left(\boldsymbol{w}^{t,0}\right)\right\|^2\right]\right)\nonumber\\
    & \leq \frac{4}{T \tau}\left(\mathcal{F}\left(\boldsymbol{\Theta}^{0,0},\boldsymbol{w}^{0,0}\right)-\mathbb{E}\left[\mathcal{F}\left(\boldsymbol{\Theta}^{\ast,0},\boldsymbol{w}^{\ast,0}\right)\right]\right)+ \frac{20L\tau\sigma^2}{T}\sum_{t=0}^{T-1}\left(\sum_{n=1}^N\left(\eta_l^t\right)^2+\left(\eta_d^t\right)^2\right)
\end{align}
Denote $S:=\sum_{t=0}^{T-1}\left(\sum_{n=1}^N\eta_l^t+\eta_d^t\right)$, the average gradient can be bounded as follows:
    \begin{align}
    &\frac{1}{S}\sum_{t=0}^{T-1}\left(\sum_{n=1}^N\eta_l^t\mathbb{E}\left[\left\|\nabla_{\boldsymbol{\theta}_n^{t,0}}\mathcal{V}_n\left(\boldsymbol{\theta}_n^{t,0}\right)\right\|^2\right]+\eta_d^t\mathbb{E}\left[\left\|\nabla_{\boldsymbol{w}^{t,0}}\mathcal{U}\left(\boldsymbol{w}^{t,0}\right)\right\|^2\right]\right)\nonumber\\
    & \leq \frac{4}{S \tau}\left(\mathcal{F}\left(\boldsymbol{\Theta}^{0,0},\boldsymbol{w}^{0,0}\right)-\mathcal{F}\left(\boldsymbol{\Theta}^{\ast},\boldsymbol{w}^{\ast}\right)\right)+ \frac{20L\tau\sigma^2}{S}\sum_{t=0}^{T-1}\left(\sum_{n=1}^N\left(\eta_l^t\right)^2+\left(\eta_d^t\right)^2\right)
\end{align}

If we let $\eta_l^t=\eta_d^t=\frac{1}{\sqrt{T\tau}}$, we have
    \begin{align}
    &\frac{1}{S}\sum_{t=0}^{T-1}\left(\sum_{n=1}^N\eta_l^t\mathbb{E}\left[\left\|\nabla_{\boldsymbol{\theta}_n^{t,0}}\mathcal{V}_n\left(\boldsymbol{\theta}_n^{t,0}\right)\right\|^2\right]+\eta_d^t\mathbb{E}\left[\left\|\nabla_{\boldsymbol{w}^{t,0}}\mathcal{U}\left(\boldsymbol{w}^{t,0}\right)\right\|^2\right]\right)\nonumber\\
    & \leq \frac{4}{T (N+1) \frac{1}{\sqrt{T\tau}}\tau}\left(\mathcal{F}\left(\boldsymbol{\Theta}^{0,0},\boldsymbol{w}^{0,0}\right)-\mathcal{F}\left(\boldsymbol{\Theta}^{\ast},\boldsymbol{w}^{\ast}\right)\right)+ \frac{20L\tau\sigma^2}{T(N+1)\frac{1}{\sqrt{T\tau}}}T (N+1) \left(\frac{1}{\sqrt{T\tau}}\right)^2\nonumber\\
    &= \frac{4}{(N+1)\sqrt{T\tau}}\left(\mathcal{F}\left(\boldsymbol{\Theta}^{0,0},\boldsymbol{w}^{0,0}\right)-\mathcal{F}\left(\boldsymbol{\Theta}^{\ast},\boldsymbol{w}^{\ast}\right)\right)+ \frac{20L\tau\sigma^2}{\sqrt{T\tau}} 
\end{align}
Then, we finish the proof.

\end{proof}

\vfill

\end{document}